\documentclass[11pt]{article}
\usepackage[margin=2.5cm]{geometry}
\usepackage[table,dvipsnames]{xcolor}

\usepackage[utf8]{inputenc}
\usepackage[T1]{fontenc}
\usepackage[numbers,sort&compress]{natbib}
\usepackage{hyperref}
\usepackage{url}
\usepackage{booktabs}
\usepackage{amsfonts}
\usepackage{nicefrac}
\usepackage{microtype}
\usepackage{tabularx}

\usepackage{amsmath}
\usepackage{amssymb}
\usepackage{amsthm}

\usepackage{float}

\usepackage{cleveref}

\usepackage{graphicx}
\usepackage{subcaption}
\usepackage{wrapfig}
\usepackage{needspace}

\usepackage{tabularx}
\usepackage{multirow}

\usepackage{tikz}
\usetikzlibrary{calc,arrows.meta}
\usepackage{tikz-3dplot}
\usepackage{tikz-cd}

\usepackage{bm}

\newtheorem{postulate}{Postulate}
\crefname{postulate}{postulate}{postulates}
\Crefname{postulate}{Postulate}{Postulates}

\newtheorem{theorem}{Theorem}
\newtheorem{proposition}{Proposition}
\newtheorem{corollary}[proposition]{Corollary}
\newtheorem{lemma}[proposition]{Lemma}

\newtheorem{definition}{Definition}

\usepackage{tcolorbox}
\usetikzlibrary{calc,positioning}

\usepackage{longtable}

\newtcolorbox{simplebox}[1]{
  colback=white,
  colframe=black,
  coltitle=black,
  colbacktitle=white,
  title=#1,
  boxrule=0.5pt,
  arc=2pt,
  titlerule=0.5pt,
  toptitle=4pt,
  bottomtitle=4pt,
  lefttitle=6pt,
  righttitle=6pt
}

\DeclareMathOperator{\spann}{span}
\DeclareMathOperator{\enc}{enc}
\DeclareMathOperator{\dec}{dec}
\DeclareMathOperator{\topk}{TopK}

\DeclareMathOperator{\relu}{ReLU}
\DeclareMathOperator{\auxk}{AuxK}

\DeclareMathOperator{\spark}{spark}
\DeclareMathOperator{\supp}{supp}
\DeclareMathOperator{\rank}{rank}

\newcommand{\RR}{\mathbb{R}}
\newcommand{\DD}{\mathcal{D}}
\newcommand{\EE}{\mathbb{E}}

\DeclareMathOperator{\gap}{gap}
\DeclareMathOperator{\dist}{dist}

\newcommand{\full}{\texttt{full}}
\newcommand{\wiki}{\texttt{wiki}}
\newcommand{\nowiki}{\texttt{no\_wiki}}
\newcommand{\randomone}{\texttt{rand1}}
\newcommand{\randomtwo}{\texttt{rand2}}

\definecolor{petcolor}{RGB}{184,221,217}
\definecolor{catcolor}{RGB}{247,207,155}
\definecolor{dogcolor}{RGB}{246,184,177}

\newcommand{\threeatommol}[3]{%
\begin{tikzpicture}[baseline=-0.5ex, every node/.style={outer sep=0pt}]
    \node[
        circle,
        draw,
        fill=petcolor,
        minimum size=0.82cm,
        inner sep=0pt
    ] (pet) at (0,0) {$#1$};

    \node[
        circle,
        draw,
        fill=catcolor,
        minimum size=0.58cm,
        inner sep=0pt,
        anchor=210] (cat) at (pet.30) {$#2$};

    \node[
        circle,
        draw,
        fill=dogcolor,
        minimum size=0.58cm,
        inner sep=0pt,
        anchor=150] (dog) at (pet.330) {$#3$};
\end{tikzpicture}%
}

\newcommand{\petcatdogmol}{\threeatommol{\mathrm{pet}}{\mathrm{cat}}{\mathrm{dog}}}

\newcommand{\twoatommol}[2]{%
\begin{tikzpicture}[baseline=-0.5ex, every node/.style={outer sep=0pt}]
    \node[
        circle,
        draw,
        fill=petcolor,
        minimum size=0.62cm,
        inner sep=0pt
    ] (first) at (0,0) {$#1$};

    \node[
        circle,
        draw,
        fill=catcolor,
        minimum size=0.62cm,
        inner sep=0pt,
        anchor=west] at (first.east) {$#2$};
\end{tikzpicture}%
}

\newcommand{\petcatmol}{\twoatommol{\mathrm{pet}}{\mathrm{cat}}}

\newcommand{\petdogmol}{%
\begin{tikzpicture}[baseline=-0.5ex, every node/.style={outer sep=0pt}]
    \node[
        circle,
        draw,
        fill=petcolor,
        minimum size=0.62cm,
        inner sep=0pt
    ] (first) at (0,0) {$\mathrm{pet}$};

    \node[
        circle,
        draw,
        fill=dogcolor,
        minimum size=0.62cm,
        inner sep=0pt,
        anchor=west] at (first.east) {$\mathrm{dog}$};
\end{tikzpicture}%
}

\newcommand{\oneatommol}[1]{%
\begin{tikzpicture}[baseline=-0.5ex, every node/.style={outer sep=0pt}]
    \node[
        circle,
        draw,
        fill=petcolor,
        minimum size=0.72cm,
        inner sep=0pt
    ] {$#1$};
\end{tikzpicture}%
}

\newcommand{\petmol}{\oneatommol{\mathrm{pet}}}

\newcommand{\catmol}{%
\begin{tikzpicture}[baseline=-0.5ex]
    \node[
        circle,
        draw,
        fill=catcolor,
        minimum size=0.72cm,
        inner sep=0pt
    ] {$\mathrm{cat}$};
\end{tikzpicture}%
}

\newcommand{\dogmol}{%
\begin{tikzpicture}[baseline=-0.5ex]
    \node[
        circle,
        draw,
        fill=dogcolor,
        minimum size=0.72cm,
        inner sep=0pt
    ] {$\mathrm{dog}$};
\end{tikzpicture}%
}

\newcommand{\prefix}{\texttt{prefix}}
\newcommand{\cities}{\texttt{cities}}
\newcommand{\gbif}{\texttt{gbif}}

\makeatletter
\newcommand{\titlethanks}[1]{%
  \g@addto@macro\@thanks{%
    \begingroup
      \renewcommand{\thefootnote}{}%
      \footnotetext{#1}%
    \endgroup
  }%
}
\makeatother

\hypersetup{
    colorlinks,
    linktoc=all,
    linkcolor={red!50!black},
    citecolor={blue!50!black},
    urlcolor={blue!80!black}
}

\usepackage{caption}

\usepackage{array,adjustbox}

\let\AppendixAddContentsLine\addcontentsline
\title{\vspace{-0.5cm}A Testable Theory of Atomic Features}
\author{Kenny Peng\qquad Jon Kleinberg \qquad Nikhil Garg}
\date{Cornell University}

\begin{document}

\begingroup
\renewcommand{\thefootnote}{}
\footnotetext{Correspondence to Kenny Peng (kennypeng@cs.cornell.edu).}
\addtocounter{footnote}{-1}
\endgroup

\addtocontents{toc}{\protect\setcounter{tocdepth}{-1}}
\maketitle

\begin{abstract}
    We develop and test a theory of language model representations in which there exist \textit{atomic features}. Our main theoretical insight is that in such a model, sparse dictionaries (e.g., SAEs) of increasing size recover an increasing prefix of the most prevalent atoms in the training data. This ``recovery principle'' yields three testable predictions: many features in small SAEs are shared by all larger SAEs, SAEs trained on different data share features prevalent in both, and sufficiently large SAEs recover both parent and child features. In contrast to conventional wisdom that SAE features are unstable and ``split'' as size increases, we find that these predictions hold on SAEs of sizes ranging from 512 to 131,072 trained on two large embedding models. From a theoretical perspective, our results suggest the promise of a scientific theory of representations based on atomic features. Practically, our results suggest the promise of scaling SAEs.
\end{abstract}

\section{Introduction}

Language models represent text inputs $x$ as vectors $f(x)\in \RR^d$. A wide-ranging set of work has sought to better understand these representations. Here, our approach is to treat the linear representation hypothesis \citep{mikolov2013distributed, mikolov2013linguistic, arora2018linear,nanda2023emergent, li2025language,panickssery2023steering, zou2023representation, turner2023steering, wang2025persona, chen2025persona,elhage2022toy,park2023linear, costa2025flat,garg2026many} as a foundation for a scientific theory. We formalize the hypothesis, mathematically derive implications, and then test these predictions on real language model representations.

The linear representation hypothesis (LRH) roughly says that $f(x)$ can be written as $z_1(x)\vec{a}_1 + z_2(x)\vec{a}_2 + \cdots + z_M(x)\vec{a}_M,$ where we think of $z_1(x), z_2(x), \cdots, z_M(x)$ as describing individual features of $x$---does $x$ mention cats, does $x$ involve algebraic geometry, and so on. Stated succinctly,
\begin{equation}
    f(x) = Az(x).
\end{equation}
We now state a striking first result. Under a natural formalization of the hypothesis \citep{garg2026many}, where the features given by $z(x)$ are accessible from $f(x)$, this decomposition is unique (\Cref{thm:2k}). That is, there exists a set of \textit{atomic} features from which language model representations are built. A scientific theory based on the linear representation hypothesis might therefore be called the \textit{atomic theory}.

\paragraph{}
Of course, \Cref{thm:2k} does not prove the existence of atomic features. On the contrary, one might take the strength of \Cref{thm:2k}'s conclusion as evidence \textit{against} the linear representation hypothesis (LRH). 
Indeed, there is much reason to be skeptical of the existence of atomic features. Sparse autoencoders (SAEs)---which are designed to recover the features posited by the LRH \citep{cunningham2023sparse, bricken2023monosemanticity, gao2024scaling, movva2025sparse, peng2025use}---learn different features depending on the setting. Even SAEs in settings that vary only by random seed learn different features \citep{paulo2025sparse,fel2025archetypal}. Coarse features in small SAEs can ``split'' into granular features in larger SAEs \citep{chanin2026absorption,bussmann2025learning,templeton2024scaling}. How can we square variation in learned SAE features with the atomic theory?

\Cref{thm:recovery-principle}, which we call the recovery principle, presents a solution. It says that under the atomic theory, SAEs of increasing size trained with appropriate sparsity on a distribution $\DD$ recover a growing prefix of the most prevalent atoms in $\DD$. \Cref{thm:hierarchical-recovery-principle} further shows that the recovery principle holds even when features are hierarchically structured. The recovery principle is consistent with observed variation in SAE features, while also establishing a substantive connection between these SAEs and the hypothesized atoms. \Cref{fig:recovery-principle-summary} illustrates this connection. 

Beyond reconciling existing empirical observations with the atomic theory, the recovery principle plays another crucial role in the development of the theory: it makes the theory empirically testable. It yields the following three testable predictions.
\begin{enumerate}
    \item[\textbf{(P1)}] Many features in smaller SAEs have matches in all larger SAEs trained on the same data (corresponding to the most prevalent atomic features in the distribution).
    \item[\textbf{(P2)}] SAEs trained on different data have many matching features (corresponding to the atomic features that are prevalent in \textit{both} distributions).
    \item[\textbf{(P3)}] There exist hierarchies where sufficiently large SAEs recover both parent \textit{and} child features. 
\end{enumerate}
These predictions---while, strictly speaking, consistent with existing results---run against the conventional wisdom that SAE features are unstable and become strictly more granular as SAE size increases. P3 suggests a particularly surprising phenomenon: parent features lost as SAE size increases (as in feature splitting \citep{templeton2024scaling,chanin2026absorption,bussmann2025learning}) should eventually return in even larger SAEs. We call this \textit{mesoscale} feature splitting. (In \Cref{sec:molecular-model}, we give a ``molecular'' explanation that also clarifies the importance of correctly choosing sparsity \citep{chanin2025sparse}.)

\definecolor{atomicblue}{RGB}{151,207,251}
\definecolor{datasetpink}{RGB}{240,161,151}
\definecolor{datasetgold}{RGB}{239,188,62}

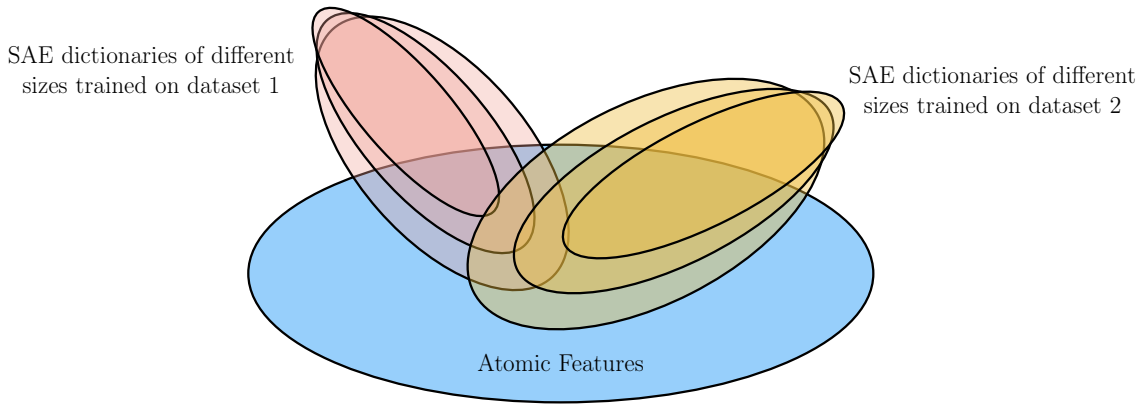
\begin{figure}
\centering
\resizebox{0.9\linewidth}{!}{%
\begin{tikzpicture}[
    x=1cm, y=1cm,
    outline/.style={
        draw=black,
        line width=1.35pt
    },
    pink ellipse/.style={
        outline,
        fill=datasetpink,
        fill opacity=0.33,
        draw opacity=1
    },
    gold ellipse/.style={
        outline,
        fill=datasetgold,
        fill opacity=0.40,
        draw opacity=1
    },
    description/.style={
        align=center,
        font=\Large,
        inner sep=0pt,
        outer sep=0pt
    }
]

\path[outline, fill=atomicblue]
    (0,0) ellipse[x radius=6.499cm, y radius=2.679cm];

\foreach \cx/\cy/\rx/\ry/\ang in {
    -2.466/2.500/3.442/1.775/-48.89,
    -2.814/2.921/3.115/1.288/-48.91,
    -3.224/3.360/2.762/0.904/-48.89
}{
    \begin{scope}[shift={(\cx,\cy)}, rotate=\ang]
        \path[pink ellipse]
            (0,0) ellipse[
                x radius=\rx cm,
                y radius=\ry cm
            ];
    \end{scope}
}

\foreach \cx/\cy/\rx/\ry/\ang in {
    1.770/1.445/4.024/2.082/27.02,
    2.348/1.715/3.643/1.511/26.79,
    2.967/2.049/3.228/1.061/26.69
}{
    \begin{scope}[shift={(\cx,\cy)}, rotate=\ang]
        \path[gold ellipse]
            (0,0) ellipse[
                x radius=\rx cm,
                y radius=\ry cm
            ];
    \end{scope}
}

\node[description] at (-8.515,4.240)
    {SAE dictionaries of different\\
     sizes trained on dataset 1};

\node[description] at (8.980,3.820)
    {SAE dictionaries of different\\
     sizes trained on dataset 2};

\node[font=\Large, inner sep=0pt] at (0,-1.860)
    {Atomic Features};

\end{tikzpicture}%
}
\caption{The recovery principle says that SAEs of increasing size recover a growing prefix of atomic features. This is consistent with observed variation in SAEs, while giving testable predictions that counter conventional wisdom: SAEs of different size and trained on different data share features.}
\label{fig:recovery-principle-summary}
\end{figure}

\paragraph{}
We test P1, P2, and P3 by training SAEs of sizes ranging in powers of two from 512 to 131,072 on two large embedding models: gemini-embedding-2 and llama-embed-nemotron-8b.
\begin{itemize}
\item Testing P1, we find that about 70\% of features recovered by size-4K SAEs simultaneously have matches in SAEs of size 8K, 16K, 32K, 65K, and 131K. 
\item Testing P2, we find that 26-59\% of features are shared across pairs of SAEs trained on five different distributions, compared to mostly 0\% and at most 15\% for a PCA baseline. 
\item Testing P3, we find that on three hierarchies (prefixes, geography, and species), recovery at every level is maximized at size 65K or 131K. In fact, we find instances of the predicted mesoscale splitting, where parent recovery initially decreases, but increases again at larger widths.
\end{itemize}
These results suggest that the atomic theory is a promising scientific theory of language model representations---one that is worth further studying and testing. A practical implication is that by scaling SAEs with appropriate sparsity, we should expect to recover more and more stable atomic features, including both coarse features and granular features. Stepping back, the theory seems to point to a ``rigidity principle'': efficient sparse representations of information seem to be pressured towards shared concepts.
\section{Overview}\label{sec:overview}
We begin by presenting the conceptual ideas of the atomic theory. We start by motivating the formalization of the linear representation hypothesis that serves as the foundation for our work. We then discuss how this formalization leads to atomic features, the recovery principle, and testable predictions. We also give intuition by analogizing to the construction of vocabularies in natural language---noting the apparent pressure towards certain shared concepts.

\paragraph{Formalizing the linear representation hypothesis.} A language model representation is a function $f: X\rightarrow \RR^d$, where $X$ is the input space (e.g., texts) and $f(x)$ is a representation of the input (e.g., neuron activations or an embedding). The main idea of the linear representation hypothesis, we recall, is to write
\begin{equation}\label{eq:lrh}
    f(x) = \sum_{i=1}^M z_i(x)a_i,
\end{equation}
where we call $z_1, z_2, \cdots, z_M: X\rightarrow \RR$ features (which we think of as measuring the presence of individual concepts in the input) and $a_1, a_2, \cdots, a_M\in \RR^d$ directions (directions in the representation space that correspond to each concept). Concretely, we can think of $z_1(x)$ as measuring the extent to which $x$ discusses cats, and $a_1$ to be the direction in which this information is represented within $f(x).$ We call $z_i(x)$ the activation of feature $i$ on input $x$. Taking $a_1, a_2, \cdots, a_M$ to be the columns of $A\in \RR^{d\times M}$ and defining the ``feature representation'' $z: x\mapsto (z_1(x), z_2(x), \cdots, z_M(x)),$ we can simply write $f(x) = Az(x).$

We should also expect features to be \textit{accessible} from a representation \citep{garg2026many}. If $f(x)$ is a sufficiently expressive representation of the input $x$, then $z_1(x), z_2(x), \cdots, z_M(x)$ should be accessible from $f(x)=Az(x)$ alone. For example, given $x = \text{``Djokovic Overwhelms Nadal for the Wimbledon Title''}$, we should be able to learn a number of properties from $f(x)=Az(x)$: that $x$ discusses tennis, that $x$ mentions Novak Djokovic, that $x$ mentions Rafael Nadal, that $x$ discusses a championship match, that $x$ describes an event in England, that $x$ describes a decisive victory, and so on. Formally, we expect there to exist a function $g: \RR^d\rightarrow \RR^M$ such that
\begin{equation}
    g(Az(x)) = z(x),
\end{equation}
i.e., mapping representations to features. However, $M$ greatly exceeds $d$ (embeddings typically have $d$ in the thousands, and there are many more concepts), so $A\in \RR^{d\times M}$ is not invertible in general. However, results in compressed sensing \citep{candes2006robust,candes2006near,donoho2006compressed,garg2026many} say that it is possible when $z(x)$ is sparse and $A$ is well-behaved (which we make precise in Section 3). 

This motivates our formalization of the linear representation hypothesis: that there exists a \textit{well-behaved} matrix $A\in \RR^{d\times M}$ and a $K$\textit{-sparse} function $z: X\rightarrow \RR^M_+$ such that for all $x\in X$, $f(x) = Az(x).$
We formalize this as \Cref{pos:atomic}.

\paragraph{Analogizing to dictionaries.} Sparsity of $z(x)$ means that while the total number of concepts may be enormous, the number of concepts needed to describe a particular input is small (at most $K$). This motivates an analogy to natural language, where we call $A,z$ a dictionary in which the directions $a_i$ are the words.\footnote{Of course, we are drawing from a long line of work on dictionary learning here and throughout \citep{olshausen1996emergence}.} Given a ``situation'' represented by $f(x)$, the idea is to write it as the combination of a small number of words (say, at most $K$). Then $z_i$ specifies in what situations to use the direction $a_i$ (it can be viewed, therefore, as providing a ``definition''). In the setting here, $z_i$ also specifies ``how much'' the word should be used. 

We call $M$ the \textit{dictionary size} (the total number of features) and $K$ the \textit{description length} (the maximum number of features needed to describe an input). Existing work has focused on understanding when dictionaries can be recovered from data for a \textit{fixed} dictionary size and description length \citep{aharon2006uniqueness,hillar2015can,arora2014new,arora2015simple,garfinkle2019uniqueness}. In contrast, the results we introduce here are each investigations of the broader space of potential dictionaries. Are dictionaries unique when size and description length can vary? What does an optimal dictionary look like when you restrict its size? What about when tailored to a particular distribution of situations?

\paragraph{Deriving atomicity and a rigidity property.} 
From the postulate, we can derive the atomic property. \Cref{thm:2k} considers a ``competing'' dictionary $(B,u)$: a well-behaved matrix $B\in \RR^{d\times M'}$ and $K'$-sparse function $u:X\rightarrow \RR^{M'}_+$ satisfying
\begin{equation}
    f(x) = Bu(x)
\end{equation}
for all $x\in X$. Then the result shows that if $z$ is sufficiently rich (exploring each combination of features), then if $B,u$ is Pareto-efficient in dictionary size and description length, then either it is equal to $A,z$ or $K'\ge 2K$ (i.e., it sometimes uses \textit{twice} as many features to describe an input). \Cref{thm:2k} shows that if a sparse system of representation exists, and the space being described is sufficiently rich, then this sparse system is ``rigid.''

To get an intuitive sense of this statement, we draw on the analogy to natural language dictionaries, from which we can state the following informal analog of \Cref{thm:2k}. Consider a dictionary of $M$ words and a set of situations that can each be exactly described by $K$ words. Suppose further that there exist situations spanning each combination of words. Then any dictionary using one fewer word---and that can still perfectly describe each situation---must for some situation use \textit{twice} as many words. (It can be informative to try to play this ``vocabulary construction game'' yourself. We return to this game in a moment.)

\paragraph{From theory to testable predictions.} Up to this point, we have established a pretty but (thus far) hypothetical picture of the underlying structure of language model representations. How can we know if this picture is a useful description of reality? A natural approach is to try and find the hypothesized atoms. To do this, we can try to apply sparse dictionary learning \citep{olshausen1996emergence,arora2014new,makhzani2013k,cunningham2023sparse} to find a matrix $\hat{A}$ and sparse function $\hat{z}$ such that
\begin{equation}
    f(x) \approx \hat{A}\hat{z}(x).
\end{equation}
But major challenges arise in practice. First, the true number of atoms $M$ is unknown (and likely very large). Second, the learned dictionary is shaped by the specific training data. Finally, given a learned dictionary, how are we to know that we have recovered the true atoms?

\paragraph{The recovery principle.} There is a powerful mathematical consequence of \Cref{pos:atomic} that helps resolve these issues. \Cref{thm:recovery-principle} shows that even sparse dictionaries of size $m < M$ that are trained on idiosyncratic (but ``well-behaved'') distributions $\DD$ approximately recover a subset of atoms. In particular, as we increase the size of a sparse dictionary, we recover a growing prefix of the most prevalent atoms in the distribution. Note that the sparsity $K$ must be chosen appropriately to get this guarantee.\footnote{To mix metaphors, the recovery principle suggests that training larger SAEs is something like building larger particle accelerators: by scaling properly, we expect to be able to recover more and more fundamental particles (here, atoms).}

Notice how the recovery principle implies testable implications. First, by training larger SAEs on a fixed distribution, we should see a growing set of ``persistent features''---i.e., features in smaller SAEs that appear in every larger SAE. Second, even if we train SAEs on different data distributions, we should expect to see a set of shared features---the atomic features that are most prevalent in both distributions. This can be contrasted with PCA or SVD, in which learned directions are highly dependent on the data distribution. (SVD had been suggested in the past as a ``solution to Plato's problem'' \citep{landauer1997solution}, but a primary objection was the sensitivity to distributions; SAE features, under the atomic theory, overcome this objection.) The recovery principle and its implications are further illustrations of the rigidity of sparse systems of representation.

\paragraph{Intuition.} Returning to the vocabulary construction game, consider the goal of constructing a dictionary of limited size to best describe situations arising in your daily life using at most $K$ words per situation. (Perhaps, you were confined to tweeting.) You are allowed to invent new words (e.g., referring to combinations of concepts that appear especially frequently for you). The recovery principle says that you would not pick entirely new, invented words; rather you would adopt many existing words---and in particular, the existing words that are most relevant for your own daily life. These words are retained even if you are allowed a larger dictionary.

To give a concrete example, suppose that each day you looked outside your window and saw a tree by a sidewalk. Perhaps, in extremely limited vocabularies, you would define a word ``sidewalktree.'' This would save you budget in dictionary size and in description length. But you would be unable to describe a ``tree'' or ``sidewalk'' on their own. If you wanted to describe an isolated tree, you would need a word that means ``nosidewalk'' and call it a ``nosidewalk'' ``sidewalktree,'' in which case you have now used just as many words and more description length than if you had words for ``tree'' and ``sidewalk.'' In this way, there is pressure towards including the atoms ``tree'' and ``sidewalk'' in a dictionary. (To actually prove the recovery principle requires some more sophisticated reasoning.)

In a sense, of course, humanity has been playing the vocabulary game for millenia. Different societies have come up with languages of different lengths to describe different distributions of situations. Quite strikingly, these languages have an enormous amount of overlap in the concepts given words. In linguistics, a ``Swadesh list'' gives concepts that appear in essentially \textit{all} languages, including concepts like I, you, five, short, knee, to drink, to fly, to count, river, sky, red, green, year, and round (among many others). The recovery principle offers some reasoning for this---and hints at a formal notion of ``universal concepts.''

\paragraph{Recovery of hierarchical features.} The recovery principle shown in \Cref{thm:recovery-principle} makes an assumption that the training distribution is well-behaved: that for every pair of features $z_i,z_j$, there exist inputs for which $z_i$ is active but $z_j$ is not (and vice-versa). In \Cref{thm:hierarchical-recovery-principle}, we show that the recovery principle continues to hold (with an adjustment) for distributions where activations are hierarchically correlated (where a parent feature is active whenever a child feature is active). The result implies that in sufficiently large SAEs, we expect to simultaneously recover parent features and child features. This prediction is especially interesting when compared to existing empirical observations in which a broad parent feature that is recovered in a smaller SAE ``splits'' as SAE size increases. Under our theory, the expected behavior is that as the size of the SAE increases even further (with sparsity chosen appropriately), we should expect the parent feature to once again be recovered.

\section{Theoretical Results}\label{sec:atomic-theory}
We now mathematically present the atomic theory starting from a single postulate and deriving the testable predictions P1-P3 as consequences. It will be useful to first establish a formal definition of a dictionary, which is the primary object we study throughout.

\begin{definition}
An $(m,k,\gamma)$-\textit{dictionary} is a pair $B\in \RR^{d\times m},u: X\rightarrow \RR^m_+$ where $u$ is $k$-sparse and $B\in \RR^{d\times m}$ has unit-norm columns and satisfies the $(\gamma,2k)$-lower restricted isometry property: that is, for all $2k$-sparse $v\in \RR^m$, $\|Bv\|_2 \ge \gamma \|v\|_2$.\footnote{For $d= \Omega_{k,\gamma}(\log m)$, lower RIP holds with high probability for random Gaussians \citep{candes2006near}}
\end{definition}
Recall that $m$ is the dictionary's \textit{size} and $k$ the dictionary's \textit{description length}. The lower restricted isometry property establishes the key property that $u(x)$ can be recovered from $Bu(x).$ Indeed, if $Bu(x) = Bu'$ for some $k$-sparse vector $u'$, then $\|u(x) - u'\|\le \frac{1}{\gamma}\|Bu(x) - Bu'\|_2 = 0$, so $u(x) = u'$. The dictionaries we study are \textit{overcomplete} ($m > d$) and \textit{sparse} ($k < d$). In what follows, we take $\gamma \in (0,1)$ to be a fixed constant and refer to an $(m,k,\gamma)$-dictionary as simply an $(m,k)$-dictionary. 

We can then succinctly state our foundational postulate about $f$.
\begin{postulate}[The Linear Representation Hypothesis]\label{pos:atomic}
    There exists an $(M,K)$-dictionary $A,z$ such that
    \begin{equation}
        f(x) = Az(x)
    \end{equation}
    for all $x\in X$.
\end{postulate}
We refer to $A,z$ as the atomic dictionary, with atomic directions $a_1, a_2, \cdots, a_M$ and atomic features $z_1, z_2, \cdots, z_M.$ Our first task is to justify calling these features atomic.

Past work has shown that if $z$ is sufficiently ``rich,''  then $A,z$ is the unique dictionary with dictionary size and description length exactly equal to $M$ and $K$ \citep{aharon2006uniqueness,hillar2015can,garfinkle2019uniqueness}. Here, we will establish a notion of uniqueness that is quite a bit broader. We say that an $(m,k)$-dictionary is \textit{efficient} if, compared to the atomic dictionary, any \textit{one} of the following three conditions holds:
\begin{enumerate}
    \item[(a)] it has the same dictionary size and description length (i.e., $m=M, k=K$);
    \item[(b)] it has smaller description length (i.e., $k < K$);
    \item[(c)] it has smaller dictionary size and less than twice the description length (i.e., $m < M$ and $k < 2K$).
\end{enumerate}
Now define richness as follows. For $S\subseteq [M]$ with $|S|=K$, let $E_S = \spann\{e_i: i\in S\}$ and $Z_S = z(X) \cap E_S$ (the set of feature activations with support $S$ realized by $X$). We call $z$ $K$-rich if $Z_S$ is not contained in any union of a finite number of proper subspaces of $E_S$.

\begin{theorem}[Unique Efficient Dictionary]\label{thm:2k}
    Suppose $z$ is $K$-rich. If an efficient dictionary $B,u$ satisfies
    \begin{equation}
        f(x) = Bu(x)
    \end{equation}
    for all $x\in X$, then $B,u=A,z$ up to a permutation of indices.
\end{theorem}
Therefore, no dictionary that exactly describes $f$ can have smaller description length than $A,z$, and any smaller dictionary that exactly describes $f$ must have \textit{double} the description length. One way to view this result is that if a set of inputs can be described by a sparse dictionary, then this sparse dictionary is \textit{rigid}.\footnote{A result by \citet{sitharam2014incidence} establishes a related rigidity result: that one cannot achieve a constant fraction of dictionary size without description length blowing up to be on the order of embedding dimension \citep{sitharam2014incidence}. We state this result as \Cref{thm:generic-rigidity} and give an alternate proof using Sard's theorem.}

The main challenge of the proof is to show that $k < 2K \implies m\ge M$. The key idea is to construct a bipartite graph between directions in the two dictionaries, where $a_i$ and $b_j$ share an edge if $b_j$ is used in the reconstruction of $a_i$ (we will first have to show that $a_i$ has a $k$-sparse representation in $B$). Then, we can show that if $k<2K$, Hall's condition is satisfied for each subset of $A$ directions of size at most $2K$; moreover, no subset of at most $K$ $A$ directions has a neighborhood of size at least $2K$. We then show that these two conditions imply Hall's condition to hold globally. (Interestingly, this appears to be a novel sufficient condition.) Therefore, there exists a one-to-one matching including all $A$ directions, showing that $m\ge M$.

\paragraph{The recovery principle.}
We now establish the recovery principle: that even undersized SAEs trained to approximate idiosyncratic distributions of inputs recover atoms. This gives testable predictions P1 (features in small SAEs persist in larger SAEs) and P2 (SAEs trained on different data learn overlapping features). 

The recovery principle concerns dictionaries that (1) are undersized (i.e., $m<M$) and (2) only approximately describe $f$ on some distribution $\DD$ over $X$. The recovery principle shows conditions under which such dictionaries still recover atoms. Define the dictionary learning objective
\begin{equation}
    \mathcal{L}_\DD (B,u) := \EE_{x\sim \DD}\|f(x) - Bu(x)\|_2^2
\end{equation}
over the reference class of $(m,K)$-dictionaries. Notice that dictionaries in this reference class have description length matching the atomic dictionary; we show in simulations in \Cref{sec:recovery-simulations} that choosing appropriate $k$ is important, but that there is some flexibility. Let $\mathcal{L}_\DD^*(m)$ be the optimal loss over this class. Taking a fixed constant $C_0 > 1,$ let $\hat{A}^{(\DD,m)},\hat{z}^{(\DD,m)}$ be an $(m,K)$-dictionary satisfying
\begin{equation}
    \mathcal{L}_{\DD}(\hat{A}^{(\DD,m)},\hat{z}^{(\DD,m)}) \le C_0\mathcal{L}_{\DD}^*(m).
\end{equation}
Conceptually, this is the kind of dictionary that a sparse autoencoder can learn in practice, obtaining loss within a constant factor of optimal. We are interested in if and when $\hat{A}^{(\DD,m)},\hat{z}^{(\DD,m)}$ recovers atomic features and directions. Formally, we define recovery as follows:
\begin{itemize}
    \item Direction recovery: $\hat{A}$ $\eta$-recovers $a_i$ if there exists $j\in [m]$ such that $\langle \hat{a}_j, a_i \rangle > \eta.$
    \item Activation recovery: $\hat{z}$ $(\eta,\DD)$-recovers $z_i$ if there exists $j\in [m]$ and a threshold $\tau$ such that $\operatorname{F1}_\DD(\mathbf{1}\{z_i > 0\}, \mathbf{1}\{\hat{z}_j > \tau\}) > \eta.$
\end{itemize}
We show that recovery can happen as long as the distribution $\DD$ is well-behaved. As an example of why some kind of assumption on $\DD$ is necessary, notice that if two atomic features $z_i, z_j$ were exactly co-occurring on a distribution, then it would not be possible to disentangle them. 

This motivates the following definition. A distribution $\DD$ over $X$ is $(\rho, c_-, c_+)$-regular if the following conditions hold for $x\sim \DD.$ First, as a simplification, let $\supp(z(x))$ have size $K$ for all inputs. Let $h_S:(0,1)^K\rightarrow \RR_+$ be the density of $z_S(x)$ conditional on the event $\supp(z(x))=S.$ Then the two conditions are
\begin{itemize}
    \item $\Pr[j\notin \supp(z(x))\,|\,i\in \supp(z(x))] \ge \rho$ for all pairs $i,j\in [M]$ with $i\neq j$.
    \item $c_- \le h_S(w) \le c_+$ for all $w\in (0,1)^K.$
\end{itemize}
The idea of the recovery principle, then, is that if $\DD$ is regular, then undersized sparse dictionaries can recover the most prevalent atoms in $\DD$. 

Given $\DD$, set $p_i := \Pr[i\in \supp(z(x))]$, the probability that an atom $i$ is active for a given input. Choose $\sigma: [M]\rightarrow [M]$ such that $p_{\sigma(1)}\ge p_{\sigma(2)} \ge \cdots \ge p_{\sigma(M)},$ determining the order of prevalence. Intuitively, we may expect these probabilities to follow a power-law distribution $p_{\sigma(i)} \propto i^{-\alpha}$. Define $T(m) = \sum_{j = m+1}^M p_{\sigma(j)}$, which we can call the tail mass beyond $m$. Then \Cref{thm:recovery-principle} below says that an atom is recovered in a dictionary of size $m$ whenever its prevalence exceeds the tail mass $T(m)$ by a constant factor.

\begin{theorem}[The Recovery Principle]
\label{thm:recovery-principle}
    Let $\DD$ be $(\rho, c_-, c_+)$-regular. Then for every $\eta < 1$, there exists a constant $C$ not depending on $m$ or $M$ such that for all $i\in [M]$ satisfying
    \begin{equation}
        p_{i} \ge CT(m),
    \end{equation}
    $\hat{A}^{\DD,m}$ $\eta$-recovers $a_i$ and $\hat{z}^{\DD,m}$ $(\eta,\DD)$-recovers $z_i$
\end{theorem}
The result guarantees recovery of atoms that appear more frequently than the tail mass beyond the $m$-th most frequent atom. For example, if atom prevalences follow a power-law distribution with $\alpha > 1$, the result guarantees that at least the $\Omega(m^{1 - 1/\alpha})$ most prevalent atoms are recovered. In \Cref{sec:recovery-simulations}, we show that in simulations, a much larger fraction of atoms can be recovered---suggesting that the bound shown here is not tight (one reason, perhaps, is that the bound here applies for any learned dictionary with low enough loss, and does not take advantage of the structure of \textit{optimal} learned dictionaries). It would be exciting to prove stronger versions of the recovery principle.

To provide some intuition for the result, we sketch the proof. The main lemma shows that if there exists an atomic support $S\subseteq[M]$ such that the learned dictionary approximates representations that lie in $\spann(A_S)$, then there must exist a learned support $T_S\subseteq[m]$ such that $\spann(\hat{A}_{T_S})\approx \spann(A_S).$ This is non-obvious, since in principle, points in $\spann(A_S)$ may have different optimal $K$-sparse supports in $\hat{A}$. An initial observation is that if we can show that a single learned support is being used to re-construct a large enough subset of $\spann(A_S)$, then the span of this support must be aligned with $\spann(A_S)$. This approach does not work immediately, since there are $\binom{m}{K}$ possible supports, so each learned support may only be used for a vanishing fraction of the span as $m$ increases. A key insight is that the lower restricted isometry property implies that optimal supports do not change too much for nearby points, allowing us to iteratively identify a large enough ball in which a single learned support is used. From this lemma, we can then intersect recovered supports to show the recovery of individual atoms. Finally, any sufficiently optimal dictionaries---when the tail mass is small---must recover atomic supports spanned by atoms in the head of the distribution.

\paragraph{}
Predictions P1 and P2 follow as corollaries. Both concern the existence of matching features across dictionaries. For matrices $U$ and $V$, we define $\mu_t(U,V)$ to be the size of the largest one-to-one matching between columns of $U$ and $V$ such that matches have cosine similarity at least $t$. Let $\mu_t(U;V_1,V_2,\cdots,V_s)$ be the maximum number of columns of $U$ that can be simultaneously matched to every $V_1,V_2,\cdots,V_s$. Then for $t < 1$, there exists a constant $C_t$ such that for $R_{\DD,t}(m) := \max\{r\in [M]: p_{\sigma(r)}\ge C_t\sum_{j=m+1}^M p_{\sigma(j)}\}$ the following results hold.

\begin{corollary}[P1: Stability Across Size]
\label{cor:stability-size}
For a regular distribution $\DD$ and dictionary sizes
$m_1,\ldots,m_s \ge m$,
\[
\mu_{t}\!\left(
    \hat A^{(\mathcal D,m)};
    \hat A^{(\mathcal D,m_1)},
    \ldots,
    \hat A^{(\mathcal D,m_s)}
\right)
\ge R_{\DD,t}(m).
\]
\end{corollary}
That is, there exists a persistent set of features across all dictionaries of size at least $m$.

\begin{corollary}[P2: Stability Across Data]
\label{cor:stability-data}
For any regular distributions $\DD$ and $\DD'$ on the same atomic representation,
\[
\mu_{t}\!\left(
    \hat A^{(\mathcal D,m)};
    \hat A^{(\mathcal D',m)}
\right)
\ge
\bigl|\sigma([R_{\DD,t}(m)])\cap\sigma'([R_{\DD',t}(m)])\bigr|.
\]
\end{corollary}
Here, $\sigma([R_{\DD,t}(m)])\cap\sigma'([R_{\DD',t}(m)])$ is the set of atoms in both distributions' guaranteed prefixes. So dictionaries trained on different distributions have shared features---the atoms prevalent in both distributions.

\paragraph{The hierarchical recovery principle.}
To establish P3 (that sufficiently large SAEs recover both parent features and child features), we show that the recovery principle holds under distributions where atomic feature activations have hierarchical correlation: i.e., when the activation of a child atom implies the activation of a parent atom. This violates the regularity condition on $\DD$ used in \Cref{thm:recovery-principle}.

We consider a setting with $N$ families, each of which has one parent atom and two child atoms. It is possible to generalize to more than two child atoms. We let $M = 3N$ and $K = 2L$, and re-index the atoms as $i, i1, i2$ for $i\in [N]$. Here, the parent feature $z_i$ is always active when the child features $z_{i1}$ or $z_{i2}$ are active. We then establish a definition of \textit{hierarchical regular} distributions. Under this definition, parent features satisfy the standard regularity assumptions, exactly one child feature activates at a time, and the space of parent-child activation pairs is well-explored. The full definition is given in \Cref{sec:recovery-proofs}.

We then define analogs of quantities used in the regular recovery principle. For a given $\DD$ and a family $i\in [N]$, let $p_i = \Pr[i\in \supp(z(x))]$ and $\sigma: [N]\rightarrow [N]$ sort the families in monotonically decreasing order of prevalence according to occurrence of the parent feature (so $\sigma(1)$ is the most prevalent family). Define $T(n) = \sum_{j=n+1}^N p_{\sigma(j)},$ the tail mass beyond the $n$-th family.

Then \Cref{thm:hierarchical-recovery-principle} below shows that an analog of \Cref{thm:recovery-principle} holds in hierarchical settings for \textit{activation recovery}. In other words, we should expect to be able to recover features that activate in the presence of the parent, and also features that activate in the presence of each child, yielding P3.

\begin{theorem}[The Hierarchical Recovery Principle]
\label{thm:hierarchical-recovery-principle}
    Let $\DD$ be hierarchical regular. Then for every $\eta < 1$, there exists a constant $C$ not depending on $m$ or $M$ such that $\hat{z}^{\DD,m}$ $(\eta,\DD)$-recovers the family of features $z_i, z_{i1}, z_{i2}$ for all $i\in [N]$ satisfying
    \begin{equation}
        p_{i} \ge CT\left(\left\lfloor\frac{m}{3}\right\rfloor\right).
    \end{equation}
\end{theorem}
The following identity provides useful intuition. If $z_{j}(x) > 0 \implies z_i(x) > 0$, then for all $\lambda \ge 0$ we can re-write
\begin{equation}\label{eq:hierarchy-substitution}
    z_i(x)a_i + z_j(x)a_j = \hat{z}_i(x)\hat{a}_i + \hat{z}_j(x)\hat{a}_j,
\end{equation}
where $\hat{z}_i(x) = z_i(x) + \lambda z_j(x), \hat{a}_i = a_i, \hat{z}_j(x) = z_j(x),$ and $\hat{a}_j = a_j - \lambda a_i$.\footnote{This identity is studied by \citet{chanin2026absorption} as a potential source of feature absorption; an important distinction in our assumptions is that we allow child feature activations to be larger than the parent, meaning that $\hat{z}_i = z_i(x) + \lambda z_j(x) \ge 0 \implies \lambda \ge 0$.} This preserves sparsity and non-negativity of $u$. More importantly, both $\hat{z}_i$ and $\hat{z}_j$ activate on exactly the same examples as $z_i$ and $z_j$. Therefore, even though atomic directions and features are not strictly unique under hierarchical correlation, alternative parameterizations are still equivalent in terms of \textit{when} each feature is active.

\paragraph{Mesoscale feature splitting.} The prediction of simultaneous recovery of both parent and child features at large SAE sizes seems in tension with empirical results that have demonstrated feature splitting and absorption \citep{chanin2026absorption,bussmann2025learning,templeton2024scaling}. P3 suggests that as SAE size continues to increase, we should expect lost parent features to return. 

We give a ``molecular'' model of this phenomenon in \Cref{sec:molecular-model}, in which small SAEs learn a family molecule, medium SAEs learn individual parent-child molecules (this is the regime where a learned feature matching the parent's activation is lost), and large SAEs learn an individual parent atom and individual child atoms. This molecular model can explain further phenomena, such as polysemanticity and semantic ``clusters'' of SAE directions.

\begin{figure}[p]
    \centering
    \captionsetup{position=top,skip=6pt}
    \caption{Testing P1, P2, and P3---the predictions of the atomic theory.}
    \label{fig:main-experiments}
    \captionsetup[subfigure]{position=bottom,font=footnotesize,skip=3pt,justification=justified,singlelinecheck=false}
    \begin{subfigure}{\linewidth}
        \centering
        \caption{\textbf{P1: Persistent features across size.} Number and proportion of directions with matches in every larger SAE. Similarity threshold $t=0.7$. About 70\% of 4K SAE directions are retained in size-8K, 16K, 32K, 65K, and 131K SAEs.}
        \label{fig:persistent-stability}
        \includegraphics[width=0.8\linewidth]{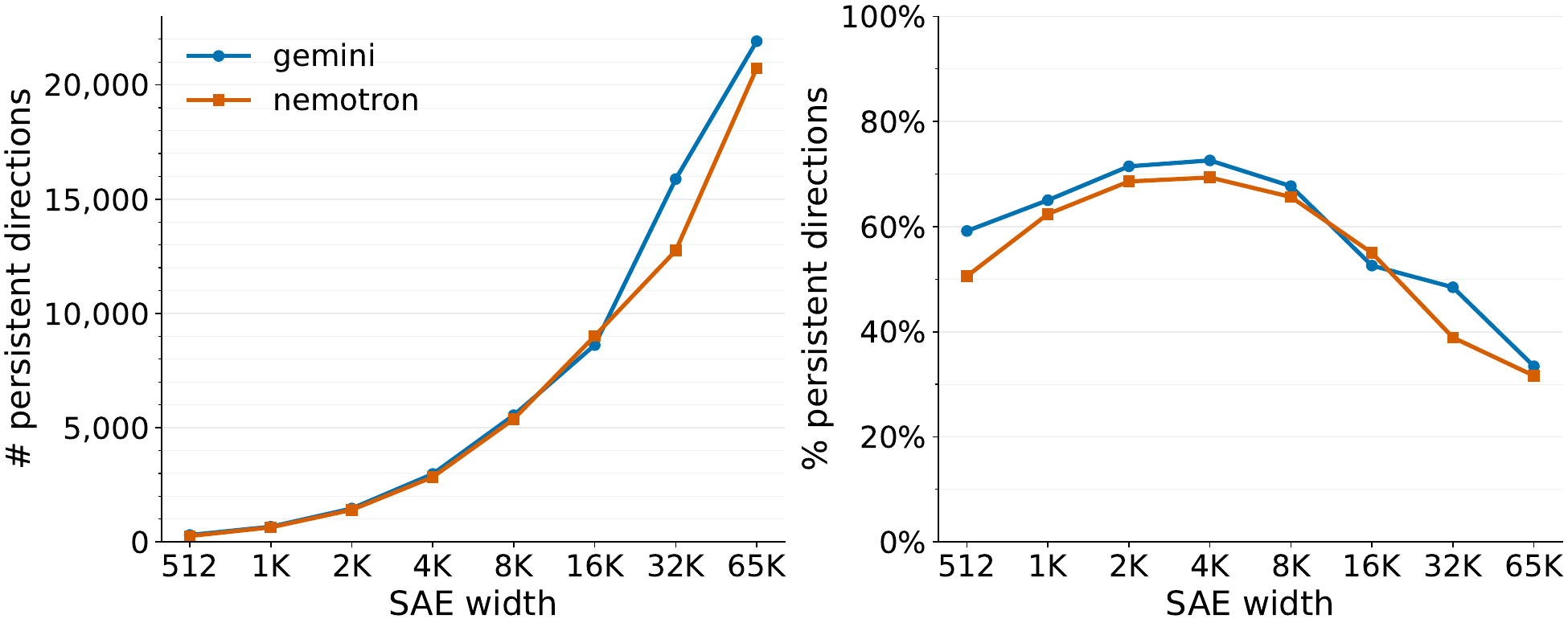}
    \end{subfigure}
    \par\medskip
    \begin{subfigure}{\linewidth}
        \centering
        \caption{\textbf{P2: Stability over data distributions.} Proportion of directions shared between SAE and PCA dictionaries trained on different distributions. Similarity threshold $t=0.7$. We use absolute cosine similarity for PCA directions. 26\%-59\% of SAE feature directions have matches, compared to mostly 0\% and at most 15\% for PCA.}
        \vspace{0.5cm}\includegraphics[width=\linewidth]{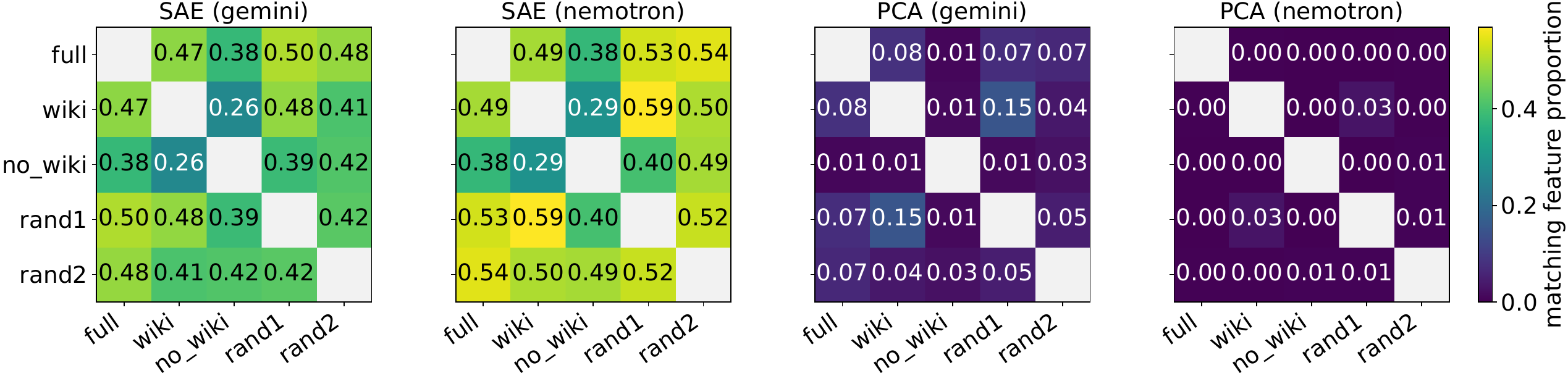}
        \label{fig:dataset-heatmaps}
    \end{subfigure}
    \par\medskip
    \begin{subfigure}{\linewidth}
        \centering
        \caption{\textbf{P3: Hierarchical recovery.} Mean held-out test F1 by hierarchy level and SAE width for Prefix, Cities, and GBIF. Each curve is an unweighted mean over category--model pairs, pooling Gemini and Nemotron. Average recovery is highest across all levels at sizes 65K and 131K.}
        \label{fig:hierarchy-mean-f1}
        \includegraphics[width=\linewidth]{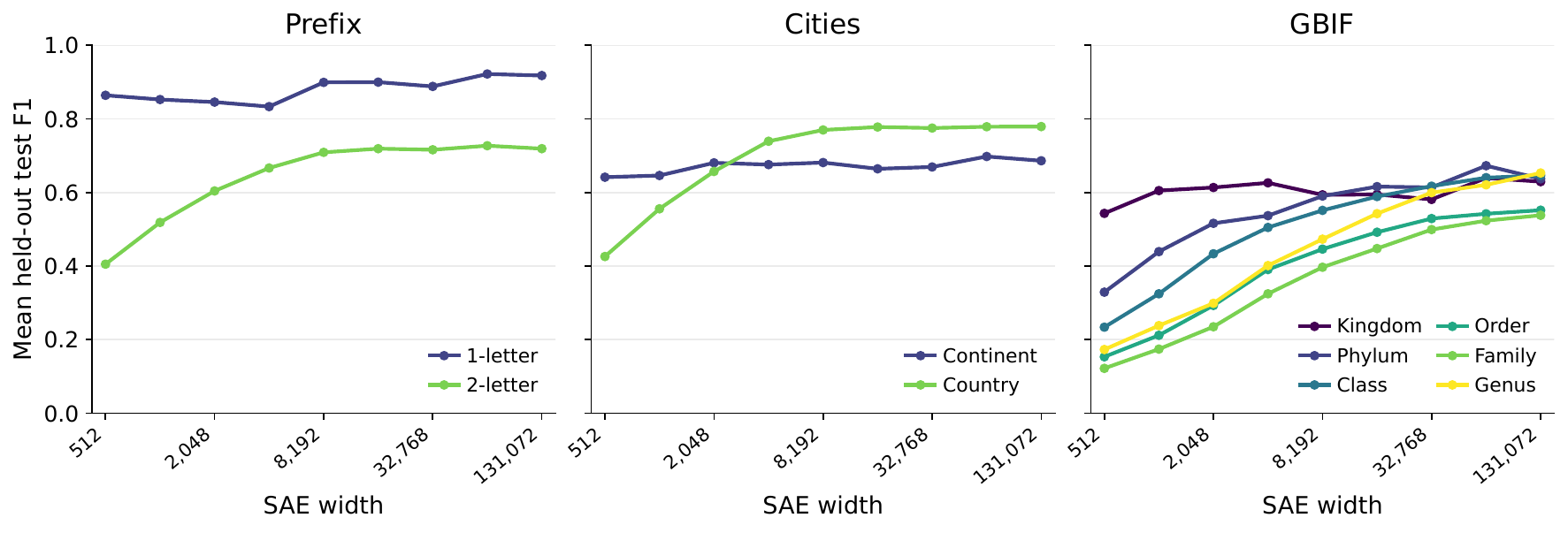}
    \end{subfigure}
\end{figure}

\section{Experiments: Testing P1, P2, and P3}\label{sec:experiments}

We now test the predictions P1, P2, and P3 derived from \Cref{pos:atomic}. We study two large embedding models, gemini-embedding-2 \citep{shanbhogue2026gemini} and llama-embed-nemotron-8b \citep{babakhin2025llama}, which give representations $f(x)$ of dimension $d = 3072$ and $d = 4096$, respectively. We train $k$-sparse autoencoders \citep{makhzani2013k} on embeddings derived from 90M text inputs, comprising 13 datasets used to fine-tune llama-embed-nemotron-8b. An SAE of size $m$ and trained with sparsity $k$ produces a dictionary $\hat{A}\in \RR^{d\times m},\hat{z}:X\rightarrow \RR^m$, where $\hat{A}=W_{\text{dec}}$ is the SAE decoder and $\hat{z}(x)=\topk_k\!\left(\relu\!\left(W_{\enc}(f(x)-b_{\dec})\right)\right)$ are the hidden layer activations. We train SAEs of sizes $m$ in powers of two ranging from 512 to 131,072, with $k=32,32,32,32,64,64,64,128,128$ respectively. (\Cref{sec:p1-extended} reproduces results with $k=128$ fixed.)

\Cref{fig:main-experiments} displays our main results. Details of our experimental setup are given in \Cref{sec:setup}. Extended analyses for P1, P2, and P3 are collected in \Cref{sec:p1-extended,sec:p2-extended,sec:p3-extended}, respectively.

\paragraph{Computing matching features.}
To test P1 and P2, we need to compute matches across dictionaries. We use the following metrics. Given a bipartite graph, a \textit{maximum matching} is a one-to-one matching with the largest number of matches. The \textit{exclusive matching} is the one-to-one matching consisting of all edges whose endpoints have only one neighbor. Given dictionaries $\hat{A},\hat{z}$ and $\hat{A}',\hat{z}',$ we draw edges between directions that have cosine similarity at least $t$ and edges between features that have Pearson correlation at least $t$ over a specified distribution. We then compute the maximum and exclusive matching. Given a matching criterion (maximum or exclusive), we compute a matching independently between the source dictionary and each larger dictionary. A direction or feature is \textit{persistent} if it is matched in every comparison. We use a threshold of $t=0.7$ as our default, following \citet{paulo2025sparse}.

\paragraph{P1: Persistent features across size.}
We start by testing P1, which says that many features in smaller SAEs have matches in all larger SAEs. Many directions and features have matches in \textit{all} larger SAEs. Under independent maximum matching at threshold $t=0.7$, in SAEs of size 4,096, $72.6\%$ of Gemini directions and $69.4\%$ of Nemotron directions are persistent (\Cref{fig:persistent-stability}). By activation similarity, $59.7\%$ of Gemini features and $66.4\%$ of Nemotron features are persistent (\Cref{fig:persistent-stability-pearson}). \Cref{sec:p1-extended} shows results under exclusive matching and under different thresholds.

We also compare feature prevalence across SAE sizes. Persistence suggests that larger SAEs do not learn strictly more granular features. We find that larger SAEs learn a similar number of coarse features as smaller SAEs while learning a longer tail of granular features; in contrast, KMeans clusters become strictly smaller as the number of clusters increases (\Cref{fig:activation-rates}). Consistent with recovering prevalent atomic features, more prevalent features are more likely to be persistent (\Cref{fig:width-prevalence-median}).

\paragraph{P2: Shared features across training data.}
We test P2, which says that SAEs trained on different data learn overlapping features. We use four additional data splits: \wiki, \nowiki, \randomone, and \randomtwo. The pair (\wiki,\nowiki) contain separate datasets (\wiki\,contains datasets that are built using Wikipedia data, and \nowiki\,contains the datasets that are not). The pair (\randomone, \randomtwo) contain separate datasets as well, but the partition is chosen randomly rather than by content type. None of the four data splits contain the multilingual dataset MIRACL. We train SAEs with $m=16384$ and $k=64$ on each distribution. We compare to PCA baseline dictionaries, consisting of the $d$ principal components fitted to each distribution. To compute matchings between PCA directions, we use absolute cosine similarity, since the signs of principal components are arbitrary.

\Cref{fig:dataset-heatmaps} shows the proportion of features shared across pairs of different dictionaries trained on different datasets for $t=0.7$. SAE dictionaries have many cross-dataset matches (between $26\%$ and $59\%$), while PCA dictionaries have substantially fewer (at most $15\%$, median $1\%$). The lower matching rates in the PCA dictionaries suggest that the high amount of overlap in SAE dictionaries is not explained only by similarity in the underlying data distributions. \Cref{fig:dataset-heatmaps-activation} also shows that SAEs have many cross-distribution matches based on activation similarity (22-61\%). Under exclusive matching, SAE dictionaries continue to have many matching directions (25-57\%) whereas the PCA match proportion remains at most 15\% (\Cref{fig:dataset-heatmaps-exclusive}).

\Cref{fig:split-prevalence-median} shows that features with greater prevalence in both distributions are more likely to be selected in direction matching, consistent with the recovery of overlapping prevalent atoms.

\paragraph{P3: Hierarchical feature recovery.}
We test P3 by studying how well SAEs of different sizes recover a ground-truth set of hierarchical features defined by three datasets: \prefix, \cities, \gbif. 

For each ground-truth binary feature $z_\ell^*: X\rightarrow \{0,1\}$ defined over inputs from each dataset, and for each SAE size from 512 to 131K trained on \texttt{full}, we report
the test F1 score $\operatorname{F1}_{\mathcal E_{\mathrm{test}}}\!(z_\ell^*,\mathbf{1}\{\hat z_j\geq\tau\}),$
where $j\in[m]$ and $\tau>0$ are selected to maximize the
same score under $\mathcal E_{\mathrm{train}}$. Here $\mathcal E_{\mathrm{train}}$ and $\mathcal E_{\mathrm{test}}$
are the empirical distributions on the train and held-out
test splits of the labeled dataset. Dataset and experiment details are given in \Cref{sec:p3-hierarchy-protocol}.
\Cref{fig:hierarchy-mean-f1} shows that across levels of the hierarchy, mean test F1 generally improves with size and is highest at width 65K or 131K. \Cref{fig:adjacent-pair-recovery} confirms that this corresponds to increasing \textit{simultaneous} recovery of parent and child features.

We also find examples of mesoscale feature splitting. \Cref{fig:a-prefix} shows that F1 recovery of the parent `a' prefix feature decreases as SAE size increases initially (splitting), but increases again at even larger sizes. (\Cref{fig:gemini-prefix-f1-grid} shows all 25 Gemini letter prefixes with eligible two-letter children, including `a' and excluding `x'.) However, this effect disappears when we increase the sparsity of smaller widths to match the largest widths, suggesting that feature splitting may primarily be the result of over-constrained sparsity. See \Cref{sec:molecular-model} for more discussion.

In addition to testing how well SAEs recover ground-truth hierarchies, we also extract hierarchical structure directly by examining correlation in feature activations over the corpus \texttt{full} (\Cref{sec:empirical-hierarchy-extraction}). We find 114,762 parent-child pairs in the 131K gemini SAE. \Cref{fig:learned-hierarchies-compact} shows examples of extracted three-level hierarchies, with top-level features ``Food'' and ``Music.'' Example lineages include ``Food'' $\rightarrow$ ``Potatoes'' $\rightarrow$ ``Hash browns / rosti'' and ``Music'' $\rightarrow$ ``Jazz'' $\rightarrow$ ``John Coltrane.''

\begin{figure}[t]
    \centering
    \begin{subfigure}[t]{0.345\linewidth}
        \vspace{-15pt}
        \centering
        \includegraphics[width=\linewidth,height=2.15in,keepaspectratio]{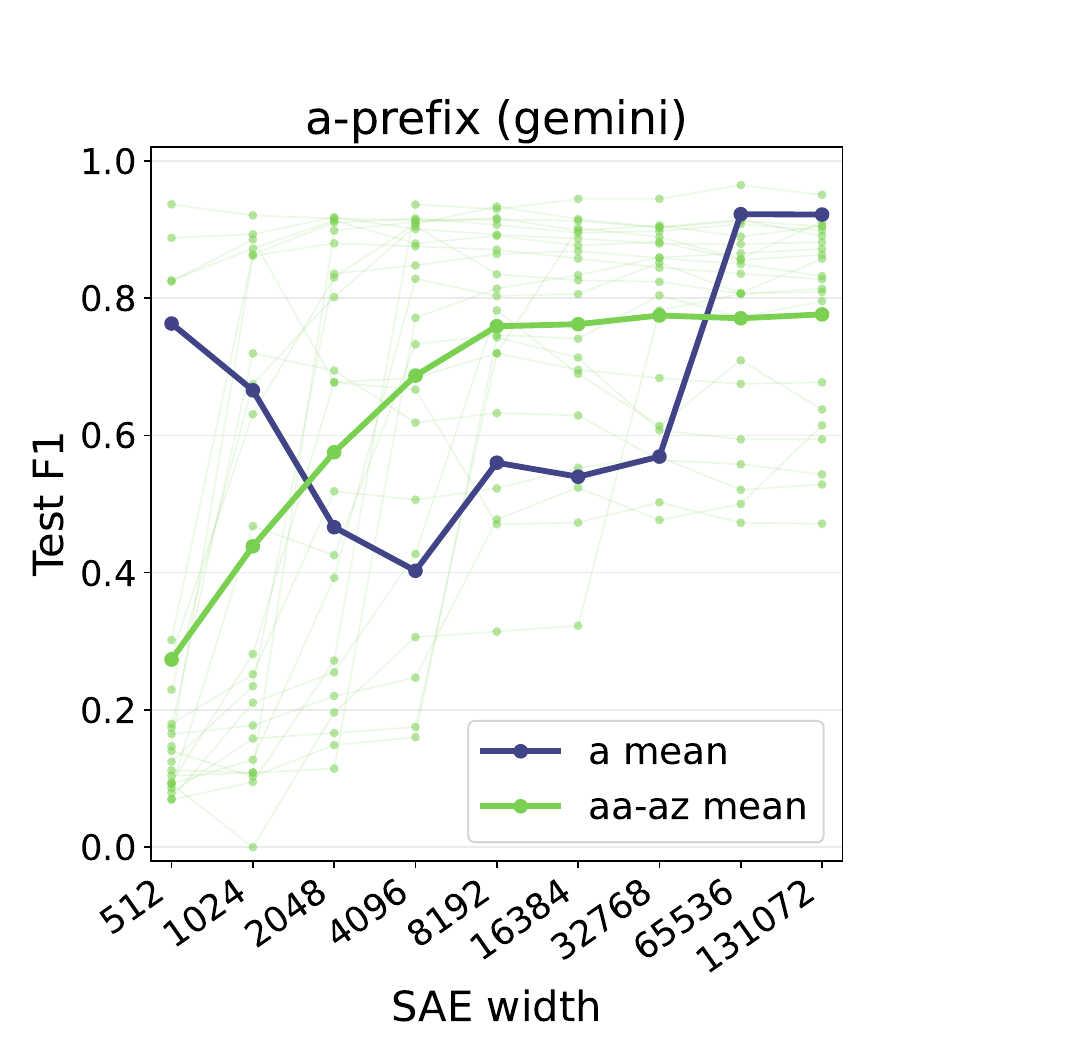}\\
        \vspace{-26pt}
        \includegraphics[width=\linewidth]{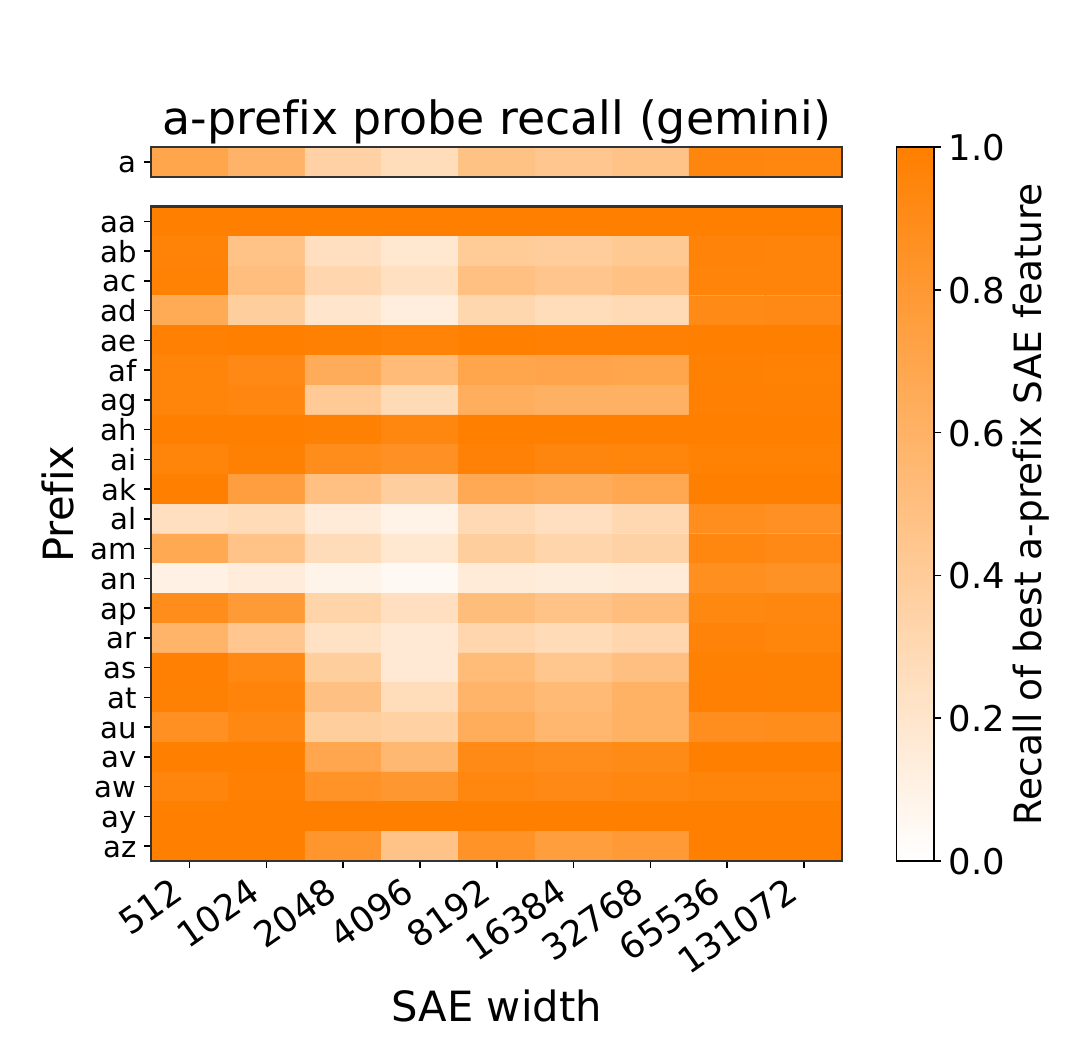}
        \caption{Mesoscale feature splitting.}
        \label{fig:a-prefix}
    \end{subfigure}\hfill
    \begin{subfigure}[t]{0.645\linewidth}
        \vspace{0pt}
        \centering
\begingroup
\scriptsize
\setlength{\tabcolsep}{2pt}
\renewcommand{\arraystretch}{1.06}
\begin{tabularx}{\linewidth}{@{}>{\raggedright\arraybackslash}p{0.12\linewidth}>{\raggedright\arraybackslash}p{0.22\linewidth}>{\raggedright\arraybackslash}X@{}}
\toprule
\textbf{Level 1} & \textbf{Level 2} & \textbf{Level 3} \\
\midrule
\multirow[t]{6}{*}{\textbf{Food}} & Cheese & Cheddar; Mozzarella / related cheeses; Blue cheeses; Gruyère / Emmental; Brie / Camembert \\
\addlinespace[2pt]
 & Potatoes & Potato cultivars; Mashed potatoes; Baked potatoes; Hash browns / rösti; Tater tots \\
\addlinespace[2pt]
 & Fast food & Burger King; Wendy’s; Jack in the Box; Hardee’s; McDonald’s burgers \\
\addlinespace[2pt]
 & Rice dishes & Biryani; Nasi goreng; Fried rice \\
\addlinespace[2pt]
 & Sausages & Chorizo; Salami; German sausages; Hot dogs \\
\addlinespace[2pt]
 & Nutritional information & Fast-food calorie questions; Calorie-breakdown template; Percent Daily Value text; Vitamins / minerals \\
\midrule
\multirow[t]{6}{*}{\textbf{Music}} & Jazz & Charlie Parker / bebop; John Coltrane; Miles Davis; Mahavishnu / John McLaughlin; Marsalis family \\
\addlinespace[2pt]
 & Hip-hop / rap & N.W.A; Eminem; Snoop Dogg; Tupac Shakur; Lil Wayne \\
\addlinespace[2pt]
 & Country music & George Jones; Rascal Flatts; Patty Loveless / Vince Gill; Hank Williams \\
\addlinespace[2pt]
 & Rock / metal & Hard-rock musicians / bands; Nu / alternative metal; Power / symphonic metal; Doom / stoner metal \\
\addlinespace[2pt]
 & Electronic / dance & Detroit techno cluster; Trance mixes / compilations; Basement Jaxx; deadmau5 \\
\addlinespace[2pt]
 & The Beatles & Please Please Me / early Beatles; Hey Jude; Let It Be / Get Back; Abbey Road / late Beatles \\
\bottomrule
\end{tabularx}
\endgroup

        \caption{Learned feature hierarchies.}
        \label{fig:learned-hierarchies-compact}
    \end{subfigure}
    \caption{\textbf{Hierarchical feature recovery.} \textbf{(a)} Top: parent and child F1 for a-prefixes. Parent recovery follows a U-shape. Bottom: recall of the selected a-prefix feature, broken down by prefix. \textbf{(b)} Extracted Food and Music hierarchies in the Gemini 131K SAE. Activating examples for each feature are given in \Cref{tab:food-hierarchy,tab:music-hierarchy}.}
    \label{fig:hierarchy-examples}
\end{figure}

\section{A Molecular Model}\label{sec:molecular-model}
We now sketch a molecular model that helps explain several additional empirical phenomena: polysemanticity, feature splitting, and semantic clusters of directions \citep{templeton2024scaling}. In the model, ``molecules'' are learned SAE directions that are linear combinations of atomic directions. The recovery principle says that molecules in small SAEs eventually split into stable atoms in larger SAEs. A prediction, then, is that since the above empirical phenomena stem from molecules, they disappear at sufficiently large widths. The results and explanations here are intended to be illustrative---we hope that future work formalizes and systematically tests the molecular extension of the atomic theory.

The intuition for molecules can be conveyed with a minimal example. If two atomic features $z_i(x)$ and $z_j(x)$ are heavily correlated over $\DD$, then it can be efficient for $\hat{A}$ (trained to minimize $\mathcal{L}_\DD$) to mix the atomic directions $a_i$ and $a_j$ into a single learned direction, saving budget in terms of dictionary size (and also description length). The direction can be visualized as follows:
\begin{equation}
    \twoatommol{i}{j}
\end{equation}
In the empirical results that follow, we use learned directions and features in the size-131K gemini SAE as proxy atoms. We then decompose learned directions in smaller SAEs as linear combinations of proxy atomic directions, giving proxy molecules. Details are given in \Cref{sec:molecules}.

\begin{table}
\centering
\caption{A polysemantic ``transform''-related feature in the 4K gemini SAE, represented as a molecule of 131K features. It activates on mentions of ``transform'' as well as strings like ``are autobots stronger than decepticons'' (referencing the Transformers franchise). Examining activating examples in the table, the 131K features (treated as proxy atoms) seem comparatively monosemantic (e.g., representing the Transformers franchise). More examples for each atom are given in \Cref{tab:target-3290-five-atom-binned-examples}.}
\label{tab:target-3290-transform-mix-selected-with-molecule}
\includegraphics[width=\linewidth]{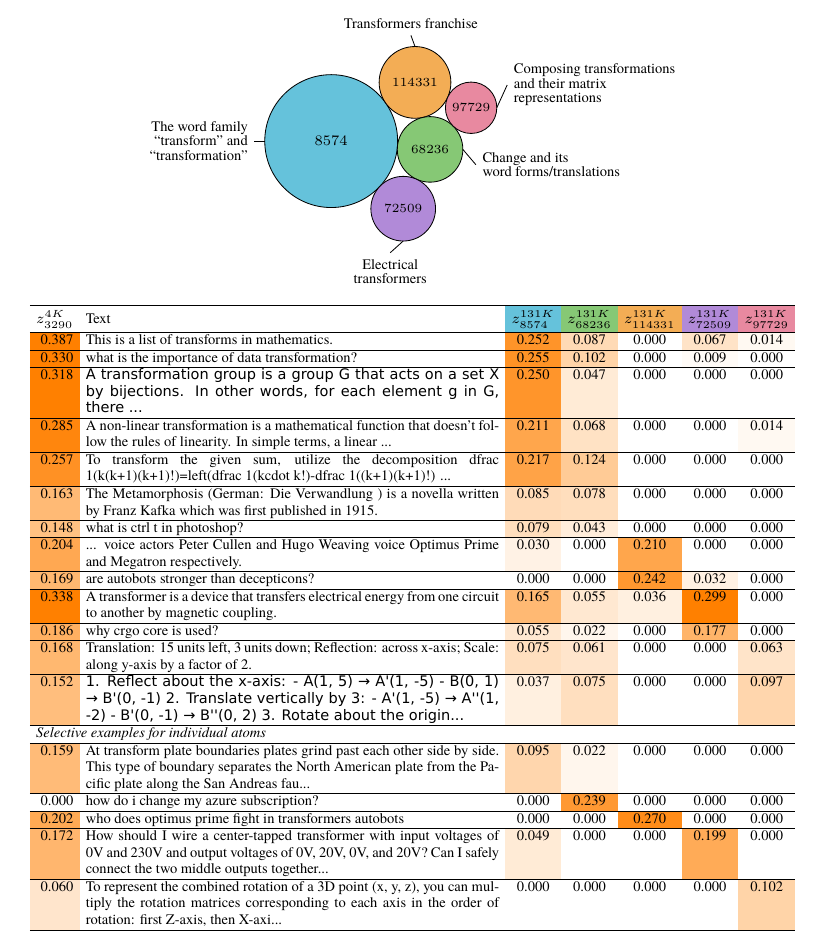}
\end{table}

\definecolor{aprefixacolor}{HTML}{65C2DB}
\definecolor{aprefixahcolor}{HTML}{B18AD8}
\definecolor{aprefixabcolor}{HTML}{86C875}
\definecolor{aprefixapcolor}{HTML}{F3AD55}
\definecolor{aprefixascolor}{HTML}{E889A0}

\newcommand{\aprefixvcenter}[1]{%
  \raisebox{-.5\height}{#1}%
}
\newcommand{\aprefixrecallcell}[2]{%
  \multicolumn{1}{>{\columncolor{orange!#1}\centering\arraybackslash}m{1.25cm}}{%
    \aprefixvcenter{$#2$}%
  }%
}

\begin{table*}[t]
\centering
\setlength{\tabcolsep}{4pt}
\setlength{\aboverulesep}{0pt}
\setlength{\belowrulesep}{0pt}
\renewcommand{\arraystretch}{1.18}
\scriptsize
\begin{tabular}{
  >{\centering\arraybackslash}m{0.9cm}
  >{\centering\arraybackslash}m{1.0cm}
  >{\centering\arraybackslash}m{5.0cm}
  >{\centering\arraybackslash}m{1.25cm}
  >{\centering\arraybackslash}m{1.25cm}
  >{\centering\arraybackslash}m{1.25cm}
  >{\centering\arraybackslash}m{1.25cm}
}
\toprule
dict. size
& $i$
& $\hat a_i$ decoder molecule
& \cellcolor{aprefixahcolor}recall on \texttt{ah}
& \cellcolor{aprefixabcolor}recall on \texttt{ab}
& \cellcolor{aprefixapcolor}recall on \texttt{ap}
& \cellcolor{aprefixascolor}recall on \texttt{as}
\\
\midrule
\aprefixvcenter{$512$}
& \aprefixvcenter{108}
& \aprefixvcenter{\begin{tabular}{@{}>{\raggedleft\arraybackslash}m{3.05cm}@{\hspace{0.30cm}}>{\centering\arraybackslash}m{1.45cm}@{}}
\aprefixvcenter{\shortstack[r]{$0.91\,a^*_{\mathrm{a}} + 0.11\,a^*_{\mathrm{as}}$\\[-1pt]$+ 0.08\,a^*_{\mathrm{ab}} + 0.03\,a^*_{\mathrm{ap}}$\\[-1pt]$+ 0.02\,a^*_{\mathrm{ah}}$}}
& \aprefixvcenter{\begin{tikzpicture}[baseline=(current bounding box.center)]
  \filldraw[draw,fill=aprefixacolor]
    (-0.1006,-0.1410) circle[radius=0.3689cm];
  \node[inner sep=0pt,font=\scriptsize,scale=1.00]
    at (-0.1006,-0.1410) {$\mathrm{a}$};
  \filldraw[draw,fill=aprefixascolor]
    (0.3229,0.2144) circle[radius=0.1840cm];
  \node[inner sep=0pt,font=\tiny,scale=0.82]
    at (0.3229,0.2144) {$\mathrm{as}$};
  \filldraw[draw,fill=aprefixabcolor]
    (0.0180,0.3750) circle[radius=0.1606cm];
  \node[inner sep=0pt,font=\tiny,scale=0.55]
    at (0.0180,0.3750) {$\mathrm{ab}$};
  \filldraw[draw,fill=aprefixapcolor]
    (0.3765,-0.0773) circle[radius=0.1125cm];
  \node[inner sep=0pt,font=\tiny,scale=0.55]
    at (0.3765,-0.0773) {$\mathrm{ap}$};
  \filldraw[draw,fill=aprefixahcolor]
    (-0.2366,0.3100) circle[radius=0.1022cm];
  \node[inner sep=0pt,font=\tiny,scale=0.55]
    at (-0.2366,0.3100) {$\mathrm{ah}$};
\end{tikzpicture}}
\end{tabular}}
& \aprefixrecallcell{100}{1.000}
& \aprefixrecallcell{96}{0.962}
& \aprefixrecallcell{88}{0.877}
& \aprefixrecallcell{99}{0.993}
\\
\midrule
\multirow{4}{*}[-4.5ex]{\aprefixvcenter{$2048$}}
& \aprefixvcenter{1457}
& \aprefixvcenter{\begin{tabular}{@{}>{\raggedleft\arraybackslash}m{3.05cm}@{\hspace{0.30cm}}>{\centering\arraybackslash}m{1.45cm}@{}}
\aprefixvcenter{\shortstack[r]{$0.92\,a^*_{\mathrm{a}} + 0.10\,a^*_{\mathrm{ah}}$}}
& \aprefixvcenter{\begin{tikzpicture}[baseline=(current bounding box.center)]
  \filldraw[draw,fill=aprefixacolor]
    (-0.1015,-0.0071) circle[radius=0.3696cm];
  \node[inner sep=0pt,font=\scriptsize,scale=1.00]
    at (-0.1015,-0.0071) {$\mathrm{a}$};
  \filldraw[draw,fill=aprefixahcolor]
    (0.4436,0.0310) circle[radius=0.1768cm];
  \node[inner sep=0pt,font=\tiny,scale=0.82]
    at (0.4436,0.0310) {$\mathrm{ah}$};
\end{tikzpicture}}
\end{tabular}}
& \aprefixrecallcell{100}{1.000}
& \aprefixrecallcell{24}{0.236}
& \aprefixrecallcell{34}{0.339}
& \aprefixrecallcell{39}{0.388}
\\
\cmidrule(lr){2-7}

& \aprefixvcenter{1794}
& \aprefixvcenter{\begin{tabular}{@{}>{\raggedleft\arraybackslash}m{3.05cm}@{\hspace{0.30cm}}>{\centering\arraybackslash}m{1.45cm}@{}}
\aprefixvcenter{\shortstack[r]{$0.85\,a^*_{\mathrm{ab}} + 0.36\,a^*_{\mathrm{a}}$}}
& \aprefixvcenter{\begin{tikzpicture}[baseline=(current bounding box.center)]
  \filldraw[draw,fill=aprefixabcolor]
    (-0.2260,0.0000) circle[radius=0.3593cm];
  \node[inner sep=0pt,font=\scriptsize,scale=1.00]
    at (-0.2260,0.0000) {$\mathrm{ab}$};
  \filldraw[draw,fill=aprefixacolor]
    (0.4025,0.0000) circle[radius=0.2693cm];
  \node[inner sep=0pt,font=\tiny,scale=0.82]
    at (0.4025,0.0000) {$\mathrm{a}$};
\end{tikzpicture}}
\end{tabular}}
& \aprefixrecallcell{0}{0.000}
& \aprefixrecallcell{87}{0.873}
& \aprefixrecallcell{0}{0.003}
& \aprefixrecallcell{0}{0.000}
\\
\cmidrule(lr){2-7}

& \aprefixvcenter{413}
& \aprefixvcenter{\begin{tabular}{@{}>{\raggedleft\arraybackslash}m{3.05cm}@{\hspace{0.30cm}}>{\centering\arraybackslash}m{1.45cm}@{}}
\aprefixvcenter{\shortstack[r]{$0.84\,a^*_{\mathrm{ap}} + 0.34\,a^*_{\mathrm{a}}$}}
& \aprefixvcenter{\begin{tikzpicture}[baseline=(current bounding box.center)]
  \filldraw[draw,fill=aprefixapcolor]
    (-0.2200,0.0000) circle[radius=0.3589cm];
  \node[inner sep=0pt,font=\scriptsize,scale=1.00]
    at (-0.2200,0.0000) {$\mathrm{ap}$};
  \filldraw[draw,fill=aprefixacolor]
    (0.4038,0.0000) circle[radius=0.2649cm];
  \node[inner sep=0pt,font=\tiny,scale=0.82]
    at (0.4038,0.0000) {$\mathrm{a}$};
\end{tikzpicture}}
\end{tabular}}
& \aprefixrecallcell{0}{0.000}
& \aprefixrecallcell{0}{0.003}
& \aprefixrecallcell{89}{0.887}
& \aprefixrecallcell{0}{0.000}
\\
\cmidrule(lr){2-7}

& \aprefixvcenter{183}
& \aprefixvcenter{\begin{tabular}{@{}>{\raggedleft\arraybackslash}m{3.05cm}@{\hspace{0.30cm}}>{\centering\arraybackslash}m{1.45cm}@{}}
\aprefixvcenter{\shortstack[r]{$0.79\,a^*_{\mathrm{as}} + 0.36\,a^*_{\mathrm{a}}$}}
& \aprefixvcenter{\begin{tikzpicture}[baseline=(current bounding box.center)]
  \filldraw[draw,fill=aprefixascolor]
    (-0.2263,-0.0399) circle[radius=0.3512cm];
  \node[inner sep=0pt,font=\scriptsize,scale=1.00]
    at (-0.2263,-0.0399) {$\mathrm{as}$};
  \filldraw[draw,fill=aprefixacolor]
    (0.3847,0.0678) circle[radius=0.2693cm];
  \node[inner sep=0pt,font=\tiny,scale=0.82]
    at (0.3847,0.0678) {$\mathrm{a}$};
\end{tikzpicture}}
\end{tabular}}
& \aprefixrecallcell{3}{0.029}
& \aprefixrecallcell{0}{0.003}
& \aprefixrecallcell{0}{0.000}
& \aprefixrecallcell{78}{0.783}
\\
\midrule
\multirow{5}{*}[-6ex]{\aprefixvcenter{$131\mathrm{K}$}}
& \aprefixvcenter{38618}
& \aprefixvcenter{\begin{tabular}{@{}>{\raggedleft\arraybackslash}m{3.05cm}@{\hspace{0.30cm}}>{\centering\arraybackslash}m{1.45cm}@{}}
\aprefixvcenter{\shortstack[r]{$1.00\,a^*_{\mathrm{a}}$}}
& \aprefixvcenter{\begin{tikzpicture}[baseline=(current bounding box.center)]
  \filldraw[draw,fill=aprefixacolor]
    (0.0000,0.0000) circle[radius=0.3800cm];
  \node[inner sep=0pt,font=\scriptsize,scale=1.00]
    at (0.0000,0.0000) {$\mathrm{a}$};
\end{tikzpicture}}
\end{tabular}}
& \aprefixrecallcell{100}{1.000}
& \aprefixrecallcell{95}{0.950}
& \aprefixrecallcell{93}{0.932}
& \aprefixrecallcell{98}{0.980}
\\
\cmidrule(lr){2-7}

& \aprefixvcenter{121657}
& \aprefixvcenter{\begin{tabular}{@{}>{\raggedleft\arraybackslash}m{3.05cm}@{\hspace{0.30cm}}>{\centering\arraybackslash}m{1.45cm}@{}}
\aprefixvcenter{\shortstack[r]{$1.00\,a^*_{\mathrm{ah}}$}}
& \aprefixvcenter{\begin{tikzpicture}[baseline=(current bounding box.center)]
  \filldraw[draw,fill=aprefixahcolor]
    (0.0000,0.0000) circle[radius=0.3800cm];
  \node[inner sep=0pt,font=\scriptsize,scale=1.00]
    at (0.0000,0.0000) {$\mathrm{ah}$};
\end{tikzpicture}}
\end{tabular}}
& \aprefixrecallcell{87}{0.871}
& \aprefixrecallcell{0}{0.000}
& \aprefixrecallcell{0}{0.000}
& \aprefixrecallcell{0}{0.000}
\\
\cmidrule(lr){2-7}

& \aprefixvcenter{100965}
& \aprefixvcenter{\begin{tabular}{@{}>{\raggedleft\arraybackslash}m{3.05cm}@{\hspace{0.30cm}}>{\centering\arraybackslash}m{1.45cm}@{}}
\aprefixvcenter{\shortstack[r]{$1.00\,a^*_{\mathrm{ab}}$}}
& \aprefixvcenter{\begin{tikzpicture}[baseline=(current bounding box.center)]
  \filldraw[draw,fill=aprefixabcolor]
    (0.0000,0.0000) circle[radius=0.3800cm];
  \node[inner sep=0pt,font=\scriptsize,scale=1.00]
    at (0.0000,0.0000) {$\mathrm{ab}$};
\end{tikzpicture}}
\end{tabular}}
& \aprefixrecallcell{0}{0.000}
& \aprefixrecallcell{75}{0.749}
& \aprefixrecallcell{0}{0.000}
& \aprefixrecallcell{0}{0.000}
\\
\cmidrule(lr){2-7}

& \aprefixvcenter{120576}
& \aprefixvcenter{\begin{tabular}{@{}>{\raggedleft\arraybackslash}m{3.05cm}@{\hspace{0.30cm}}>{\centering\arraybackslash}m{1.45cm}@{}}
\aprefixvcenter{\shortstack[r]{$1.00\,a^*_{\mathrm{ap}}$}}
& \aprefixvcenter{\begin{tikzpicture}[baseline=(current bounding box.center)]
  \filldraw[draw,fill=aprefixapcolor]
    (0.0000,0.0000) circle[radius=0.3800cm];
  \node[inner sep=0pt,font=\scriptsize,scale=1.00]
    at (0.0000,0.0000) {$\mathrm{ap}$};
\end{tikzpicture}}
\end{tabular}}
& \aprefixrecallcell{0}{0.000}
& \aprefixrecallcell{0}{0.003}
& \aprefixrecallcell{80}{0.798}
& \aprefixrecallcell{0}{0.000}
\\
\cmidrule(lr){2-7}

& \aprefixvcenter{32170}
& \aprefixvcenter{\begin{tabular}{@{}>{\raggedleft\arraybackslash}m{3.05cm}@{\hspace{0.30cm}}>{\centering\arraybackslash}m{1.45cm}@{}}
\aprefixvcenter{\shortstack[r]{$1.00\,a^*_{\mathrm{as}}$}}
& \aprefixvcenter{\begin{tikzpicture}[baseline=(current bounding box.center)]
  \filldraw[draw,fill=aprefixascolor]
    (0.0000,0.0000) circle[radius=0.3800cm];
  \node[inner sep=0pt,font=\scriptsize,scale=1.00]
    at (0.0000,0.0000) {$\mathrm{as}$};
\end{tikzpicture}}
\end{tabular}}
& \aprefixrecallcell{3}{0.029}
& \aprefixrecallcell{0}{0.000}
& \aprefixrecallcell{0}{0.000}
& \aprefixrecallcell{77}{0.765}
\\
\bottomrule
\end{tabular}
\caption{Empirical a-prefix feature splitting viewed in terms of molecules. A family molecule splits into parent-child molecules into atoms. In the middle regime, no molecule corresponds to a feature that activates on the full family (feature splitting). Such a feature exists in the small and large regimes.}
\label{tab:a-prefix-empirical-feature-splitting}
\end{table*}

\paragraph{Polysemanticity.} 

Some SAE features are polysemantic, mixing loosely related concepts. In the molecular view, this is due to a direction that is a mixture of atomic directions.

\Cref{tab:target-3290-transform-mix-selected-with-molecule} gives a direction in the 4K gemini SAE written as  a molecule, with atoms corresponding to the word transform, as well as the Transformers franchise, electrical transformers, mathematical compositions of transforms, and the concept of change. Intuitively, these are mixed together in the small SAE since they are correlated.

The table shows how texts that activate the polysemantic 4K SAE feature also activate the proxy atomic features. For example, ``are autobots stronger than decepticons?'' activates the original feature, as well as the Transformers franchise feature in the 131K SAE, but not the other proxy atomic features. 

We also find that many inputs that activate the broader ``change'' feature in 131K do not activate the polysemantic 4K feature. This is an example of asymmetry between the molecular representation of learned \textit{directions} and the activation pattern of the corresponding feature. This plays a central role in the explanation for feature splitting below.

\paragraph{Feature splitting and absorption.} Consider a family with one parent atom and two child atoms, and suppose that the parent feature activation is highly correlated with the child feature activation. Concretely, we can consider a pet atom, along with a dog atom and cat atom. Then consider what happens when a different number of features are ``budgeted'' to this family.
\begin{itemize}
    \item If only one feature is budgeted, then the learned direction is a mixture of the parent direction with small contributions from each child direction (this is the first un-normalized principal component).
    \begin{equation}
        \petcatdogmol
    \end{equation}
    This direction is used in the reconstruction of every input that activates the family. Therefore, its corresponding feature has the same activation pattern as the parent atom.
    \item If two features are budgeted, then the optimal learned directions include one for each parent-child pair; each is given by a sum of the parent direction and child direction.
    \begin{equation}
        \petcatmol\quad \petdogmol
    \end{equation}
    Each molecule is used only for inputs that activate the specific child. Therefore, there is no learned feature that activates for all members of the family (giving ``feature splitting''). Notice how this solution also minimizes description length ($k$).
    \item If three features are budgeted, then the learned directions can match the actual parent and child atoms (up to the re-parameterization given in \eqref{eq:hierarchy-substitution}). 
    \begin{equation}
        \petmol \quad \catmol \quad \dogmol
    \end{equation}
    In particular, as shown in \Cref{thm:hierarchical-recovery-principle}, there is a learned feature with an activation pattern corresponding to that of the parent and each child.
\end{itemize}
This gives a theoretical explanation of mesoscale feature splitting behavior that is consistent with P3. The example illustrates a second reason for feature splitting: over-restriction of the dictionary's description length (sparsity).

\Cref{tab:a-prefix-empirical-feature-splitting} illustrates this behavior in real SAEs, where we have restricted our attention to a subset of proxy atoms that correspond to one- and two-letter ``a'' prefixes. At small dictionary size, there is a large molecule that is active for ah, ab, ap, and as. At intermediate size, there are instead molecules that mix the broad parent atom with the individual ah, ab, ap, and as child atoms. None are active for the whole family, and so the feature ``a'' appears to disappear (though its direction is still included in the molecules). Finally, at the largest size, individual atoms are recovered, including a feature that activates on the whole family.

\paragraph{Correlated directions forming semantic clusters.} Various papers report seeing clusters of learned SAE directions that are correlated. As a concrete example, in Anthropic's SAEs trained on Claude 3 Sonnet, \citet{templeton2024scaling} find that ``around the Golden Gate Bridge feature 34M/31164353 ... there
are features corresponding to particular locations in San Francisco such as Alcatraz and the Presidio'' \citep{templeton2024scaling}. 
Our model suggests that these directions are molecules rather than atoms: the directions are correlated since they share a common atom (corresponding to the shared geographical region). Indeed, \citet{templeton2024scaling} note a broader correlated San Francisco direction in a smaller SAE, consistent with the feature splitting view presented above.

In underbudgeted regimes (in either dictionary size or description length) learning these molecules can be efficient. Our results suggest, however, SAEs with enough budget should learn individual atoms, rather than molecules. Then, at large budgets, we should see the observed correlation in directions decrease substantially.

We find evidence consistent with this prediction. \Cref{fig:a-prefix-probe-decoder-cosines} compares the decoder directions selected as a-prefix probes in our 32K and 131K gemini SAEs. The median cosine over the 231 unique off-diagonal pairs is 0.144 at width 32K and 0.008 at width 131K. Direction correlation between extremely related features (e.g., the `ap'-prefix and the `ad'-prefix) can be very small or even negative.

Given the large size of the SAEs trained by Anthropic, our model predicts that these SAEs are underbudgeted in description length ($k<K$).\footnote{In the case of Anthropic's SAEs, which do not use a hard sparsity constraint, the sparsity penalty is too high.} The theory and observations here are consistent with work by \citet{chanin2026absorption}, who show that mis-specified $k$ leads to correlation in learned feature directions, and suggest using orthogonality in learned features as a way to tune sparsity.

\paragraph{} A broad implication is that many undesirable aspects of learned SAEs are explained by molecules. So problems can be solved by scaling SAEs, with appropriate sparsity \citep{chanin2026absorption}, so that molecules split into their constituent atoms.

\section{Discussion and Further Results}\label{sec:discussion}
We developed an atomic theory of language model representations, deriving three testable predictions from the LRH and validating these predictions on SAEs trained on two large embedding models. (A natural next step is to test these predictions on other representations---e.g., residual streams.) Here, we discuss further applications of the atomic theory and connections to related work, alongside additional experiments. 

\paragraph{The Recovery Principle.} The recovery principle suggests the promise of scaling standard SAEs. In particular, large SAEs should be able to simultaneously recover stable coarse and granular features, reducing the need for incorporating hierarchical structure directly (cf., \cite{muchane2025incorporating,bussmann2025learning,costa2026flat,zaigrajew2025interpreting,cao2026tree}). Two aspects of the recovery principle are worth further investigation. First, \Cref{thm:recovery-principle} gives fairly weak guarantees in terms of the number of recovered atoms (e.g., a vanishing fraction of SAE size $m$). In \Cref{sec:recovery-simulations}, we find that SAEs trained on simulated data can recover a much larger proportion of atomic features. Second, \Cref{thm:recovery-principle} requires exact sparsity $K$. While our results suggest that correct sparsity specification is important, our simulations again suggest some robustness. Practically, it is important to know when sparsity is being chosen correctly \citep{chanin2025sparse}.

\paragraph{Multi-dimensional, non-linear, and manifold hypotheses.} Some work has noted that representations can exhibit higher-dimensional non-linear structure \citep{engels2025not,bhalla2026sparse,fel2026structuring,fel2025into}. In \Cref{sec:manifold}, we suggest that these observations are consistent with the atomic theory. We show that correlations in feature activations alone (absent any correlation in their associated directions) are enough to produce manifold structure. \Cref{fig:manifold} shows that small sets of features recover manifolds on dates and hex colors. (As a ``bonus,'' \Cref{tab:gemini-color-examples} shows that these color features generalize across hex codes, text, images.)

\paragraph{A rigidity principle, platonic representations, and human interpretability.} A growing literature documents increasing similarity in representations learned between different AI models \citep{huh2024platonic,mahner2026characterizing,jha2026harnessing,dravid2023rosetta}; work studying ``algorithmic monoculture'' has theoretically and empirically analyzed homogeneity in model outputs and its implications \citep{kleinberg2021algorithmic,peng2024monoculture,wu2025generative, padmakumar2024does, raghavan2024competition, ugander2024art, doshi2024generative, goelgreat2025, kimcorrelated2025, toups2023ecosystem,jo2026subjectivity}.

A ``rigidity principle'' illustrated by \Cref{thm:2k,thm:recovery-principle,thm:hierarchical-recovery-principle} suggests an explanation: that sparse systems of representation tend to involve the same concepts (recall, again, the vocabulary construction game). If AI models all use sparse representations, then these should be the same. Formally, we might hypothesize that different representations factor as $f=Az$ and $f=A'z$ for different choices of $A\in\RR^{d\times M}$ and $A'\in \RR^{d'\times M}$, where the atomic feature representation $z:X\rightarrow \RR^M$ is shared. In the language of the platonic representation hypothesis \citep{huh2024platonic}, one could hypothesize a set of \textit{platonic atoms}. Then, as in the recovery principle, larger and larger models could be learning a growing subset of these platonic atoms, explaining increasing representational similarity. (Empirically, homogeneity seems to be increasing with model size beyond what can be explained by increasing accuracy \citep{kimcorrelated2025,goelgreat2025}.) 

In \Cref{sec:platonic}, we identify matches between features in SAEs trained on gemini and nemotron. We find that size-131K gemini and nemotron SAEs have 11K matching features at Pearson similarity threshold $0.7$ and 27K matching features at threshold $0.5$. Example shared features (shown in \Cref{tab:platonic-pearson-semantic-combined-2,tab:platonic-pearson-semantic-combined-1,tab:platonic-pearson-diverse-2,tab:platonic-pearson-diverse-3,tab:platonic-pearson-diverse-4,tab:platonic-pearson-diverse-5,tab:platonic-pearson-diverse-6,tab:platonic-pearson-diverse-7,tab:platonic-pearson-diverse-8,tab:platonic-pearson-diverse-9,tab:platonic-pearson-diverse-10,tab:platonic-pearson-diverse-1}) include ``mechanical bearings,'' ``comments in source code,'' ``trade tariffs and protectionism,'' ``trigonometric half-angle identities,'' and ``astragalus and milkvetch species.''

The fact that SAE features trained on AI models are \textit{human}-interpretable \citep{cunningham2023sparse,templeton2024scaling,movva2025sparse,peng2025use} is quite amazing, and not obvious a priori: we could have imagined AI models learning different, inscrutable representations. One hypothesis is that since both humans and AI rely on sparse representations, there is pressure towards shared concepts.

\subsection*{Code}
Code to reproduce data preparation, SAE training, and experiments is available in this repository: \url{https://github.com/kennylpeng/atomic_theory}.

\subsection*{AI Use}
The conceptual ideas in the paper are those of the authors. We used AI to prove the ultimate theorem statements in the paper (after proving restricted versions ourselves) and help draft appendix sections. We reviewed and edited proofs (and rewrote proofs of central claims for clarity). Lean verifications for theoretical results are given in this repository: \url{https://github.com/nikhgarg/atomic-features-formalization/}, with autoformalization and verification using the AppliedModelingLib project \cite{garg2026econcslib}.

\bibliography{bib}

\begin{thebibliography}{77}
\providecommand{\natexlab}[1]{#1}
\providecommand{\url}[1]{\texttt{#1}}
\expandafter\ifx\csname urlstyle\endcsname\relax
  \providecommand{\doi}[1]{doi: #1}\else
  \providecommand{\doi}{doi: \begingroup \urlstyle{rm}\Url}\fi

\bibitem[Mikolov et~al.(2013{\natexlab{a}})Mikolov, Sutskever, Chen, Corrado, and Dean]{mikolov2013distributed}
Tomas Mikolov, Ilya Sutskever, Kai Chen, Greg~S Corrado, and Jeff Dean.
\newblock Distributed representations of words and phrases and their compositionality.
\newblock \emph{Advances in neural information processing systems}, 26, 2013{\natexlab{a}}.

\bibitem[Mikolov et~al.(2013{\natexlab{b}})Mikolov, Yih, and Zweig]{mikolov2013linguistic}
Tom{\'a}{\v{s}} Mikolov, Wen-tau Yih, and Geoffrey Zweig.
\newblock Linguistic regularities in continuous space word representations.
\newblock In \emph{Proceedings of the 2013 conference of the north american chapter of the association for computational linguistics: Human language technologies}, pages 746--751, 2013{\natexlab{b}}.

\bibitem[Arora et~al.(2018)Arora, Li, Liang, Ma, and Risteski]{arora2018linear}
Sanjeev Arora, Yuanzhi Li, Yingyu Liang, Tengyu Ma, and Andrej Risteski.
\newblock Linear algebraic structure of word senses, with applications to polysemy.
\newblock \emph{Transactions of the Association for Computational Linguistics}, 6:\penalty0 483--495, 2018.

\bibitem[Nanda et~al.(2023)Nanda, Lee, and Wattenberg]{nanda2023emergent}
Neel Nanda, Andrew Lee, and Martin Wattenberg.
\newblock Emergent linear representations in world models of self-supervised sequence models.
\newblock \emph{arXiv preprint arXiv:2309.00941}, 2023.

\bibitem[Li et~al.(2025)Li, Guo, and Andreas]{li2025language}
Belinda~Z Li, Zifan~Carl Guo, and Jacob Andreas.
\newblock (how) do language models track state?
\newblock \emph{arXiv preprint arXiv:2503.02854}, 2025.

\bibitem[Panickssery et~al.(2023)Panickssery, Gabrieli, Schulz, Tong, Hubinger, and Turner]{panickssery2023steering}
Nina Panickssery, Nick Gabrieli, Julian Schulz, Meg Tong, Evan Hubinger, and Alexander~Matt Turner.
\newblock Steering llama 2 via contrastive activation addition.
\newblock \emph{arXiv preprint arXiv:2312.06681}, 2023.

\bibitem[Zou et~al.(2023)Zou, Phan, Chen, Campbell, Guo, Ren, Pan, Yin, Mazeika, Dombrowski, et~al.]{zou2023representation}
Andy Zou, Long Phan, Sarah Chen, James Campbell, Phillip Guo, Richard Ren, Alexander Pan, Xuwang Yin, Mantas Mazeika, Ann-Kathrin Dombrowski, et~al.
\newblock Representation engineering: A top-down approach to ai transparency.
\newblock \emph{arXiv preprint arXiv:2310.01405}, 2023.

\bibitem[Turner et~al.(2023)Turner, Thiergart, Leech, Udell, Vazquez, Mini, and MacDiarmid]{turner2023steering}
Alexander~Matt Turner, Lisa Thiergart, Gavin Leech, David Udell, Juan~J Vazquez, Ulisse Mini, and Monte MacDiarmid.
\newblock Steering language models with activation engineering.
\newblock \emph{arXiv preprint arXiv:2308.10248}, 2023.

\bibitem[Wang et~al.(2025)Wang, la~Tour, Watkins, Makelov, Chi, Miserendino, Heidecke, Patwardhan, and Mossing]{wang2025persona}
Miles Wang, Tom~Dupr{\'e} la~Tour, Olivia Watkins, Alex Makelov, Ryan~A Chi, Samuel Miserendino, Johannes Heidecke, Tejal Patwardhan, and Dan Mossing.
\newblock Persona features control emergent misalignment.
\newblock \emph{arXiv preprint arXiv:2506.19823}, 2025.

\bibitem[Chen et~al.(2025)Chen, Arditi, Sleight, Evans, and Lindsey]{chen2025persona}
Runjin Chen, Andy Arditi, Henry Sleight, Owain Evans, and Jack Lindsey.
\newblock Persona vectors: Monitoring and controlling character traits in language models.
\newblock \emph{arXiv preprint arXiv:2507.21509}, 2025.

\bibitem[Elhage et~al.(2022)Elhage, Hume, Olsson, Schiefer, Henighan, Kravec, Hatfield-Dodds, Lasenby, Drain, Chen, et~al.]{elhage2022toy}
Nelson Elhage, Tristan Hume, Catherine Olsson, Nicholas Schiefer, Tom Henighan, Shauna Kravec, Zac Hatfield-Dodds, Robert Lasenby, Dawn Drain, Carol Chen, et~al.
\newblock Toy models of superposition.
\newblock \emph{arXiv preprint arXiv:2209.10652}, 2022.

\bibitem[Park et~al.(2023)Park, Choe, and Veitch]{park2023linear}
Kiho Park, Yo~Joong Choe, and Victor Veitch.
\newblock The linear representation hypothesis and the geometry of large language models.
\newblock \emph{arXiv preprint arXiv:2311.03658}, 2023.

\bibitem[Costa et~al.(2025)Costa, Fel, Lubana, Tolooshams, and Ba]{costa2025flat}
Val{\'e}rie Costa, Thomas Fel, Ekdeep~Singh Lubana, Bahareh Tolooshams, and Demba Ba.
\newblock From flat to hierarchical: Extracting sparse representations with matching pursuit.
\newblock \emph{arXiv preprint arXiv:2506.03093}, 2025.

\bibitem[Garg et~al.(2026)Garg, Kleinberg, and Peng]{garg2026many}
Nikhil Garg, Jon Kleinberg, and Kenny Peng.
\newblock How many features can a language model store under the linear representation hypothesis?
\newblock In Steve Hanneke and Tor Lattimore, editors, \emph{Proceedings of Thirty Ninth Conference on Learning Theory}, volume 336 of \emph{Proceedings of Machine Learning Research}, pages 5358--5376. PMLR, 29 Jun--03 Jul 2026.
\newblock URL \url{https://proceedings.mlr.press/v336/garg26a.html}.

\bibitem[Cunningham et~al.(2023)Cunningham, Ewart, Riggs, Huben, and Sharkey]{cunningham2023sparse}
Hoagy Cunningham, Aidan Ewart, Logan Riggs, Robert Huben, and Lee Sharkey.
\newblock Sparse autoencoders find highly interpretable features in language models.
\newblock \emph{arXiv preprint arXiv:2309.08600}, 2023.

\bibitem[Bricken et~al.(2023)Bricken, Templeton, Batson, Chen, Jermyn, Conerly, Turner, Anil, Denison, Askell, Lasenby, Wu, Kravec, Schiefer, Maxwell, Joseph, Hatfield-Dodds, Tamkin, Nguyen, McLean, Burke, Hume, Carter, Henighan, and Olah]{bricken2023monosemanticity}
Trenton Bricken, Adly Templeton, Joshua Batson, Brian Chen, Adam Jermyn, Tom Conerly, Nick Turner, Cem Anil, Carson Denison, Amanda Askell, Robert Lasenby, Yifan Wu, Shauna Kravec, Nicholas Schiefer, Tim Maxwell, Nicholas Joseph, Zac Hatfield-Dodds, Alex Tamkin, Karina Nguyen, Brayden McLean, Josiah~E Burke, Tristan Hume, Shan Carter, Tom Henighan, and Christopher Olah.
\newblock Towards monosemanticity: Decomposing language models with dictionary learning.
\newblock \emph{Transformer Circuits Thread}, 2023.
\newblock https://transformer-circuits.pub/2023/monosemantic-features/index.html.

\bibitem[Gao et~al.(2024)Gao, la~Tour, Tillman, Goh, Troll, Radford, Sutskever, Leike, and Wu]{gao2024scaling}
Leo Gao, Tom~Dupr{\'e} la~Tour, Henk Tillman, Gabriel Goh, Rajan Troll, Alec Radford, Ilya Sutskever, Jan Leike, and Jeffrey Wu.
\newblock Scaling and evaluating sparse autoencoders.
\newblock \emph{arXiv preprint arXiv:2406.04093}, 2024.

\bibitem[Movva et~al.(2025)Movva, Peng, Garg, Kleinberg, and Pierson]{movva2025sparse}
Rajiv Movva, Kenny Peng, Nikhil Garg, Jon Kleinberg, and Emma Pierson.
\newblock Sparse autoencoders for hypothesis generation.
\newblock \emph{arXiv preprint arXiv:2502.04382}, 2025.

\bibitem[Peng et~al.(2025)Peng, Movva, Kleinberg, Pierson, and Garg]{peng2025use}
Kenny Peng, Rajiv Movva, Jon Kleinberg, Emma Pierson, and Nikhil Garg.
\newblock Use sparse autoencoders to discover unknown concepts, not to act on known concepts.
\newblock \emph{arXiv preprint arXiv:2506.23845}, 2025.

\bibitem[Paulo and Belrose(2025)]{paulo2025sparse}
Gon{\c{c}}alo Paulo and Nora Belrose.
\newblock Sparse autoencoders trained on the same data learn different features.
\newblock \emph{arXiv preprint arXiv:2501.16615}, 2025.

\bibitem[Fel et~al.(2025{\natexlab{a}})Fel, Lubana, Prince, Kowal, Boutin, Papadimitriou, Wang, Wattenberg, Ba, and Konkle]{fel2025archetypal}
Thomas Fel, Ekdeep~Singh Lubana, Jacob~S Prince, Matthew Kowal, Victor Boutin, Isabel Papadimitriou, Binxu Wang, Martin Wattenberg, Demba Ba, and Talia Konkle.
\newblock Archetypal sae: Adaptive and stable dictionary learning for concept extraction in large vision models.
\newblock \emph{arXiv preprint arXiv:2502.12892}, 2025{\natexlab{a}}.

\bibitem[Chanin et~al.(2026)Chanin, Wilken-Smith, Dulka, Bhatnagar, Golechha, and Bloom]{chanin2026absorption}
David Chanin, James Wilken-Smith, Tom{\'a}{\v{s}} Dulka, Hardik Bhatnagar, Satvik Golechha, and Joseph Bloom.
\newblock A is for absorption: Studying feature splitting and absorption in sparse autoencoders.
\newblock \emph{Advances in Neural Information Processing Systems}, 38:\penalty0 82318--82355, 2026.

\bibitem[Bussmann et~al.(2025)Bussmann, Nabeshima, Karvonen, and Nanda]{bussmann2025learning}
Bart Bussmann, Noa Nabeshima, Adam Karvonen, and Neel Nanda.
\newblock Learning multi-level features with matryoshka sparse autoencoders.
\newblock \emph{arXiv preprint arXiv:2503.17547}, 2025.

\bibitem[Templeton et~al.(2024)Templeton, Conerly, Marcus, Lindsey, Bricken, Chen, Pearce, Citro, Ameisen, Jones, et~al.]{templeton2024scaling}
Adly Templeton, Tom Conerly, Jonathan Marcus, Jack Lindsey, Trenton Bricken, Brian Chen, Adam Pearce, Craig Citro, Emmanuel Ameisen, Andy Jones, et~al.
\newblock Scaling monosemanticity: Extracting interpretable features from claude 3 sonnet. transformer circuits thread, 2024.

\bibitem[Chanin and Garriga-Alonso(2025)]{chanin2025sparse}
David Chanin and Adri{\`a} Garriga-Alonso.
\newblock Sparse but wrong: Incorrect l0 leads to incorrect features in sparse autoencoders.
\newblock \emph{arXiv preprint arXiv:2508.16560}, 2025.

\bibitem[Cand{\`e}s et~al.(2006)Cand{\`e}s, Romberg, and Tao]{candes2006robust}
Emmanuel~J Cand{\`e}s, Justin Romberg, and Terence Tao.
\newblock Robust uncertainty principles: Exact signal reconstruction from highly incomplete frequency information.
\newblock \emph{IEEE Transactions on information theory}, 52\penalty0 (2):\penalty0 489--509, 2006.

\bibitem[Candes and Tao(2006)]{candes2006near}
Emmanuel~J Candes and Terence Tao.
\newblock Near-optimal signal recovery from random projections: Universal encoding strategies?
\newblock \emph{IEEE transactions on information theory}, 52\penalty0 (12):\penalty0 5406--5425, 2006.

\bibitem[Donoho(2006)]{donoho2006compressed}
David~L Donoho.
\newblock Compressed sensing.
\newblock \emph{IEEE Transactions on information theory}, 52\penalty0 (4):\penalty0 1289--1306, 2006.

\bibitem[Olshausen and Field(1996)]{olshausen1996emergence}
Bruno~A Olshausen and David~J Field.
\newblock Emergence of simple-cell receptive field properties by learning a sparse code for natural images.
\newblock \emph{Nature}, 381\penalty0 (6583):\penalty0 607--609, 1996.

\bibitem[Aharon et~al.(2006)Aharon, Elad, and Bruckstein]{aharon2006uniqueness}
Michal Aharon, Michael Elad, and Alfred~M Bruckstein.
\newblock On the uniqueness of overcomplete dictionaries, and a practical way to retrieve them.
\newblock \emph{Linear algebra and its applications}, 416\penalty0 (1):\penalty0 48--67, 2006.

\bibitem[Hillar and Sommer(2015)]{hillar2015can}
Christopher~J Hillar and Friedrich~T Sommer.
\newblock When can dictionary learning uniquely recover sparse data from subsamples?
\newblock \emph{IEEE Transactions on Information Theory}, 61\penalty0 (11):\penalty0 6290--6297, 2015.

\bibitem[Arora et~al.(2014)Arora, Ge, and Moitra]{arora2014new}
Sanjeev Arora, Rong Ge, and Ankur Moitra.
\newblock New algorithms for learning incoherent and overcomplete dictionaries.
\newblock In \emph{Conference on Learning Theory}, pages 779--806. PMLR, 2014.

\bibitem[Arora et~al.(2015)Arora, Ge, Ma, and Moitra]{arora2015simple}
Sanjeev Arora, Rong Ge, Tengyu Ma, and Ankur Moitra.
\newblock Simple, efficient, and neural algorithms for sparse coding.
\newblock In \emph{Conference on learning theory}, pages 113--149. PMLR, 2015.

\bibitem[Garfinkle and Hillar(2019)]{garfinkle2019uniqueness}
Charles~J Garfinkle and Christopher~J Hillar.
\newblock On the uniqueness and stability of dictionaries for sparse representation of noisy signals.
\newblock \emph{IEEE Transactions on Signal Processing}, 67\penalty0 (23):\penalty0 5884--5892, 2019.

\bibitem[Makhzani and Frey(2013)]{makhzani2013k}
Alireza Makhzani and Brendan Frey.
\newblock K-sparse autoencoders.
\newblock \emph{arXiv preprint arXiv:1312.5663}, 2013.

\bibitem[Landauer and Dumais(1997)]{landauer1997solution}
Thomas~K Landauer and Susan~T Dumais.
\newblock A solution to plato's problem: The latent semantic analysis theory of acquisition, induction, and representation of knowledge.
\newblock \emph{Psychological review}, 104\penalty0 (2):\penalty0 211, 1997.

\bibitem[Sitharam et~al.(2014)Sitharam, Tarifi, and Wang]{sitharam2014incidence}
Meera Sitharam, Mohamad Tarifi, and Menghan Wang.
\newblock An incidence geometry approach to dictionary learning.
\newblock \emph{arXiv preprint arXiv:1402.7344}, 2014.

\bibitem[Shanbhogue et~al.(2026)Shanbhogue, Li, Zhang, {\'A}brego, Huang, Jain, Salz, Goenka, Hegde, Ma, et~al.]{shanbhogue2026gemini}
Madhuri Shanbhogue, Zhe Li, Shanfeng Zhang, Gustavo~Hern{\'a}ndez {\'A}brego, Shih-Cheng Huang, Aashi Jain, Daniel Salz, Sonam Goenka, Chaitra Hegde, Ji~Ma, et~al.
\newblock Gemini embedding 2: A native multimodal embedding model from gemini.
\newblock \emph{arXiv preprint arXiv:2605.27295}, 2026.

\bibitem[Babakhin et~al.(2025)Babakhin, Osmulski, Ak, Moreira, Xu, Schifferer, Liu, and Oldridge]{babakhin2025llama}
Yauhen Babakhin, Radek Osmulski, Ronay Ak, Gabriel Moreira, Mengyao Xu, Benedikt Schifferer, Bo~Liu, and Even Oldridge.
\newblock Llama-embed-nemotron-8b: A universal text embedding model for multilingual and cross-lingual tasks.
\newblock \emph{arXiv preprint arXiv:2511.07025}, 2025.

\bibitem[Muchane et~al.(2025)Muchane, Richardson, Park, and Veitch]{muchane2025incorporating}
Mark Muchane, Sean Richardson, Kiho Park, and Victor Veitch.
\newblock Incorporating hierarchical semantics in sparse autoencoder architectures.
\newblock \emph{arXiv preprint arXiv:2506.01197}, 2025.

\bibitem[Costa et~al.(2026)Costa, Fel, Lubana, Tolooshams, and Ba]{costa2026flat}
Val{\'e}rie Costa, Thomas Fel, Ekdeep~S Lubana, Bahareh Tolooshams, and Demba Ba.
\newblock From flat to hierarchical: Extracting sparse representations with matching pursuit.
\newblock \emph{Advances in Neural Information Processing Systems}, 38:\penalty0 34381--34424, 2026.

\bibitem[Zaigrajew et~al.(2025)Zaigrajew, Baniecki, and Biecek]{zaigrajew2025interpreting}
Vladimir Zaigrajew, Hubert Baniecki, and Przemyslaw Biecek.
\newblock Interpreting clip with hierarchical sparse autoencoders.
\newblock \emph{arXiv preprint arXiv:2502.20578}, 2025.

\bibitem[Cao et~al.(2026)Cao, Nhat, Alharbi, Nguyen, and Thai]{cao2026tree}
Tue~M Cao, Hoang~X Nhat, Raed Alharbi, Phi~Le Nguyen, and My~T Thai.
\newblock Tree sae: Learning hierarchical feature structures in sparse autoencoders.
\newblock \emph{arXiv preprint arXiv:2605.07922}, 2026.

\bibitem[Engels et~al.(2025)Engels, Michaud, Liao, Gurnee, and Tegmark]{engels2025not}
Josh Engels, Eric Michaud, Isaac Liao, Wes Gurnee, and Max Tegmark.
\newblock Not all language model features are one-dimensionally linear.
\newblock In \emph{International Conference on Learning Representations}, volume 2025, pages 84591--84622, 2025.

\bibitem[Bhalla et~al.(2026)Bhalla, Fel, Rager, Feucht, Haklay, Wurgaft, Boppana, Kowal, Shyam, Merullo, et~al.]{bhalla2026sparse}
Usha Bhalla, Thomas Fel, Can Rager, Sheridan Feucht, Tal Haklay, Daniel Wurgaft, Siddharth Boppana, Matthew Kowal, Vasudev Shyam, Jack Merullo, et~al.
\newblock Do sparse autoencoders capture concept manifolds?
\newblock \emph{arXiv preprint arXiv:2604.28119}, 2026.

\bibitem[Fel et~al.(2026)Fel, Kowal, Jacobs, Hazra, Bhalla, Sharkey, Bushnaq, Grant, Haklay, Icard, et~al.]{fel2026structuring}
Thomas Fel, Matthew Kowal, Mozes Jacobs, Dron Hazra, Usha Bhalla, Lee Sharkey, Lucius Bushnaq, Satchel Grant, Tal Haklay, Thomas Icard, et~al.
\newblock Structuring sparsity: Block-sparse featurizers capture visual concept manifolds.
\newblock \emph{arXiv preprint arXiv:2606.25234}, 2026.

\bibitem[Fel et~al.(2025{\natexlab{b}})Fel, Wang, Lepori, Kowal, Lee, Balestriero, Joseph, Lubana, Konkle, Ba, et~al.]{fel2025into}
Thomas Fel, Binxu Wang, Michael~A Lepori, Matthew Kowal, Andrew Lee, Randall Balestriero, Sonia Joseph, Ekdeep~S Lubana, Talia Konkle, Demba Ba, et~al.
\newblock Into the rabbit hull: From task-relevant concepts in dino to minkowski geometry.
\newblock \emph{arXiv preprint arXiv:2510.08638}, 2025{\natexlab{b}}.

\bibitem[Huh et~al.(2024)Huh, Cheung, Wang, and Isola]{huh2024platonic}
Minyoung Huh, Brian Cheung, Tongzhou Wang, and Phillip Isola.
\newblock The platonic representation hypothesis.
\newblock \emph{arXiv preprint arXiv:2405.07987}, 2024.

\bibitem[Mahner et~al.(2026)Mahner, Roth, Lam, Bonner, Pereira, and Hebart]{mahner2026characterizing}
Florian~P Mahner, Johannes Roth, Ka~Chun Lam, Michael~F Bonner, Francisco Pereira, and Martin~N Hebart.
\newblock Characterizing universal object representations across vision models.
\newblock \emph{arXiv preprint arXiv:2605.13675}, 2026.

\bibitem[Jha et~al.(2026)Jha, Zhang, Shmatikov, and Morris]{jha2026harnessing}
Rishi Jha, Collin Zhang, Vitaly Shmatikov, and John Morris.
\newblock Harnessing the universal geometry of embeddings.
\newblock \emph{Advances in Neural Information Processing Systems}, 38:\penalty0 45963--45987, 2026.

\bibitem[Dravid et~al.(2023)Dravid, Gandelsman, Efros, and Shocher]{dravid2023rosetta}
Amil Dravid, Yossi Gandelsman, Alexei~A Efros, and Assaf Shocher.
\newblock Rosetta neurons: Mining the common units in a model zoo.
\newblock In \emph{2023 IEEE/CVF International Conference on Computer Vision (ICCV)}, pages 1934--1943. IEEE, 2023.

\bibitem[Kleinberg and Raghavan(2021)]{kleinberg2021algorithmic}
Jon Kleinberg and Manish Raghavan.
\newblock Algorithmic monoculture and social welfare.
\newblock \emph{Proceedings of the National Academy of Sciences}, 118\penalty0 (22):\penalty0 e2018340118, 2021.

\bibitem[Peng and Garg(2024)]{peng2024monoculture}
Kenny Peng and Nikhil Garg.
\newblock Monoculture in matching markets.
\newblock In \emph{The Thirty-eighth Annual Conference on Neural Information Processing Systems}, 2024.
\newblock URL \url{https://openreview.net/forum?id=P5yezHuMSS}.

\bibitem[Wu et~al.(2025)Wu, Black, and Chandrasekaran]{wu2025generative}
Fan Wu, Emily Black, and Varun Chandrasekaran.
\newblock Generative monoculture in large language models.
\newblock In \emph{The Thirteenth International Conference on Learning Representations}, 2025.

\bibitem[Padmakumar and He(2024)]{padmakumar2024does}
Vishakh Padmakumar and He~He.
\newblock Does writing with language models reduce content diversity?
\newblock In \emph{The Twelfth International Conference on Learning Representations}, 2024.
\newblock URL \url{https://openreview.net/forum?id=Feiz5HtCD0}.

\bibitem[Raghavan(2024)]{raghavan2024competition}
Manish Raghavan.
\newblock Competition and diversity in generative {AI}.
\newblock \emph{arXiv preprint arXiv:2412.08610}, 2024.

\bibitem[Ugander and Epstein(2024)]{ugander2024art}
Johan Ugander and Ziv Epstein.
\newblock The art of randomness: Sampling and chance in the age of algorithmic reproduction.
\newblock \emph{Harvard Data Science Review}, 6\penalty0 (4), 2024.

\bibitem[Doshi and Hauser(2024)]{doshi2024generative}
Anil~R Doshi and Oliver~P Hauser.
\newblock Generative {AI} enhances individual creativity but reduces the collective diversity of novel content.
\newblock \emph{Science advances}, 10\penalty0 (28):\penalty0 eadn5290, 2024.

\bibitem[Goel et~al.(2025)Goel, Str{\"u}ber, Auzina, Chandra, Kumaraguru, Kiela, Prabhu, Bethge, and Geiping]{goelgreat2025}
Shashwat Goel, Joschka Str{\"u}ber, Ilze~Amanda Auzina, Karuna~K Chandra, Ponnurangam Kumaraguru, Douwe Kiela, Ameya Prabhu, Matthias Bethge, and Jonas Geiping.
\newblock Great models think alike and this undermines {AI} oversight.
\newblock In \emph{Forty-second International Conference on Machine Learning}, 2025.

\bibitem[Kim et~al.(2025)Kim, Garg, Peng, and Garg]{kimcorrelated2025}
Elliot~Myunghoon Kim, Avi Garg, Kenny Peng, and Nikhil Garg.
\newblock Correlated errors in large language models.
\newblock In \emph{Forty-second International Conference on Machine Learning}, 2025.

\bibitem[Toups et~al.(2023)Toups, Bommasani, Creel, Bana, Jurafsky, and Liang]{toups2023ecosystem}
Connor Toups, Rishi Bommasani, Kathleen Creel, Sarah Bana, Dan Jurafsky, and Percy~S Liang.
\newblock Ecosystem-level analysis of deployed machine learning reveals homogeneous outcomes.
\newblock \emph{Advances in Neural Information Processing Systems}, 36:\penalty0 51178--51201, 2023.

\bibitem[Jo et~al.(2026)Jo, Garg, and Raghavan]{jo2026subjectivity}
Nathanael Jo, Nikhil Garg, and Manish Raghavan.
\newblock The subjectivity of monoculture.
\newblock \emph{arXiv preprint arXiv:2602.24086}, 2026.

\bibitem[Garg(2026)]{garg2026econcslib}
Nikhil Garg.
\newblock Econcslib: Ai-assisted lean formalization for economics \& computation research.
\newblock \emph{arXiv preprint arXiv:2606.13306}, 2026.
\newblock URL \url{https://gargnikhil.com/AppliedModelingLib/}.

\bibitem[Saravia et~al.(2018)Saravia, Liu, Huang, Wu, and Chen]{saravia-etal-2018-carer}
Elvis Saravia, Hsien-Chi~Toby Liu, Yen-Hao Huang, Junlin Wu, and Yi-Shin Chen.
\newblock {CARER}: Contextualized affect representations for emotion recognition.
\newblock In Ellen Riloff, David Chiang, Julia Hockenmaier, and Jun{'}ichi Tsujii, editors, \emph{Proceedings of the 2018 Conference on Empirical Methods in Natural Language Processing}, pages 3687--3697, Brussels, Belgium, October-November 2018. Association for Computational Linguistics.
\newblock \doi{10.18653/v1/D18-1404}.
\newblock URL \url{https://aclanthology.org/D18-1404}.

\bibitem[Thorne et~al.(2018)Thorne, Vlachos, Christodoulopoulos, and Mittal]{Thorne18Fever}
James Thorne, Andreas Vlachos, Christos Christodoulopoulos, and Arpit Mittal.
\newblock {FEVER}: a large-scale dataset for fact extraction and {VERification}.
\newblock In \emph{NAACL-HLT}, 2018.

\bibitem[Thakur et~al.(2021)Thakur, Reimers, R{\"u}ckl{\'e}, Srivastava, and Gurevych]{thakur2021beir}
Nandan Thakur, Nils Reimers, Andreas R{\"u}ckl{\'e}, Abhishek Srivastava, and Iryna Gurevych.
\newblock {BEIR}: A heterogeneous benchmark for zero-shot evaluation of information retrieval models.
\newblock In \emph{Thirty-fifth Conference on Neural Information Processing Systems Datasets and Benchmarks Track (Round 2)}, 2021.
\newblock URL \url{https://openreview.net/forum?id=wCu6T5xFjeJ}.

\bibitem[Khashabi et~al.(2021)Khashabi, Ng, Khot, Sabharwal, Hajishirzi, and Callison-Burch]{gooaq2021}
Daniel Khashabi, Amos Ng, Tushar Khot, Ashish Sabharwal, Hannaneh Hajishirzi, and Chris Callison-Burch.
\newblock Gooaq: Open question answering with diverse answer types.
\newblock \emph{arXiv preprint}, 2021.

\bibitem[Yang et~al.(2018)Yang, Qi, Zhang, Bengio, Cohen, Salakhutdinov, and Manning]{yang-etal-2018-hotpotqa}
Zhilin Yang, Peng Qi, Saizheng Zhang, Yoshua Bengio, William Cohen, Ruslan Salakhutdinov, and Christopher~D. Manning.
\newblock {H}otpot{QA}: A dataset for diverse, explainable multi-hop question answering.
\newblock In Ellen Riloff, David Chiang, Julia Hockenmaier, and Jun{'}ichi Tsujii, editors, \emph{Proceedings of the 2018 Conference on Empirical Methods in Natural Language Processing}, pages 2369--2380, Brussels, Belgium, October-November 2018. Association for Computational Linguistics.
\newblock \doi{10.18653/v1/D18-1259}.
\newblock URL \url{https://aclanthology.org/D18-1259}.

\bibitem[Bajaj et~al.(2016)Bajaj, Campos, Craswell, Deng, Gao, Liu, Majumder, McNamara, Mitra, Nguyen, et~al.]{DBLP:journals/corr/NguyenRSGTMD16}
Payal Bajaj, Daniel Campos, Nick Craswell, Li~Deng, Jianfeng Gao, Xiaodong Liu, Rangan Majumder, Andrew McNamara, Bhaskar Mitra, Tri Nguyen, et~al.
\newblock {MS} {MARCO:} {A} human generated machine reading comprehension dataset.
\newblock \emph{CoRR}, abs/1611.09268, 2016.
\newblock URL \url{http://arxiv.org/abs/1611.09268}.

\bibitem[Kwiatkowski et~al.(2019)Kwiatkowski, Palomaki, Redfield, Collins, Parikh, Alberti, Epstein, Polosukhin, Devlin, Lee, Toutanova, Jones, Kelcey, Chang, Dai, Uszkoreit, Le, and Petrov]{kwiatkowski-etal-2019-natural}
Tom Kwiatkowski, Jennimaria Palomaki, Olivia Redfield, Michael Collins, Ankur Parikh, Chris Alberti, Danielle Epstein, Illia Polosukhin, Jacob Devlin, Kenton Lee, Kristina Toutanova, Llion Jones, Matthew Kelcey, Ming-Wei Chang, Andrew~M. Dai, Jakob Uszkoreit, Quoc Le, and Slav Petrov.
\newblock Natural questions: A benchmark for question answering research.
\newblock \emph{Transactions of the Association for Computational Linguistics}, 7:\penalty0 452--466, 2019.
\newblock \doi{10.1162/tacl_a_00276}.
\newblock URL \url{https://aclanthology.org/Q19-1026/}.

\bibitem[Boteva et~al.(2016)Boteva, Gholipour, Sokolov, and Riezler]{boteva2016}
Vera Boteva, Demian Gholipour, Artem Sokolov, and Stefan Riezler.
\newblock A full-text learning to rank dataset for medical information retrieval.
\newblock 2016.
\newblock URL \url{http://www.cl.uni-heidelberg.de/~riezler/publications/papers/ECIR2016.pdf}.

\bibitem[Lewis et~al.(2021)Lewis, Wu, Liu, Minervini, Küttler, Piktus, Stenetorp, and Riedel]{lewis2021paq}
Patrick Lewis, Yuxiang Wu, Linqing Liu, Pasquale Minervini, Heinrich Küttler, Aleksandra Piktus, Pontus Stenetorp, and Sebastian Riedel.
\newblock Paq: 65 million probably-asked questions and what you can do with them.
\newblock 2021.

\bibitem[Wadden et~al.(2020)Wadden, Lin, Lo, Wang, van Zuylen, Cohan, and Hajishirzi]{wadden-etal-2020-fact}
David Wadden, Shanchuan Lin, Kyle Lo, Lucy~Lu Wang, Madeleine van Zuylen, Arman Cohan, and Hannaneh Hajishirzi.
\newblock Fact or fiction: Verifying scientific claims.
\newblock In \emph{Proceedings of the 2020 Conference on Empirical Methods in Natural Language Processing (EMNLP)}, pages 7534--7550. Association for Computational Linguistics, 2020.
\newblock \doi{10.18653/v1/2020.emnlp-main.609}.
\newblock URL \url{https://aclanthology.org/2020.emnlp-main.609/}.

\bibitem[Rajpurkar et~al.(2016)Rajpurkar, Zhang, Lopyrev, and Liang]{rajpurkar-etal-2016-squad}
Pranav Rajpurkar, Jian Zhang, Konstantin Lopyrev, and Percy Liang.
\newblock {SQ}u{AD}: 100,000+ questions for machine comprehension of text.
\newblock In Jian Su, Kevin Duh, and Xavier Carreras, editors, \emph{Proceedings of the 2016 Conference on Empirical Methods in Natural Language Processing}, pages 2383--2392, Austin, Texas, November 2016. Association for Computational Linguistics.
\newblock \doi{10.18653/v1/D16-1264}.
\newblock URL \url{https://aclanthology.org/D16-1264}.

\bibitem[{Joshi} et~al.(2017){Joshi}, {Choi}, {Weld}, and {Zettlemoyer}]{2017arXivtriviaqa}
Mandar {Joshi}, Eunsol {Choi}, Daniel {Weld}, and Luke {Zettlemoyer}.
\newblock {triviaqa: A Large Scale Distantly Supervised Challenge Dataset for Reading Comprehension}.
\newblock \emph{arXiv e-prints}, art. arXiv:1705.03551, 2017.

\bibitem[Yue et~al.(2024)Yue, Zheng, Zhang, and Chen]{yue2024mammoth2}
Xiang Yue, Tuney Zheng, Ge~Zhang, and Wenhu Chen.
\newblock Mammoth2: Scaling instructions from the web.
\newblock \emph{Advances in Neural Information Processing Systems}, 2024.

\bibitem[Zhang et~al.(2023)Zhang, Thakur, Ogundepo, Kamalloo, Alfonso-Hermelo, Li, Liu, Rezagholizadeh, and Lin]{zhang-etal-2023-miracl}
Xinyu Zhang, Nandan Thakur, Odunayo Ogundepo, Ehsan Kamalloo, David Alfonso-Hermelo, Xiaoguang Li, Qun Liu, Mehdi Rezagholizadeh, and Jimmy Lin.
\newblock {MIRACL}: A multilingual retrieval dataset covering 18 diverse languages.
\newblock \emph{Transactions of the Association for Computational Linguistics}, 11:\penalty0 1114--1131, 2023.
\newblock URL \url{https://aclanthology.org/2023.tacl-1.63/}.

\end{thebibliography}

\clearpage

\appendix
\let\addcontentsline\AppendixAddContentsLine
\addtocontents{toc}{\protect\setcounter{tocdepth}{2}}
\renewcommand{\contentsname}{Appendix contents}
\tableofcontents
\clearpage
\section{Proof of Theorem 1 (Existence of Atomic Features)}\label{sec:proofs}
For a matrix $A\in \RR^{d\times M}$, define
\begin{equation}
    U_K(A) := \bigcup_{\substack{S\subseteq [M]\\|S|=K}} \spann(A_S),
\end{equation}
where $A_S:=\{a_i\}_{i\in S}.$

\begin{lemma}\label{lem:inclusion}
We have that
    \begin{equation}\label{eq:inclusion}
        U_K(A)\subseteq U_{K'}(B)
    \end{equation}
\end{lemma}

\begin{proof}
    Fix $S\subseteq[M]$ with $|S|=K$, and write
    $E_S:=\operatorname{span}\{e_i:i\in S\}$.
    For each $T\subseteq[M']$ with $|T|=K'$, define the linear
    subspace
    \begin{equation}
        L_T:=\{w\in E_S:Aw\in\operatorname{span}(B_T)\}.
    \end{equation}
    Since $Az(x)=Bz'(x)$ and $z'(x)$ is at most $K'$-sparse,
    \begin{equation}
        z(X)\cap E_S
        \subseteq
        \bigcup_{\substack{T\subseteq[M']\\|T|=K'}} L_T.
    \end{equation}
    By $K$-richness, these finitely many linear subspaces cannot
    all be proper subspaces of $E_S$. Thus $L_T=E_S$ for some $T$,
    which implies
    \begin{equation}
        \operatorname{span}(A_S)=A(E_S)
        \subseteq\operatorname{span}(B_T).
    \end{equation}
    Taking the union over all $S\subseteq[M]$ with $|S|=K$ gives
    $U_K(A)\subseteq U_{K'}(B)$, as desired.
\end{proof}
The inclusion \eqref{eq:inclusion} is the primary consequence of richness that we invoke.

We first prove that $K'\ge K.$ Invoking \Cref{lem:inclusion}, we have that $U_K(A)\subseteq U_{K'}(B).$ Consider $S\subseteq [M]$ with $|S|=K.$ Then $\spann(A_S)\subseteq U_K(A)\subseteq U_{K'}(B).$ Since a finite-dimensional subspace cannot be covered by finitely many proper linear subspaces, there exists some $T\subseteq [M']$ with $|T|=K'$ such that $\spann(A_S)\subseteq \spann(B_T).$ Since $|S|=K<\spark(A),$ as implied by the lower RIP, $\rank(A_S)=K$. Therefore, $K'=\rank(B_T)\ge K.$

Next we prove that if $M',K'=M,K$, then $(A,z)\equiv (B,u).$ Suppose that $M'=M$ and $K'=K$. Invoking \Cref{lem:inclusion}, we have that $U_K(A)\subseteq U_{K'}(B).$ As in the proof above, for each $S\subseteq [M]$ with $|S|=K,$ there must exist some $T\subseteq [M]$ with $|T|=K$ such that $\spann(A_S)\subseteq\spann(B_T).$ Since both are $K$-dimensional (since $\spark(A),\spark(B)>2K$), we must have $\spann(A_S)=\spann(B_T).$ Define $g:\binom{[M]}{K}\rightarrow \binom{[M]}{K}$, $S\mapsto T.$ Since $\spann(A_S)\neq \spann(A_{S'})$ when $S\neq S'$ (again from the spark condition), $g$ is a bijection.

Moreover,
\begin{equation}
    |S\cap S'| = \dim(\spann(A_S)\cap \spann(A_{S'})) = \dim(\spann(B_{g(S)})\cap \spann(B_{g(S')})) = |g(S)\cap g(S')|.
\end{equation}
Now for each $i\in [M]$, we can choose $S,S'\in \binom{[M]}{K}$ such that $S\cap S'=\{i\},$ so $g(S)\cap g(S') = \{j\}$ for some $j\in [M].$ This implies that
\begin{equation}
    \spann(a_i) = \spann(A_S) \cap \spann(A_{S'}) = \spann(B_{g(S)}) \cap \spann(B_{g(S')}) = \spann(b_j).
\end{equation}
So $b_j = a_i$ since columns are unit norm and coefficients must be non-negative. Since columns of $A$ are pairwise linearly independent, each $i$ maps to a distinct $j$ according to a permutation matrix $P$. Then for each $x\in X$, $Bz'(x)=Az(x).$ Since $\spark(B)>2K,$ there is a unique $K$-sparse choice of $z'(x)$. In particular, $z'(x) = P^Tz(x).$

Finally, we show that if $M' < M,$ then $K' \ge 2K.$

The proof rests on the following new characterization of when Hall's condition holds: if Hall's condition is satisfied for subsets of size at most $2K$ and subsets of size at most $K$ have neighborhoods of size at most $2K-1$, then Hall's condition is satisfied globally.

\begin{lemma}[Local-to-global Hall lemma]
\label{lem:local-global-hall}
Let $G=(X,Y,E)$ be a bipartite graph, and let $K\ge 1$. For $S\subseteq X$, let $N(S):=\{y\in Y: (x,y)\in E\text{ for some $x\in S$}\}$. Suppose that for every $S\subseteq X$,
\begin{enumerate}
    \item[(i)] if $|S|\le 2K$, then $|N(S)|\ge |S|$;
    \item[(ii)] if $|S|\le K$, then $|N(S)|\le 2K-1$.
\end{enumerate}
Then $G$ satisfies Hall's condition: $|N(S)|\ge |S|$ for every $S\subseteq X.$
Consequently, there exists a matching that saturates $X$, and in particular
$|Y|\ge |X|.$
\end{lemma}

\paragraph{Proof of \Cref{lem:local-global-hall}.} Assume for sake of contradiction that (i) and (ii) are satisfied, but that there exists $C\subseteq X$ such that $|N(C)| < |C|.$ Take $C$ to be inclusion minimal, in the sense that $|N(C')|\ge |C'|$ for all $C'\subsetneq C.$ Then $|N(C)| = |C|-1,$ since otherwise we can freely remove any vertex from $C$ to obtain $C'$ violating Hall's condition. By (i), we have that $|C| \ge 2K+1.$ We now show that (ii) is contradicted.

    \begin{figure}
\centering
\resizebox{0.35\linewidth}{!}{
\begin{tikzpicture}[
    >=Stealth,
    leftv/.style={circle, fill=black, inner sep=1.6pt},
    rightv/.style={circle, fill=black, inner sep=1.6pt},
    match/.style={thick},
    aux/.style={dashed, thick},
    tree/.style={->, thick}
]

\node at (0,6.7) {$C$};
\node at (5,6.7) {$N(C)$};

\node[leftv] (c0) at (0,0) {};
\draw (c0.center) circle[radius=1.5mm];
\node[left] at ([yshift=-1.5mm]c0.west) {$r$};

\foreach \i in {1,...,6} {
    \node[leftv] (c\i) at (0,\i) {};
}

\foreach \i in {1,...,6} {
    \node[rightv] (u\i) at (5,\i) {};
}

\foreach \i in {1,...,6} {
    \draw[match] (c\i) -- (u\i);
}

\draw[aux] (c0) -- (u1);
\draw[aux] (c1) -- (u2);
\draw[aux] (c0) -- (u3);
\draw[aux] (c2) -- (u4);
\draw[aux] (c4) -- (u5);
\draw[aux] (c4) -- (u6);

\draw[tree] (c0.center) to[bend left=55] (c1.west);
\draw[tree] (c1.center) to[bend left=55] (c2.west);

\draw[tree] (c0.center) to[bend left=72] (c3.west);

\draw[tree] (c2.center) to[bend left=65] (c4.west);

\draw[tree] (c4.center) to[bend left=55] (c5.west);
\draw[tree] (c4.center) to[bend left=65] (c6.west);

\end{tikzpicture}
}
\caption{A sketch of how the tree construction completes the proof. Removing one vertex $r$ from $C$, there exists a matching denoted by the solid lines from $C\setminus \{r\}$ to $N(C)$ (by the ``Hall minimality'' of $C$). We can then construct a tree on $C$ where the existence of an edge $ij$ implies that $m(j)\in F_i$ (so $F_i\cap F_j$ is non-empty). Splitting the tree into its two parts, this gives us a subset of $C$ that has a large neighborhood, providing the desired contradiction.}
\end{figure}
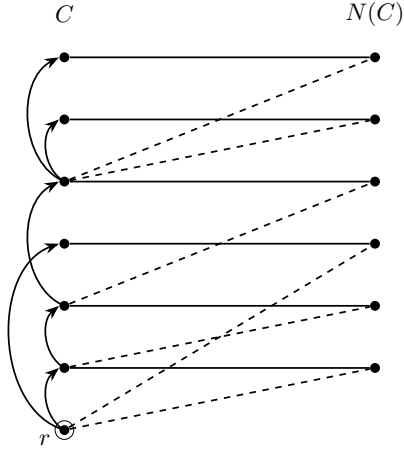

    Choose any $r\in C$. By the minimality of $C$, every strict subset of $C$ satisfies Hall's condition. Therefore, we can construct a perfect matching $m:C\setminus \{r\}\rightarrow N(C)$. 
    
    We now iteratively construct a tree $T=(C,E_T)$ alongside a bijection $g: E_T\rightarrow N(C)$. Start with the set of vertices $R=\{r\}$ and the empty set of edges $E_T=\emptyset$. At each step, choose a vertex $j\in C\setminus R$ such that $m(j)\in N(\{i\})$ for some $i\in R.$ Then add $j$ to $R$, the edge $ij$ to $E_T$, and set $g(ij):=m(j).$ Such a vertex $j$ exists as long as $R\neq C$. Indeed, suppose that there were not such a vertex. This would mean that for every $u\in N(R),$ $m^{-1}(u)\in R\setminus \{r\}$. Since $m$ is a bijection, this implies that $|N(R)| \le |R|-1.$ Therefore, if $R\neq C$, this would contradict the minimality of $C$.
    
    A key property of $T$ is that if $g(ij)=m(j)$, then $m(j)\in N(\{i\})$ and $m(j)\in N(\{j\})$.
    
    Now take any subtree of $T$ with $2K+1$ vertices (equivalently, $2K$ edges). Let the two sides of the tree (which is bipartite) have vertices $S_1$ and $S_2$. Without loss of generality, suppose $|S_1|\le K.$ Every edge $ij$ in the subtree has one endpoint in $S_1$. Therefore, for each such edge $ij$, $g(ij)$ is in $N(S_1).$ Since $g$ is a bijection, this implies that $|N(S_1)|\ge 2K$. This contradicts (ii), as desired.

\paragraph{Finishing the proof.} We now complete the proof. Invoking \Cref{lem:inclusion}, we have that $U_K(A)\subseteq U_{K'}(B).$ Therefore, for all $i\in [M]$, since $a_i\in U_K(A)\subseteq U_{K'}(B),$ $a_i$ has $K'$-sparse representation in $B$. Since $\spark(B)>2K',$ this representation is unique. Let $F_i\subseteq [M']$ denote the support of this representation.

Now consider a bipartite graph $G=([M],[M'],E)$ where $E=\{(i,j): i\in [M], j\in F_i\}.$ Then we have $N(S) = \cup_{i\in S} F_i\subseteq [M'].$ 

Our approach is to show that if $K'\le 2K-1$, then $G$ satisfies the two conditions in \Cref{lem:local-global-hall}, implying that $M'\ge M.$

To show (i) holds, observe that since $|S| \le 2K < \spark(A)$, then $\rank(A_S)=|S|$. Also,
$\spann(A_S)\subseteq \spann(B_{\cup_{i\in S} F_i}).$ Now observe that $|N(S)| = |\cup_{i\in S} F_i| \ge \rank(B_{\cup_{i\in S} F_i}) \ge \rank(A_S) = |S|$. (Notice that this property holds regardless of $K'$.)

To show (ii) holds, observe that if $|S|\le K$, $\spann(A_S)\subseteq U_{K}(A)\subseteq U_{K'}(B).$ Since a subspace can't be covered by the union of a finite number of strictly lower-dimensional subspaces, there exists $T_S\subseteq [M']$ with $|T_S|=K'$ such that $\spann(A_S)\subseteq \spann(B_{T_S}).$ For each $i\in S$, since $a_i\in \spann(A_S)$ and $F_i$ is the support of the unique $K'$-sparse representation of $a_i$ in $B$, then $F_i\subseteq T_S.$ Therefore, $|N(S)| = |\cup_{i\in S} F_i| \le |T_S| \le K' \le 2K-1,$ as desired.

Define a featurization of $f$ to be a dictionary that exactly represents $f$.

\begin{theorem}[Generic rigidity]
\label{thm:generic-rigidity}
Fix $d,M,K$ with $M>d\ge 2$. For almost every
$A=(a_1,\ldots,a_M)\in (\mathbb S^{d-1})^M$, the following holds. Suppose that $(A,z)$ is a $K$-rich $(M,K)$-featurization of $f:X\to\RR^d$. Then every $(M',K')$-featurization $(A',z')$ of $f$, for $M' > d$, satisfies
\begin{equation}\label{eq:generic-bound}
    K'\ge 1 + (d-1)\frac{M-M'}{M}.
\end{equation}
Since $d+1\ge \spark(A') > 2K'$, \eqref{eq:generic-bound} implies that $M' > \frac{M}{2}.$
\end{theorem}

\paragraph{Proof of \Cref{thm:generic-rigidity}.}

Suppose that $(A,z)$ is a $K$-rich $(M,K)$-featurization of $f: X\rightarrow \RR^d$. Then by \Cref{lem:inclusion}, we have that $a_i\in U_{K'}(B).$ We show that if
\begin{equation}
    K' < 1 + (d-1)\frac{M-M'}{M},
\end{equation}
then for almost every $A=(a_1,a_2,\cdots,a_M)\in (\mathbb{S}^{d-1})^M$, there does not exist $B\in \mathbb{R}^{d\times M'}$ such that $a_1, a_2, \cdots, a_M$ all have $K'$-sparse representations in $B$. (Notice that this argument does not require that $U_{K'}(B)$ include all of $U_K(A),$ but rather just $a_1, a_2, \cdots, a_M$, so it seems natural that the bound given here could be further improved.)

The idea is a dimension counting argument on parameter spaces. In particular, we first define a smooth manifold that parameterizes $A$. We then define a smooth manifold that parameterizes $(M',K')$-featurizations $B,z'$ that can represent $a_1, a_2, \cdots, a_M.$ Finally, we define a smooth map between the $B$-manifold and the $A$-manifold, with the property that every representable $A$-featurization must be in the image of this map. Then invoking Sard's theorem, if the dimension of the $B$-manifold is smaller than the $A$-manifold, then the image of the map has measure zero. So for a measure one subset of $A$, no $(M,K)$-featurization $A,z$ has an $(M',K')$-featurization.

Let $\mathcal{A}_M = (\mathbb{S}^{d-1})^M$. Let $\mathcal{B}_{M',K',T} = (\mathbb{S}^{d-1})^{M'}\times \prod_{i=1}^M S^{|T_i|-1}$. Let $T=(T_1,\cdots,T_M)$ where $T_i\subseteq [M']$ with $1\le |T_i|\le K'.$ Define the smooth function $\Phi:\mathcal{B}_{M',K',T}\rightarrow (\RR^d)^M$ such that
\begin{equation}
    (B,\gamma_1,\gamma_2,\cdots,\gamma_M)\mapsto \left(\sum_{j\in T_1}\gamma_{1j}b_j, \sum_{j\in T_2}\gamma_{2j}b_j, \cdots, \sum_{j\in T_M}\gamma_{Mj}b_j\right).
\end{equation}
Now define $\Theta: (\RR^d\setminus\{0\})^M \rightarrow \mathcal{A}_M: (a_1',a_2',\cdots,a_M')\mapsto (\frac{a_1'}{\|a_1'\|}, \frac{a_2'}{\|a_2'\|}, \cdots, \frac{a_M'}{\|a_M'\|}),$ which is also smooth. Let $\mathcal{B}^*_{M',K',T} = \Phi^{-1}((\RR^d\setminus\{0\})^M)$. Since $(\RR^d\setminus\{0\})^M$ is open, $\mathcal{B}^*_{M',K',T}$ is itself a smooth manifold. $\mathcal{B}^*_{M',K',T}$ has dimension at most $M'(d-1) + M(K'-1).$ $\mathcal{A}_M$ has dimension $M(d-1).$ So if
\begin{equation}\label{eq:dim-count}
    M'(d-1) + M(K'-1) < M(d-1),
\end{equation}
by Sard's theorem, $(\Theta\circ\Phi)(\mathcal{B}^*_{M',K',T})\subseteq \mathcal{A}_M$ has measure zero. Now observe that every $(a_1,a_2,\cdots,a_M)$ that has a $K'$-sparse decomposition in $(b_1,b_2,\cdots,b_{M'})$ must lie in the image of $(\Theta\circ\Phi)(\mathcal{B}^*_{M',K',T})$ for some $T$. Since there are a finite number of supports $T$, and the union of a finite number of measure zero sets has measure zero, this implies that with probability one, $(a_1,a_2,\cdots,a_M)$ does not have a $K'$-sparse decomposition in any $(b_1,b_2,\cdots,b_{M'})$. Similarly, we can repeat the argument for all $M',K'$ satisfying \eqref{eq:dim-count}.

Finally, observe that since $d+1\ge \spark(A') > 2K'$, then $K'-1< \frac{d-1}{2},$ so
\begin{align}
    \frac{d-1}{2} &> (d-1)\frac{M-M'}{M}\\
    \frac{M}{2} &> M-M'\\
    M' &> \frac{M}{2},
\end{align}
as desired.

\usetikzlibrary{arrows.meta,patterns,patterns.meta}

\begin{figure}
    
\centering
\begin{tikzpicture}[
    >=Stealth,
    every node/.style={font=\small},
    maparrow/.style={->, very thick}
]

\path[
    fill=blue!10
]
plot[smooth cycle, tension=0.9]
coordinates {
    (0.15,-0.15)
    (0.18,0.70)
    (0.55,1.32)
    (1.10,1.48)
    (1.58,1.12)
    (1.88,0.38)
    (1.82,-0.40)
    (1.45,-0.92)
    (0.88,-0.82)
    (0.40,-0.48)
};

\path[
    pattern={
        Lines[
            angle=45,
            distance=5.5pt,
            line width=0.4pt
        ]
    },
    pattern color=blue!55!black,
    draw=black,
    thick
]
plot[smooth cycle, tension=0.9]
coordinates {
    (0.15,-0.15)
    (0.18,0.70)
    (0.55,1.32)
    (1.10,1.48)
    (1.58,1.12)
    (1.88,0.38)
    (1.82,-0.40)
    (1.45,-0.92)
    (0.88,-0.82)
    (0.40,-0.48)
};

\node[font=\normalsize] at (1.00,1.88)
    {$\mathcal B^*_{M',K',T}$};

\draw[maparrow]
    (2.25,0.18)
    --
    node[midway, above=6pt, font=\normalsize]
        {$\Theta\circ\Phi$}
    (5.00,0.18);

\path[
    shade,
    top color=orange!4,
    bottom color=orange!28,
    shading angle=8,
    draw=black,
    thick
]
plot[smooth cycle, tension=0.8]
coordinates {
    (5.50,-0.35)
    (5.50,0.75)
    (5.78,1.65)
    (6.35,2.28)
    (7.05,2.20)
    (7.68,1.52)
    (8.00,0.48)
    (8.00,-0.55)
    (7.62,-1.15)
    (6.80,-1.38)
    (6.02,-1.08)
};

\draw[thick]
    (5.45,-0.05)
    arc[
        start angle=180,
        end angle=360,
        x radius=1.30,
        y radius=0.38
    ];

\draw[thick, dashed]
    (8.05,-0.05)
    arc[
        start angle=0,
        end angle=180,
        x radius=1.30,
        y radius=0.38
    ];

\path[
    fill=blue!10
]
plot[smooth cycle, tension=0.85]
coordinates {
    (6.20,0.88)
    (6.35,1.09)
    (6.72,1.30)
    (7.14,1.32)
    (7.38,1.18)
    (7.24,0.98)
    (6.88,0.82)
    (6.48,0.78)
};

\path[
    pattern=north east lines,
    pattern color=blue!55!black,
    draw=black,
    thick
]
plot[smooth cycle, tension=0.85]
coordinates {
    (6.20,0.88)
    (6.35,1.09)
    (6.72,1.30)
    (7.14,1.32)
    (7.38,1.18)
    (7.24,0.98)
    (6.88,0.82)
    (6.48,0.78)
};

\node[
    text=blue!70!black,
    anchor=west
] at (8.1,1.42)
    {$(\Theta\circ\Phi)\!\left(
    \mathcal B_{M',K',T}
    \right)$};

\draw[
    ->,
    thick,
    blue!70!black
]
    (8.15,1.34) -- (7.4,1.18);

\node[font=\normalsize] at (6.75,2.68)
    {$\mathcal A_{M}$};

\end{tikzpicture}

\caption{When $\dim(\mathcal{B}_{M',K',S}) < \dim(\mathcal{A}_{M}),$ the set of $(M,K)$-factorizations that also have an $(M',K')$-factorization is measure zero under a generic distribution over $(M,K)$-factorizations.}

\end{figure}
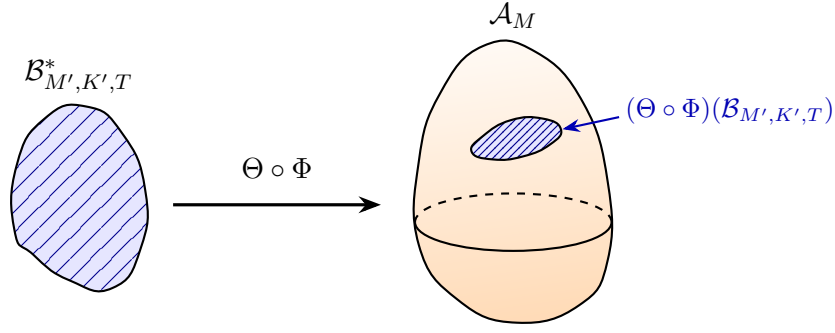

\providecommand{\RR}{\mathbb R}
\providecommand{\EE}{\mathbb E}
\providecommand{\DD}{\mathcal D}
\providecommand{\spann}{\operatorname{span}}
\providecommand{\supp}{\operatorname{supp}}
\providecommand{\dist}{\operatorname{dist}}
\providecommand{\gap}{\operatorname{gap}}

\section{Proofs of Theorems 2 and 3 (the Recovery Principle)}\label{sec:recovery-proofs}

Let us recall our assumptions. We assume throughout that $A$ and $\hat A$ have unit-norm columns and satisfy the
$(\gamma,2K)$-lower restricted isometry property. To simplify, we simply say that a matrix is $(\gamma,2K)$-stable. We also assume that $z(x)$ and $\hat z(x)$ are nonnegative, and that
$\hat z(x)$ is at most $K$-sparse. The atomic support
$S(x):=\supp(z(x))$ has size $K$ almost surely. Conditional on any
positive-probability support $S$, the coefficients $z_S(x)$ lie in
$(0,1)^K$ and have density $h_S$ satisfying $c_-\le h_S\le c_+$.
The densities may differ across supports.
All probabilities are with respect to $\DD$.
We abbreviate $\hat A^{(\DD,m)},\hat z^{(\DD,m)}$ by $\hat A,\hat z$ and write
\[
    E:=L_{\DD}(\hat A,\hat z),
    \qquad
    E_S:=\EE_{x\sim\DD}\!\left[
        \|Az(x)-\hat A\hat z(x)\|_2^2
        \,\middle|\, S(x)=S
    \right].
\]
For equal-dimensional subspaces $U,V\subseteq\RR^d$, let
$\gap(U,V):=\sup_{u\in U:\,\|u\|_2=1}\dist(u,V)$; the same value is
obtained after exchanging $U,V$. We set $\gap(\{0\},\{0\})=0$.
The first three lemmas do not require regularity, so
we can use them for both the ordinary and hierarchical settings.

The core idea is the following lemma, which shows that if a full atomic support span is accurately reconstructed by the learned dictionary, then there must be a single learned support whose span matches the atomic span (in a way that is independent of $m$ and $M$). Then, by intersecting these spans, we can show the recovery of individual atoms. A key challenge in showing the result is that the atomic span can, in theory, be ``tiled'' by many different learned spans (the number of which depends on $m$).

\begin{lemma}\label{lem:local-support}
    For every $\epsilon>0$, there exists
    $\delta=\delta(\epsilon,K,\gamma,c_-)>0$ such that any
    positive-probability support $S$ with $E_S\le\delta$ has a
    learned support $T\subseteq[m]$ of size $K$ satisfying
    $\gap(\spann(A_S),\spann(\hat A_T))\le\epsilon$.
\end{lemma}

\begin{proof}
    Fix $\epsilon>0$. Set $r_0:=\gamma/4$ and
    $r_{j+1}:=\gamma r_j/(16K)$ for $j=0,\ldots,K-1$.
    Write $R:=r_K$ and choose
    $\alpha:=R\min\{\epsilon,1/8\}$.
    Note that these quantities are independent of $m,M$.

    Define
    $U_K(\hat A):=\bigcup_{T\subseteq[m]:\,|T|\le K}
    \spann(\hat A_T)$ and
    $g(z):=\dist(A_Sz,U_K(\hat A))$.
    The conditional density lower bound gives
    $E_S\ge c_-\int_{(0,1)^K}g(z)^2\,dz$.
    Moreover, since distance to a fixed set is $1$-Lipschitz and
    $\|A_S\|_{\mathrm{op}}\le\sqrt K$, we have
    $|g(z)-g(z')|\le\sqrt K\|z-z'\|_2$.

    Suppose $g(z_0)>\alpha$ at some
    $z_0\in[1/4,3/4]^K$. Then $g(z)>\alpha/2$ throughout the
    coefficient ball of radius $\alpha/(2\sqrt K)$ centered at
    $z_0$. This ball lies inside $(0,1)^K$ since unit norms and stability
    imply $\gamma\le1$, and our choice of $\alpha$ makes its radius
    smaller than $1/4$. Writing $\omega_K$ for the volume of the unit
    ball in $\RR^K$, choose
    \[
        \delta:=
        \frac{c_-\omega_K}{2^{K+3}K^{K/2}}\,
        \alpha^{K+2}.
    \]
    The contribution of this ball would give
    $E_S\ge c_-(\alpha^2/4)\omega_K
    (\alpha/(2\sqrt K))^K=2\delta$, a contradiction.
    Thus $g(z)\le\alpha$ throughout $[1/4,3/4]^K$.

    Write $V:=\spann(A_S)$ and
    $y_0:=A_S(1/2,\ldots,1/2)$.
    Since $\sigma_{\min}(A_S)\ge\gamma$, the image
    $A_S[1/4,3/4]^K$ contains the ball
    $\mathcal B_0:=B_V(y_0,r_0)$.
    Indeed, any displacement in $V$ of norm at most $\gamma/4$
    equals $A_Sw$ for some $\|w\|_2\le1/4$.
    Here and below, $B_V$ denotes a closed ball taken within $V$.

    For each $y\in\mathcal B_0$, choose a best signed $K$-sparse
    coefficient vector $v(y)$. Such a vector exists because
    $U_K(\hat A)$ is a finite union of subspaces.
    The preceding bound gives $\|y-\hat A v(y)\|_2\le\alpha$.
    Since $v(y)-v(y')$ is at most $2K$-sparse, stability and the
    triangle inequality imply that for all $y,y'\in\mathcal B_0,$
    \begin{equation}\label{eq:nearby-codes}
        \|v(y)-v(y')\|_2
        \le
        \frac{\|\hat A(v(y)-v(y'))\|_2}{\gamma}
        \le
        \frac{\|y-y'\|_2+2\alpha}{\gamma}
    \end{equation}
    We now construct nested balls
    $\mathcal B_j:=B_V(y_j,r_j)$ and sets $T_j\subseteq[m]$
    of size $j$ such that
    $T_j\subseteq\supp(v(y))$ throughout $\mathcal B_j$.
    Initially, take $T_0=\varnothing$.

    Suppose $\mathcal B_j,T_j$ have been constructed, with $j<K$.
    Since $\dim V=K>j$, choose a unit vector in $V$ orthogonal to
    $\spann(\hat A_{T_j})$. At least one of the two points obtained
    by displacing $y_j$ by $\pm r_j/2$ in this direction has
    distance at least $r_j/2$ from $\spann(\hat A_{T_j})$.
    Call this point $y_{j+1}$. Then, since errors are at most $\alpha$ and since columns are unit-norm,
    \begin{equation}
        \frac{r_j}{2}
        \le
        \dist\bigl(y_{j+1},\spann(\hat A_{T_j})\bigr)
        \le
        \alpha+\sum_{\ell\notin T_j}|v_\ell(y_{j+1})|.
    \end{equation}
    Since $\alpha\le r_j/4$ and $v(y_{j+1})$ has at most $K$
    nonzero entries, some $i\notin T_j$ satisfies
    $|v_i(y_{j+1})|\ge r_j/(4K)$.

    Set $T_{j+1}:=T_j\cup\{i\}$ and
    $\mathcal B_{j+1}:=B_V(y_{j+1},r_{j+1})$.
    This ball lies inside $\mathcal B_j$, since its center is
    displaced by $r_j/2$ and $r_{j+1}\le r_j/2$.
    For every $y\in\mathcal B_{j+1}$,
    \eqref{eq:nearby-codes} and $2\alpha\le r_{j+1}$ give
    $|v_i(y)-v_i(y_{j+1})|
    \le 2r_{j+1}/\gamma=r_j/(8K)$.
    Hence $|v_i(y)|\ge r_j/(8K)>0$ throughout the smaller ball.
    The coordinates in $T_j$ remain nonzero by the induction
    hypothesis, completing the construction.

    After $K$ steps, $T:=T_K$ has size $K$, so sparsity forces
    $\supp(v(y))=T$ throughout $\mathcal B_K$.
    Consequently,
    $\dist(y,\spann(\hat A_T))\le\alpha$
    for every $y\in B_V(y_K,R)$.

    Finally, let $P_T$ be the orthogonal projection onto
    $\spann(\hat A_T)$. For any unit vector $w\in V$, both
    $y_K+Rw$ and $y_K-Rw$ belong to the final ball. Therefore,
    \[
        2R\|(I-P_T)w\|_2
        \le
        \|(I-P_T)(y_K+Rw)\|_2
        +
        \|(I-P_T)(y_K-Rw)\|_2
        \le 2\alpha.
    \]
    Noting that $I$ is the projection onto $\spann(A_S)$ and that both $A_S$ and $\hat{A}_T$ have dimension $K$, so taking the supremum
    over unit vectors $w\in V$ bounds their gap by
    $\alpha/R\le\epsilon$.
\end{proof}

The next lemma shows that approximate matching of support spans
preserves their intersections. Importantly, its bound does not depend
on the number of supports being intersected.

\begin{lemma}\label{lem:stable-intersection}
    Suppose a nonempty finite collection of pairs $S_\ell,T_\ell$,
    with $|S_\ell|=|T_\ell|=K$, satisfies
    \[
        \gap(\spann(A_{S_\ell}),\spann(\hat A_{T_\ell}))\le r
        \qquad\text{for every }\ell.
    \]
    Write $I:=\bigcap_\ell S_\ell$ and
    $J:=\bigcap_\ell T_\ell$.
    If $r(1+2K/\gamma)<1$, then
    \[
        |I|=|J|,
        \qquad
        \gap\bigl(\spann(A_I),\spann(\hat A_J)\bigr)
        \le r(1+2K/\gamma).
    \]
\end{lemma}

\begin{proof}
    Take a unit vector $u\in\spann(A_I)$. For each $\ell$, choose
    $v_\ell\in\RR^m$, supported on $T_\ell$, such that
    $\|u-\hat A v_\ell\|_2\le r$.
    Fix one pair, indexed by $1$. For every $j\in T_1\setminus J$,
    some $T_\ell$ omits $j$. Stability then gives
    $|(v_1)_j|\le\|v_1-v_\ell\|_2\le2r/\gamma$.
    Deleting these at most $K$ coordinates from $v_1$, and using
    the unit norms of the learned columns, gives
    \[
        \dist(u,\spann(\hat A_J))
        \le \|u-\hat A v_1\|_2
           +\sum_{j\in T_1\setminus J}|(v_1)_j|
        \le r+2Kr/\gamma.
    \]
    Exchanging $A,S_\ell$ with $\hat A,T_\ell$ gives the same
    bound in the reverse direction. If the two intersection spans
    had different dimensions, the larger would contain a unit
    vector orthogonal to the smaller, contradicting the bound
    being less than $1$. Their dimensions are $|I|$ and $|J|$ by
    stability, proving both claims. This argument also covers the
    case in which an intersection is empty.
\end{proof}

We next connect these geometric matches with the actual learned
activations. This step uses both nonnegativity of $\hat z$ and the
upper density bound. It does not assume that $\hat z(x)$ minimizes
the reconstruction error for each $x$.

\begin{lemma}[Activation recovery]\label{lem:activation-transfer}
    For every $\eta\in(0,1)$, there exist $r,c,\tau>0$,
    depending only on $\eta,K,\gamma,c_+$, with the following
    property.

    Let $\mathcal G$ be a collection of positive-probability
    atomic supports. Suppose each $S\in\mathcal G$ has a learned
    support $T_S$ of size $K$ satisfying
    $\gap(\spann(A_S),\spann(\hat A_{T_S}))\le r$.
    Suppose also that an atomic coordinate $i$ of prevalence
    $p_i:=\Pr(z_i>0)>0$ and a learned coordinate $j$
    satisfy $j\in T_S$ if and only if $i\in S$, for every
    $S\in\mathcal G$.

    If $E+\Pr(S(x)\notin\mathcal G)\le c p_i$, then
    \[
        \operatorname{F1}_{\DD}
        \bigl(\mathbf{1}\{z_i>0\},
              \mathbf{1}\{\hat z_j>\tau\}\bigr)
        >\eta.
    \]
\end{lemma}

\begin{proof}
    We first derive a bound for any $0\le r\le1/2$ and
    $\tau>0$ satisfying $2Kr\le\gamma\tau$.
    Write $\beta:=\Pr(S(x)\notin\mathcal G)$.

    Fix $S\in\mathcal G$, and let $P_{T_S}$ denote orthogonal
    projection onto $\spann(\hat A_{T_S})$.
    Let $H_S$ be the $K\times K$ matrix satisfying
    $\hat A_{T_S}H_S=P_{T_S}A_S$.
    It exists and is unique because $\hat A_{T_S}$ has full
    column rank. For an input with support $S$, let
    $v_S(x)\in\RR^m$ have coordinates $H_Sz_S(x)$ on $T_S$
    and zero elsewhere. These auxiliary coefficients may have
    either sign.

    Since the atomic coefficients lie in $(0,1)^K$ and the
    columns have unit norm, $\|Az(x)\|_2\le K$.
    The gap bound therefore gives
    $\|Az(x)-\hat A v_S(x)\|_2\le Kr$.
    Writing $e(x):=\|Az(x)-\hat A\hat z(x)\|_2$,
    stability and the triangle inequality imply
    \begin{equation}\label{eq:actual-code-comparison}
        \|\hat z(x)-v_S(x)\|_2
        \le
        \frac{\|\hat A(\hat z(x)-v_S(x))\|_2}{\gamma}
        \le
        \frac{e(x)+Kr}{\gamma}.
    \end{equation}
    Here stability applies because $\hat z(x)-v_S(x)$ is
    at most $2K$-sparse.

    If $j\in T_S$, the coordinate $(v_S(x))_j$ is a linear
    function of $z_S(x)$ with at least one coefficient of
    magnitude $\gamma/(2K)$ or larger. To see this, for every
    $w\in\RR^K$,
    \[
        \sqrt K\|H_Sw\|_2
        \ge \|P_{T_S}A_Sw\|_2
        \ge \sqrt{1-r^2}\|A_Sw\|_2
        \ge \frac{\gamma}{2}\|w\|_2.
    \]
    Thus $\sigma_{\min}(H_S)\ge\gamma/(2\sqrt K)$.
    Since $H_S$ is square, each row has at least this norm,
    and hence has an entry of magnitude at least
    $\gamma/(2K)$.

    Holding the other $K-1$ atomic coefficients fixed, the
    values of the corresponding coefficient for which
    $|(v_S(x))_j|\le2\tau$ form an interval of length at most
    $8K\tau/\gamma$.
    Integrating over the remaining coordinates and using
    $h_S\le c_+$ gives
    \begin{equation}\label{eq:small-coordinate}
        \Pr\bigl(
            |(v_S(x))_j|\le2\tau
            \,\big|\,S(x)=S
        \bigr)
        \le \frac{8Kc_+\tau}{\gamma}.
    \end{equation}
    No independence of the coefficients is needed.

    By Markov's inequality, outside an event of probability
    at most $\beta+4E/(\gamma^2\tau^2)$, we have
    $S(x)\in\mathcal G$ and $e(x)\le\gamma\tau/2$.
    On this complement, \eqref{eq:actual-code-comparison}
    and $Kr\le\gamma\tau/2$ give
    $\|\hat z(x)-v_S(x)\|_2\le\tau$, where $S=S(x)$.

    If $i\notin S$, then $j\notin T_S$, so
    $(v_S(x))_j=0$ and $\hat z_j(x)\le\tau$.
    Thus there is no false positive on this complement.
    If $i\in S$ but $\hat z_j(x)\le\tau$, nonnegativity of
    $\hat z_j(x)$ implies
    $-\tau\le(v_S(x))_j\le2\tau$.
    Summing \eqref{eq:small-coordinate} over supports in $\mathcal G$
    containing $i$, the probability of these false negatives
    is at most $(8Kc_+\tau/\gamma)p_i$.

    Let $Y:=\mathbf{1}\{z_i>0\}$ and
    $\widehat Y:=\mathbf{1}\{\hat z_j>\tau\}$.
    The total disagreement probability is therefore at most
    $\beta+4E/(\gamma^2\tau^2)+(8Kc_+\tau/\gamma)p_i$.
    Since $\Pr(Y=1)=p_i>0$, the definition of the F1
    score gives
    \begin{equation}\label{eq:activation-bound}
        \begin{aligned}
            \operatorname{F1}_{\DD}(Y,\widehat Y)
            &=
            1-
            \frac{\Pr_{\DD}(Y\ne\widehat Y)}
                 {\Pr_{\DD}(Y=1)+\Pr_{\DD}(\widehat Y=1)}
            \\
            &\ge
            1-\frac{\beta}{p_i}
             -\frac{4E}{\gamma^2\tau^2p_i}
             -\frac{8Kc_+\tau}{\gamma}.
        \end{aligned}
    \end{equation}

    Finally, choose $\tau>0$ sufficiently small that
    $8Kc_+\tau/\gamma<(1-\eta)/2$, and choose
    $0<r\le\min\{1/2,\gamma\tau/(2K)\}$.
    Choose $c>0$ sufficiently small that
    $c(1+4/(\gamma^2\tau^2))<(1-\eta)/2$.
    These choices depend only on $\eta,K,\gamma,c_+$.
    If $E+\beta\le c p_i$, then the sum of the first two
    error terms in \eqref{eq:activation-bound} is at most
    $c(1+4/(\gamma^2\tau^2))$, so the F1 score exceeds $\eta$.
\end{proof}

\paragraph{The ordinary recovery principle.}
We now impose regularity of the support law: for a fixed
$\rho\in(0,1]$ and distinct atoms $i,k$ with $p_i>0$, assume
$\Pr(k\notin S(x)\mid i\in S(x))\ge\rho$.
Let $\sigma$ order the atom prevalences nonincreasingly and let
$T(m):=\sum_{j=m+1}^M p_{\sigma(j)}$.
Write $L_{\DD}^*(m)$ for the infimum of the dictionary learning loss
over the unit-norm, $(\gamma,2K)$-stable reference class with nonnegative,
at most $K$-sparse codes. We do not need the infimum to be attained.

\begin{proof}[Proof of Theorem~2]
    Retain the $m$ most prevalent columns of $A$, together with
    their true coefficients. This gives an admissible dictionary.
    For every input, the squared norm of the omitted sum is at most
    $K\sum_{k\notin\sigma([m])}z_k(x)^2$, which is at most $K$
    times the number of omitted active atoms. Therefore,
    \begin{equation}\label{eq:ordinary-tail}
        E\le C_0L_{\DD}^*(m)\le C_0K T(m).
    \end{equation}
    We will show that an atom of prevalence $p_i>0$ is recovered
    whenever $E\le cp_i$, for a sufficiently small constant $c>0$.

    It suffices to take $0<\eta<1$. Let $r_*,c_*,\tau$ be the
    constants supplied by Lemma~\ref{lem:activation-transfer}.
    Choose $0<r\le r_*$ small enough that
    $r(1+2K/\gamma)<\min\{\gamma/(4\sqrt2),\sqrt{1-\eta^2}\}$.
    Let $\delta$ be the threshold supplied by
    Lemma~\ref{lem:local-support} with $\epsilon=r$, decreasing it
    if necessary so that
    $\delta<c_-\gamma^2(\gamma/(8K))^{K-1}/32$.
    Finally, choose $c>0$ such that $c<\rho\delta$ and
    $c(1+\delta^{-1})\le c_*$.
    These choices depend only on the parameters in the theorem.

    Call a positive-probability support $S$ good if $E_S\le\delta$,
    and write $\mathcal G$ for the good supports. Their complement
    has probability $\beta\le E/\delta$. For each good support,
    choose $T_S$ as in Lemma~\ref{lem:local-support}, so its span
    has gap at most $r$ from $\spann(A_S)$.

    Fix an atom $i$ with $p_i>0$ and $E\le cp_i$, and let
    $\mathcal G_i:=\{S\in\mathcal G:i\in S\}$.
    Since $\beta<\rho p_i\le p_i$, this collection is nonempty.
    For every $k\ne i$, regularity gives
    $\Pr(i\in S(x),k\notin S(x))\ge\rho p_i>\beta$.
    Hence some good support contains $i$ and omits $k$, implying
    $\bigcap_{S\in\mathcal G_i}S=\{i\}$.
    Lemma~\ref{lem:stable-intersection} now gives a single learned
    coordinate $j$ such that
    \begin{equation}\label{eq:ordinary-match}
        \bigcap_{S\in\mathcal G_i}T_S=\{j\},
        \qquad
        \gap\bigl(\spann(a_i),\spann(\hat a_j)\bigr)
        \le r(1+2K/\gamma).
    \end{equation}
    Thus $j\in T_S$ whenever a good support contains $i$.
    If a good support $S$ omits $i$, adding $S$ to the preceding
    intersection makes the atomic intersection empty.
    Lemma~\ref{lem:stable-intersection} makes the learned
    intersection empty as well, so $j\notin T_S$.
    We have proved $j\in T_S$ if and only if $i\in S$ on all good
    supports, not just on those containing $i$.

    The line bound in \eqref{eq:ordinary-match} gives
    $|\langle a_i,\hat a_j\rangle|>\eta$. We must also fix its
    sign. If the inner product were negative, the same bound and
    our choice of $r$ would give $\|a_i+\hat a_j\|_2<\gamma/4$.
    Choose any $S\in\mathcal G_i$ and consider coefficients with
    $z_i\in(1/2,1)$ and $z_k\in(0,\gamma/(8K))$ for
    $k\in S\setminus\{i\}$. Since $\hat z$ is nonnegative and at
    most $K$-sparse, $\hat z+z_i e_j$ is at most $2K$-sparse and
    has norm at least $z_i$. Stability and the triangle inequality
    therefore give, on this coefficient region,
    \[
        \|Az-\hat A\hat z\|_2
        \ge (\gamma-\|a_i+\hat a_j\|_2)z_i
             -\sum_{k\in S\setminus\{i\}}z_k
        \ge\gamma/4.
    \]
    The density lower bound gives this region conditional
    probability at least $(c_-/2)(\gamma/(8K))^{K-1}$.
    Thus $E_S\ge c_-\gamma^2(\gamma/(8K))^{K-1}/32$, contradicting
    goodness and our choice of $\delta$. Consequently,
    $\langle a_i,\hat a_j\rangle>\eta$, as required for signed
    direction recovery.

    For activation recovery, observe that
    $E+\beta\le E(1+\delta^{-1})\le c_*p_i$.
    Lemma~\ref{lem:activation-transfer} therefore applies to the
    same coordinate $j$ and gives an F1 score greater than $\eta$
    at threshold $\tau$. This threshold is shared by all recovered
    coordinates.
    Finally, take $C:=C_0K/c$. By \eqref{eq:ordinary-tail}, the
    condition $p_i\ge CT(m)$ implies $E\le cp_i$ and completes
    the proof.
\end{proof}

\paragraph{The hierarchical recovery principle.}
We make precise the hierarchical regularity invoked in the main text.
Let $M=3N$ and $K=2L$, with families $i,i_1,i_2$ for $i\in[N]$.
An input chooses a set $F(x)\subseteq[N]$ of exactly $L$ active
families. Each active family contributes its parent and exactly one
child, so
$S(x)=\bigcup_{i\in F(x)}\{i,i_{b_i(x)}\}$ with
$b_i(x)\in\{1,2\}$.
The law of $F(x)$ satisfies
$\Pr(k\notin F(x)\mid i\in F(x))\ge\rho$ for distinct parents
$i,k$ whenever $\Pr(i\in F(x))>0$.
In addition, for a fixed $\kappa\in(0,1/2]$, we require
\begin{equation}\label{eq:child-balance}
    \Pr\bigl(b_i(x)=b\mid F(x)=F\bigr)\ge\kappa
    \qquad
    (i\in F,\ b\in\{1,2\})
\end{equation}
for every positive-probability family set $F$.
Conditional on each resulting atomic support, all $K$ coefficients
have a joint density between $c_-$ and $c_+$ on $(0,1)^K$, as above.
We call such a distribution hierarchical regular.

The child choices need not be independent, and neither do the active
coefficients. Independent uniform child choices give the special case
$\kappa=1/2$. The child-choice bound and the full conditional density
bounds are part of this definition; they do not follow from the
child-implies-parent condition alone.

As in the main text, now let $p_i:=\Pr(i\in F(x))$ be the parent
prevalences, let $\sigma$ sort these prevalences, and let
$T(n):=\sum_{j=n+1}^N p_{\sigma(j)}$.
For a learned dictionary of width $m$, we set $n:=\lfloor m/3\rfloor$,
with $T(0):=\sum_{i=1}^N p_i=L$.

\begin{proof}[Proof of Theorem~3]
    Retain all three atoms in each of the $n=\lfloor m/3\rfloor$
    most prevalent families. If $3n<m$, fill the remaining slots
    with unused columns of $A$ and set their codes to zero.
    This gives an admissible dictionary of exactly $m$ unit-norm
    columns. Each omitted active family contributes two omitted
    atoms, so the same truncation calculation as before gives
    \begin{equation}\label{eq:hierarchical-tail}
        E\le C_0L_{\DD}^*(m)\le 2C_0K T(n).
    \end{equation}
    We show that a family of prevalence $p_i>0$ is recovered
    whenever $E\le cp_i$, for a sufficiently small constant $c>0$.

    Again it suffices to take $0<\eta<1$. Let $r_*,c_*,\tau$
    be the constants from Lemma~\ref{lem:activation-transfer}.
    Choose $0<r\le r_*$ with $r(1+2K/\gamma)<1$, and let
    $\delta$ be the threshold from
    Lemma~\ref{lem:local-support} with $\epsilon=r$.
    Choose $c>0$ so that $c<\kappa\rho\delta$ and
    $c(1+\delta^{-1})\le\kappa c_*$.
    As before, let $\mathcal G$ contain the positive-probability
    supports with $E_S\le\delta$, choose their learned supports
    $T_S$, and write $\beta\le E/\delta$ for the probability of
    the remaining supports.

    Fix a family $i$ with $p_i>0$ and $E\le cp_i$.
    Each child has prevalence at least $\kappa p_i$.
    Moreover, for a different family $k$,
    \eqref{eq:child-balance} and parent regularity give
    \[
        \Pr\bigl(i_b\in S(x),\ k\notin F(x)\bigr)
        \ge\kappa\Pr\bigl(i\in F(x),\ k\notin F(x)\bigr)
        \ge\kappa\rho p_i
        >\beta.
    \]
    Thus, even after discarding the bad supports, either child can
    occur with $i$, and each chosen child can occur without any
    specified other family. It follows that the collections
    below are nonempty and
    \begin{equation}\label{eq:hierarchical-intersections}
        \bigcap_{S\in\mathcal G:\,i\in S}S=\{i\},
        \qquad
        \bigcap_{S\in\mathcal G:\,i_b\in S}S=\{i,i_b\}
        \quad (b\in\{1,2\}).
    \end{equation}
    Indeed, the other child is excluded whenever $i_b$ occurs,
    and the displayed probability bound excludes every atom
    belonging to any other family.

    Applying Lemma~\ref{lem:stable-intersection} to these
    intersections gives learned sets
    \[
        J_i:=\bigcap_{S\in\mathcal G:\,i\in S}T_S=\{j_i\},
        \qquad
        J_{i_b}:=\bigcap_{S\in\mathcal G:\,i_b\in S}T_S
        =\{j_i,j_{i_b}\}.
    \]
    Here $J_i\subseteq J_{i_b}$ because every support containing
    $i_b$ also contains $i$.
    The good supports containing $i$ are exactly the union of
    those containing $i_1$ and those containing $i_2$. Hence
    $J_{i_1}\cap J_{i_2}=J_i$, so
    $j_i,j_{i_1},j_{i_2}$ are distinct.

    We claim that each of these learned coordinates has the same
    membership pattern as its corresponding atom across all good
    supports. For the parent, this is the same argument as in
    Theorem~2: adding a good support
    omitting $i$ to its intersection forces that support's learned
    set to omit $j_i$.
    For a child $i_b$, every good support containing $i_b$ contains
    $J_{i_b}$ on the learned side. If a good support omits the
    family, its atomic intersection with $\{i,i_b\}$ is empty,
    so its learned intersection with $J_{i_b}$ is empty as well.
    If it instead contains $i$ and the other child, that atomic
    intersection is $\{i\}$. Lemma~\ref{lem:stable-intersection}
    makes the learned intersection a singleton, which must be
    $\{j_i\}$. In particular it omits $j_{i_b}$.
    This proves the claimed membership patterns.

    Finally, $E+\beta\le E(1+\delta^{-1})\le\kappa c_*p_i$,
    and each of the three target coordinates has prevalence at
    least $\kappa p_i$. Lemma~\ref{lem:activation-transfer}
    therefore gives an F1 score greater than $\eta$ for each
    target, using the distinct learned coordinates
    $j_i,j_{i_1},j_{i_2}$. The same threshold $\tau$ works for
    all three and for all recovered families.
    Take $C:=2C_0K/c$ and invoke
    \eqref{eq:hierarchical-tail}. This proves the theorem.
\end{proof}

\clearpage

\section{Experimental Setup}\label{sec:setup}

This section describes the shared data, training setup, and matching metrics, followed by the setup for the main P1, P2, and P3 experiments in \Cref{fig:persistent-stability,fig:dataset-heatmaps,fig:hierarchy-mean-f1}, respectively. Extended experiments (including results and setup) are discussed in \Cref{sec:p1-extended,sec:p2-extended,sec:p3-extended}.

\paragraph{Code.}
Code to reproduce data preparation, SAE training, and evaluation are given in the repository at \url{https://github.com/kennylpeng/atomic_theory}.

\subsection{Data and SAE training details}\label{sec:training-details}

\paragraph{Training corpus.}
The full training mix, \full, combines 13 source datasets used to fine-tune the Nemotron embedding model. \Cref{tab:training-corpus} counts the stored \emph{embedding vectors} available to SAE training: a question and its answer, for example, are embedded separately and contribute two vectors.

\begin{table}[htbp]
    \centering
    \small
    \caption{SAE training datasets.}
    \label{tab:training-corpus}
    \begin{tabularx}{\linewidth}{@{}Xr@{}}
        \toprule
        Source collection & \# embeddings \\
        \midrule
        \href{https://huggingface.co/datasets/mteb/emotion}{Emotion} \citep{saravia-etal-2018-carer} & 19,930 \\
        \href{https://huggingface.co/datasets/BeIR/fever}{FEVER} \citep{Thorne18Fever,thakur2021beir} & 11,079,420 \\
        \href{https://huggingface.co/datasets/sentence-transformers/gooaq}{GooAQ} \citep{gooaq2021} & 6,024,992 \\
        \href{https://huggingface.co/datasets/mteb/hotpotqa}{HotpotQA} \citep{yang-etal-2018-hotpotqa} & 10,564,510 \\
        \href{https://huggingface.co/datasets/mteb/msmarco}{MS MARCO} \citep{DBLP:journals/corr/NguyenRSGTMD16} & 9,351,785 \\
        \href{https://huggingface.co/datasets/sentence-transformers/natural-questions}{Natural Questions} \citep{kwiatkowski-etal-2019-natural} & 200,462 \\
        \href{https://huggingface.co/datasets/mteb/nfcorpus}{NFCorpus} \citep{boteva2016} & 10,503 \\
        \href{https://huggingface.co/datasets/sentence-transformers/paq}{PAQ}, sampled 10M pairs \citep{lewis2021paq} & 20,000,000 \\
        \href{https://huggingface.co/datasets/mteb/scifact}{SciFact} \citep{wadden-etal-2020-fact} & 11,475 \\
        \href{https://huggingface.co/datasets/sentence-transformers/squad}{SQuAD} \citep{rajpurkar-etal-2016-squad} & 175,198 \\
        \href{https://huggingface.co/datasets/sentence-transformers/trivia-qa}{TriviaQA} \citep{2017arXivtriviaqa} & 146,692 \\
        \href{https://huggingface.co/datasets/TIGER-Lab/WebInstructSub}{WebInstructSub} \citep{yue2024mammoth2} & 4,670,440 \\
        \href{https://huggingface.co/datasets/miracl/miracl-corpus}{MIRACL} & 27,572,151 \\
        \midrule
        Full mix (\full) & 89,827,558 \\
        \bottomrule
    \end{tabularx}
\end{table}

\paragraph{Corpus.}
Each dataset gives a query and an answer. We embed both separately. We sample 10M question-answer pairs from PAQ. For MIRACL \citep{zhang-etal-2023-miracl}, we sample up to 500,000 passages per language. All available passages are used when a language has fewer than 500,000. (Due to an error, our nemotron embeddings are missing 500,000 Thai passages and 100,000 sampled Telugu passages present for Gemini. Thus the full corpora contain 89.8M Gemini vectors and 89.2M Nemotron vectors.) When comparing activations between the two embedding families, we use the 89.2M shared embeddings.

\paragraph{Embeddings.}
For gemini embeddings, we use \texttt{gemini-embedding-2-preview} with the task type ``RETRIEVAL\_DOCUMENT'' and dimension $d=3072$. The nemotron representations are generated from \texttt{nvidia/llama-embed-nemotron-8b} with $d=4096$. We truncate at 10,000 characters for gemini and 256 tokens for Nemotron.

\paragraph{Additional data splits for P2 experiments.}
We train a separate $m=16,384, k=64$ SAE for each of the following mixtures and each embedding model:
\begin{itemize}
    \item \randomone: emotion, fever, gooaq, natural-questions, sampled PAQ, and scifact; 37,336,279 vectors.
    \item \randomtwo: hotpotqa, msmarco, nfcorpus, squad, trivia-qa, and WebInstructSub; 24,919,128 vectors.
    \item \wiki: fever, hotpotqa, natural-questions, sampled PAQ, squad, and trivia-qa; 42,166,282 vectors.
    \item \nowiki: emotion, gooaq, msmarco, nfcorpus, scifact, WebInstructSub; 20,089,125 vectors.
\end{itemize}
Each mixture contains six complete prepared datasets. The \wiki/\nowiki and \randomone/\randomtwo pairs are two partitions of the same 12 non-MIRACL datasets; both MIRACL collections occur only in \full. \wiki/\nowiki splits based on if the dataset includes Wikipedia data or not.

\paragraph{SAE architecture and objective.}
We use the ordinary (non-batch) TopK SAE implementation from \url{https://github.com/bartbussmann/BatchTopK/tree/main}. For an embedding $x\in\RR^d$ the forward pass is
\begin{align}
    h(x) &= \relu\!\left(W_{\enc}(x-b_{\dec})\right),\\
    z(x) &= \topk_k\!\left(h(x)\right),\\
    \hat{x} &= W_{\dec}z(x)+b_{\dec},
\end{align}
where $W_{\enc}\in\RR^{m\times d}$, $W_{\dec}\in\RR^{d\times m}$, and $b_{\dec}\in\RR^d$. There is no encoder bias. $\topk_k$ retains the $k$ largest nonnegative activations in each example.

Let $D$ be the set of dead neurons: neurons with no positive post-TopK activation in the last 512 consecutive training batches. For each example, $z^{\auxk}(x)$ retains the largest $\min\{32,|D|\}$ entries of the \emph{pre-TopK} activation vector $h(x)$ restricted to $D$, with all other entries set to zero. For a batch of $B$ embeddings, the implemented loss is
\begin{equation}
    \mathcal{L}
    = \frac{1}{Bd}\sum_{i=1}^{B}
      \left[
        \|x_i-\hat{x}_i\|_2^2
        + w_{\mathrm{aux}}\,
          \left\|W_{\dec}z^{\auxk}(x_i)-(x_i-\hat{x}_i)\right\|_2^2
      \right].
\end{equation}
The auxiliary term is zero when $D$ is empty. Its residual $x_i-\hat{x}_i$ remains in the gradient computation. We use no $\ell_1$ penalty. The auxiliary coefficient is $w_{\mathrm{aux}}=0.25$ except for the Gemini width-131,072 SAE, which uses $w_{\mathrm{aux}}=1.0$ (we found that at $0.25$ in this setting, the number of dead neurons would remain large).

\paragraph{Optimization and training budget.}
We train SAEs given in \Cref{tab:sae-training-seeds}. The four additional distribution mixtures give 26 SAEs in the main width/split experiments. We use float32 parameters, batch size $B=2048$, and Adam with learning rate $5\times10^{-4}$, $(\beta_1,\beta_2)=(0.9,0.99)$, $\epsilon=10^{-8}$, and no weight decay or learning-rate schedule. Gradients are clipped to global norm 1. The encoder is initialized with Kaiming-uniform weights, the decoder with the normalized transpose of the encoder, and the bias with zeros. Before each optimizer step, decoder directions are normalized and their radial gradient components are removed.

All of these SAEs train for 1,000,000 optimizer steps except the Gemini width-131,072 SAE, which trains for 2,000,000 steps. These budgets correspond to 2.048 billion and 4.096 billion sampled embeddings, respectively, with repeated sampling from the finite corpora above. A shard is selected with probability proportional to its row count and supplies four batches before the next shard selection. Each batch is a randomly positioned contiguous block with its rows shuffled; shards smaller than one batch are sampled with replacement. Training uses NVIDIA H100, A100 (80GB), or A6000 GPUs. Seeds are fixed separately for each run, rather than paired between embedding models. The synthetic experiments use a separate training recipe described in \Cref{sec:recovery-simulations}.

\begin{table}[htbp]
    \centering
    \small
    \caption{Sparsities for the 26 SAEs in the main width and distribution experiments. Subset runs all use $m=16,384$ and $k=64$. The Gemini width-131,072 run uses 2M steps and auxiliary coefficient 1.0; every other listed run uses 1M steps and coefficient 0.25.}
    \label{tab:sae-training-seeds}
    \begin{tabular}{lr}
        \toprule
        Width or subset & $k$ \\
        \midrule
        512 & 32 \\
        1,024 & 32 \\
        2,048 & 32 \\
        4,096 & 32 \\
        8,192 & 64 \\
        16,384 & 64 \\
        32,768 & 64 \\
        65,536 & 128 \\
        131,072 & 128 \\
        \midrule
        \wiki & 64 \\
        \nowiki & 64 \\
        \randomone & 64 \\
        \randomtwo & 64 \\
        \bottomrule
    \end{tabular}
\end{table}

\paragraph{Further sparsity and seed SAE controls.}
\Cref{tab:sae-control-checkpoints} lists six additional fixed-$k=128$ SAEs and three alternate-seed Gemini SAEs.

\begin{table}[htbp]
    \centering
    \small
    \caption{The nine additional SAEs used for the fixed-sparsity and alternate-seed controls. All use 1,000,000 optimizer steps, batch size 2048, and auxiliary coefficient 0.25.}
    \label{tab:sae-control-checkpoints}
    \begin{tabular}{llrr}
        \toprule
        Control & Model & Width & $k$ \\
        \midrule
        Fixed $k$ & Gemini & 512 & 128 \\
        Fixed $k$ & Gemini & 4,096 & 128 \\
        Fixed $k$ & Gemini & 32,768 & 128 \\
        \midrule
        Fixed $k$ & Nemotron & 512 & 128 \\
        Fixed $k$ & Nemotron & 4,096 & 128 \\
        Fixed $k$ & Nemotron & 32,768 & 128 \\
        \midrule
        Alternate seed & Gemini & 512 & 32 \\
        Alternate seed & Gemini & 4,096 & 32 \\
        Alternate seed & Gemini & 32,768 & 64 \\
        \bottomrule
    \end{tabular}
\end{table}

\subsection{Matching criteria and computation}
\label{sec:matching-metrics}
The P1, P2, and cross-model comparisons combine a pairwise similarity score with a rule for selecting matches. We define the shared metrics here; the experiment-specific sections specify the dictionaries and evaluation corpora. We use a defualt threshold of $t=0.7$ \citep{paulo2025sparse}.

\subsubsection{Direction similarity}
\label{sec:direction-similarity}
For nonzero dictionary directions $u$ and $v$ in the same embedding space, we use signed cosine similarity,
\[
    s_{\mathrm{dir}}(u,v)=\frac{\langle u,v\rangle}{\|u\|_2\|v\|_2}.
\]
For SAEs, $u$ and $v$ are decoder columns. The PCA baseline uses $|s_{\mathrm{dir}}(u,v)|$ for both maximum-cardinality and exclusive matching, because the signs of principal components are arbitrary.

\subsubsection{Activation similarity}
\label{sec:activation-similarity}
For two SAE features evaluated on the same examples $x_1,\ldots,x_N$, let $a_n$ and $b_n$ be their continuous post-TopK activations. With $\bar a=N^{-1}\sum_n a_n$ and $\bar b=N^{-1}\sum_n b_n$, their signed Pearson correlation is
\[
    s_{\mathrm{act}}(a,b)=
    \frac{\sum_{n=1}^N(a_n-\bar a)(b_n-\bar b)}
    {\sqrt{\sum_{n=1}^N(a_n-\bar a)^2}\sqrt{\sum_{n=1}^N(b_n-\bar b)^2}}.
\]
Every evaluation example receives equal weight, including examples with zero activation. The metric depends on the choice of evaluation data.

\subsubsection{One-to-one matching and persistence}
\label{sec:matching-computation}
For either similarity score, eligible pairs have similarity at least $t$ and form the edges of a bipartite graph. A maximum-cardinality matching is the largest set of eligible pairs in which each feature occurs at most once. We denote this maximum cardinality by $\mu_t(U,V)$. For PCA directions, we use the absolute-cosine modification above. Reported proportions divide the matching count by the source dictionary width, including inactive and unmatched features in the denominator.

For all maximum-matching persistence results, we compute a maximum-cardinality matching $M_\ell$ independently on the threshold graph between the full source dictionary $U$ and each larger target dictionary $V_\ell$, using retained candidate edges for empirical activation comparisons (\Cref{sec:matching-candidates}). Let
\[
    S_\ell=\{u\in U:\text{there exists }v\in V_\ell\text{ with }(u,v)\in M_\ell\},
    \qquad S_{\mathrm{persistent}}=\bigcap_{\ell=1}^{s}S_\ell.
\]

Maximum-cardinality matchings can be nonunique, so the persistent set and its size can depend on tie-breaking. We retain ascending original source and target feature IDs, canonicalize each graph as a sorted CSR matrix, and use SciPy 1.15.3's \texttt{maximum\_bipartite\_matching} with \texttt{perm\_type='column'}. This fixes the matching choices for these results.

\subsubsection{Candidate graph construction}
\label{sec:matching-candidates}
Direction comparisons evaluate all pairs. For activation comparisons, we use 512-dimensional CountSketch representations to propose 32 candidates per feature in each search direction, with the same evaluation rows and sketch hashes for both dictionaries. We recompute candidate correlations over the full evaluation corpus and retain the five highest-scoring candidates per feature in each direction. Thresholded matching uses the union of retained pairs with correlation at least $t$.

The sketch hash is the fixed SplitMix64 transformation of the original global row index, implemented by \texttt{row\_hash} in \path{scripts/match_cross_model_sae_activations_131k.py}.

The within-family width sweeps use 89,827,558 Gemini rows and 89,227,558 Nemotron rows, respectively. Same-width cross-model comparisons use the 89,227,558 aligned shared rows, with hashes based on their original Gemini offsets. Split comparisons instead use the row distribution identified in each evaluation manifest. These frozen manifests record ordered shard identities, offsets, model-checkpoint SHA-256 hashes, and producing-code hashes. The cross-model \texttt{alignment\_audit.json} additionally records SHA-256 hashes of each shard's ordered source-row IDs after checking dataset/configuration/split/shard identity, text columns, and row order across the two models. The reproduction workflows preserve these manifests and retained neighbor arrays; the code archive's \texttt{SOURCE\_MANIFEST.json} identifies the supplied source revision.

Pairwise maximum matching is exact on these retained graphs, and its cardinality is a lower bound on exhaustive pairwise Pearson matching.

\subsubsection{Alternative exclusive matching criterion}
\label{sec:exclusive-matching}
The exclusive direction comparisons in \Cref{fig:persistent-stability-strict,fig:dataset-heatmaps-exclusive} count a pair $(u,v)$ only when its similarity exceeds $t$ and neither endpoint has another above-threshold partner. We use $s(u,v)=s_{\mathrm{dir}}(u,v)$ for SAEs and $s(u,v)=|s_{\mathrm{dir}}(u,v)|$ for PCA, applying the same score to pair admission and competing partners. The requirements are
\begin{equation}
    s(u,v)>t,\qquad
    s(u,v')\leq t\ \text{for all }v'\neq v,\qquad
    s(u',v)\leq t\ \text{for all }u'\neq u.
\end{equation}
The exclusive pairs themselves form a feasible one-to-one matching, so their count is bounded above by $\mu_t$. Lowering $t$ can reduce the exclusive count by introducing competitors, whereas maximum-cardinality counts cannot decrease.

\subsection{P1: Persistence across SAE widths}
\label{sec:p1-main-protocol}
We apply the independent pairwise matching and source-set intersection procedure in \Cref{sec:matching-computation} separately within Gemini and Nemotron, using each source SAE below width 131,072 and every larger trained SAE as targets. We report counts and source-width-normalized proportions. The largest SAE is a target only, since it has no larger comparison dictionary. \Cref{fig:persistent-stability} uses the signed decoder-cosine metric in \Cref{sec:direction-similarity}.

\subsection{P2: Shared features across training distributions}
\label{sec:p2-main-protocol}
\Cref{fig:dataset-heatmaps} compares dictionaries trained on \full, \wiki, \nowiki, \randomone, and \randomtwo, separately for each embedding model and dictionary family. SAE dictionaries contain 16,384 directions; PCA uses all $d$ principal components, with $d=3072$ for Gemini and $d=4096$ for Nemotron.

\paragraph{PCA dictionaries.}
We fit PCA separately for each embedding model and training mixture, using every embedding in that mixture. We stream the data to accumulate the count $N$, sum $s=\sum_{i=1}^N x_i$, and second-moment matrix $Q=\sum_{i=1}^N x_i x_i^\top$ in float64, then compute
\[
    \bar{x}=s/N,\qquad
    \Sigma=\frac{Q-N\bar{x}\bar{x}^{\top}}{N-1}.
\]
We diagonalize this centered covariance matrix and order its orthonormal eigenvectors by decreasing eigenvalue. The PCA dictionary consists of these directions in the original embedding space; P2 uses the full $d$-component basis. PCA fitting uses the stored embeddings without additional length normalization or coordinate standardization.

\paragraph{Dictionary matching.}
For each pair of training distributions, \Cref{fig:dataset-heatmaps} applies the metric in \Cref{sec:matching-metrics}. Exclusive direction matching and distribution-specific activation comparisons are reported in \Cref{sec:p2-extended}.

\subsection{P3: Hierarchical datasets and feature-recovery protocol}
\label{sec:p3-hierarchy-protocol}

For \Cref{fig:hierarchy-mean-f1}, we evaluate individual SAE features against binary category labels in three domains: word prefixes, geography, and organism taxonomy. The labels define hierarchical relationships, such as a two-letter prefix implying its first letter or a city's country determining its continent (the organism taxonomy has more levels).

\paragraph{English prefixes (\prefix).}
We use the English frequency list from the \href{https://github.com/rspeer/wordfreq}{\texttt{wordfreq} package}. We retain the top 200,000 words in frequency order. Each normalized word is embedded on its own. The first character defines a parent label, and the first two characters define a child label for words of length at least two. We retain categories with at least 100 positive word rows across the full evaluation collection, giving 26 one-letter and 216 two-letter categories per embedding model.

\paragraph{Geography (\cities).}
We use all 69,182 records in the experiment's \href{https://download.geonames.org/export/dump/readme.txt}{GeoNames \texttt{cities5000} snapshot}, with country labels from each record's ISO country code and continent labels obtained from \texttt{countryInfo.txt}. We retain 87 country and six continent categories with at least 100 city rows each. For the reported results, we embed two question variants of every city name: \texttt{what country is [city] in} for country prediction and \texttt{what continent is [city] in} for continent prediction, following the use of elicitation in \citet{bhalla2026sparse}. (We found that without this elicitation, SAEs at no widths reliably could detect the continent.)

\paragraph{Organism taxonomy (\gbif).}
We use the \href{https://doi.org/10.15468/39omei}{GBIF Backbone Taxonomy}. From its \texttt{Taxon.tsv} and \texttt{VernacularName.tsv} tables, we select accepted species having at least one nonempty English vernacular name. This yields 123,495 species. We embed one scientific-name row per species and 259,286 deduplicated English common-name strings, leading to 382,781 name rows in total. Binary labels follow the species' kingdom, phylum, class, order, family, and genus. We retain categories with at least 100 distinct species. This gives 4 kingdoms, 12 phyla, 33 classes, 131 orders, 276 families, and 43 genera. Once labels are constructed, F1 is evaluated over name rows, so species with more names can contribute more evaluation examples.

\paragraph{Selection and test splits.}
We partition the data into selection and test sets. Word rows are divided equally within each two-letter-prefix stratum, with one-character words forming their own strata. City rows are divided equally within each country, and both question variants inherit this partition. GBIF partitions the underlying species into 61,748 selection species and 61,747 test species, with scientific names inheriting the species assignment. Common-name rows are assigned a single split, even when their species associations span both partitions.

Every probe reads these saved split labels unchanged, and the same assignments are used across both embedding models and all widths. Exact reruns require these row tables and their source snapshots: the provided preparation scripts supply deterministic seed-1729 reconstructions for the word and city splits, but do not reproduce the historical row-by-row assignments. Fresh source downloads can also change the evaluation collections.

\paragraph{Embedding and SAE evaluation.}
We compute gemini and nemotron embeddings for each input and compute activations for the nine full-mix SAEs for each model, with widths 512 through 131,072 and the corresponding TopK sparsities in \Cref{tab:sae-training-seeds}.

\paragraph{Selecting a single-feature probe.}
For a binary target $z_i^*$ and SAE activations $\hat z$, define
\[
    g_{j,\tau}(x)=\mathbf{1}\{\hat z_j(x)\geq\tau\},\qquad \tau>0.
\]
For each coordinate active on selection rows, we sort its positive selection activations and evaluate every distinct observed value as an inclusive threshold. We choose the coordinate and threshold maximizing selection F1, requiring at least one selection true positive.

For $E=\{z_i^*=1\}$ and $F=\{g_{j,\tau}=1\}$, we use
\[
\operatorname{F1}_{\DD}(E,F)
=\frac{2\Pr_{\DD}(E\cap F)}{\Pr_{\DD}(E)+\Pr_{\DD}(F)}
=\frac{2\mathrm{TP}}{2\mathrm{TP}+\mathrm{FP}+\mathrm{FN}}
\]
on the relevant split. This evaluates how well a single thresholded feature identifies a category.

\paragraph{\Cref{fig:hierarchy-mean-f1} details.}
Each point in \Cref{fig:hierarchy-mean-f1} is the unweighted mean of held-out test F1 over all retained categories at that level and both embedding models. The number of category--model pairs is twice the number of categories listed above, and the same pairs contribute at every width. The main-text a-prefix example, empirical Food and Music hierarchies, and fixed-sparsity recovery analysis are described in \Cref{sec:p3-extended}.

\clearpage

\section{Extended Experiments for P1: Persistence across Size}
\label{sec:p1-extended}
We extend the main persistent feature experiment in \Cref{fig:persistent-stability} with controls for sparsity and training seed, activation-based matching, alternative matching criteria and similarity thresholds, and analyses of feature prevalence.

\subsection{Fixed sparsity}
\label{sec:fixed-sparsity-control}
At widths 512, 4,096, and 32,768, we replace each original source SAE by its $k=128$ counterpart, retaining every original larger-width target and using signed cosine at threshold $0.7$ with the independent pairwise matching and intersection criterion in \Cref{sec:matching-computation}.

\begin{table}[htbp]
    \centering
    \small
    \caption{Fixed-sparsity source replacement, using the same full-sweep target sets as the seed control. Only the source changes from its original $k$ to $k=128$; all original larger-width targets are retained. Same-width counts match the original and replacement source dictionaries. All comparisons use signed decoder cosine similarity $\geq0.7$; percentages divide by source width.}
    \label{tab:fixed-sparsity-persistence}
    \begin{tabular}{llccc}
        \toprule
        Model & \shortstack{Original\\width/$k$} & \shortstack{Same-width\\cross-sparsity matches} & \shortstack{Original source\\persistence} & \shortstack{$k=128$ source\\persistence} \\
        \midrule
        Gemini & 512/32 & 389 (76.0\%) & 303 (59.2\%) & 221 (43.2\%) \\
         & 4,096/32 & 2,949 (72.0\%) & 2,974 (72.6\%) & 2,405 (58.7\%) \\
         & 32,768/64 & 16,466 (50.3\%) & 15,881 (48.5\%) & 12,893 (39.3\%) \\
        \midrule
        Nemotron & 512/32 & 410 (80.1\%) & 259 (50.6\%) & 210 (41.0\%) \\
         & 4,096/32 & 3,286 (80.2\%) & 2,841 (69.4\%) & 2,607 (63.6\%) \\
         & 32,768/64 & 18,198 (55.5\%) & 12,750 (38.9\%) & 13,609 (41.5\%) \\
        \bottomrule
    \end{tabular}
\end{table}

\Cref{tab:fixed-sparsity-persistence} shows persistence remains substantial, but is somewhat lower in general, especially for smaller widths. This seems to suggest that the conventional wisdom to increase $k$ alongside $m$ is reasonable (and also is somewhat in alignment with our simulations \Cref{sec:recovery-simulations}).

\subsection{Variation by seed}
\label{sec:seed-replication}
We repeat the direction comparisons for gemini SAEs 512/32, 4,096/32, and 32,768/64 (width/$k$) with different training seeds. \Cref{tab:seed-direction-replication} shows a substantial number of matches across seeds, and that persistent remains similar (suggesting that persistence doesn't vary heavily between SAEs). The same-width column uses pairwise maximum matching (\Cref{sec:matching-computation}); the persistence columns use independent pairwise matching and intersection (\Cref{sec:matching-computation}). Consistent with \cite{paulo2025sparse}, there are a substantial number of different features learned when varying the random seed (especially for larger widths); simultaneously, there are many features that have matches even in all larger dictionaries.

\begin{table}[htbp]
    \centering
    \small
    \caption{Gemini direction stability across training seeds. Persistence uses the independent pairwise matching and intersection criterion.}
    \label{tab:seed-direction-replication}
    \begin{tabular}{lccc}
        \toprule
        Width/$k$ & \shortstack{Same-width\\cross-seed matches} & \shortstack{Original source\\persistence} & \shortstack{Alternate source\\persistence} \\
        \midrule
        512/32 & 470 (91.8\%) & 303 (59.2\%) & 293 (57.2\%) \\
        4,096/32 & 3,564 (87.0\%) & 2,974 (72.6\%) & 2,979 (72.7\%) \\
        32,768/64 & 17,856 (54.5\%) & 15,881 (48.5\%) & 15,998 (48.8\%) \\
        \bottomrule
    \end{tabular}
\end{table}

\clearpage

\subsection{Exclusive matching}
\label{sec:stability-comparisons}
\Cref{fig:persistent-stability-strict} repeats the source--target comparisons in \Cref{sec:p1-main-protocol} using the exclusive persistence criterion in \Cref{sec:exclusive-matching}.

\begin{figure}[htbp]
    \centering
    \includegraphics[width=0.9\linewidth]{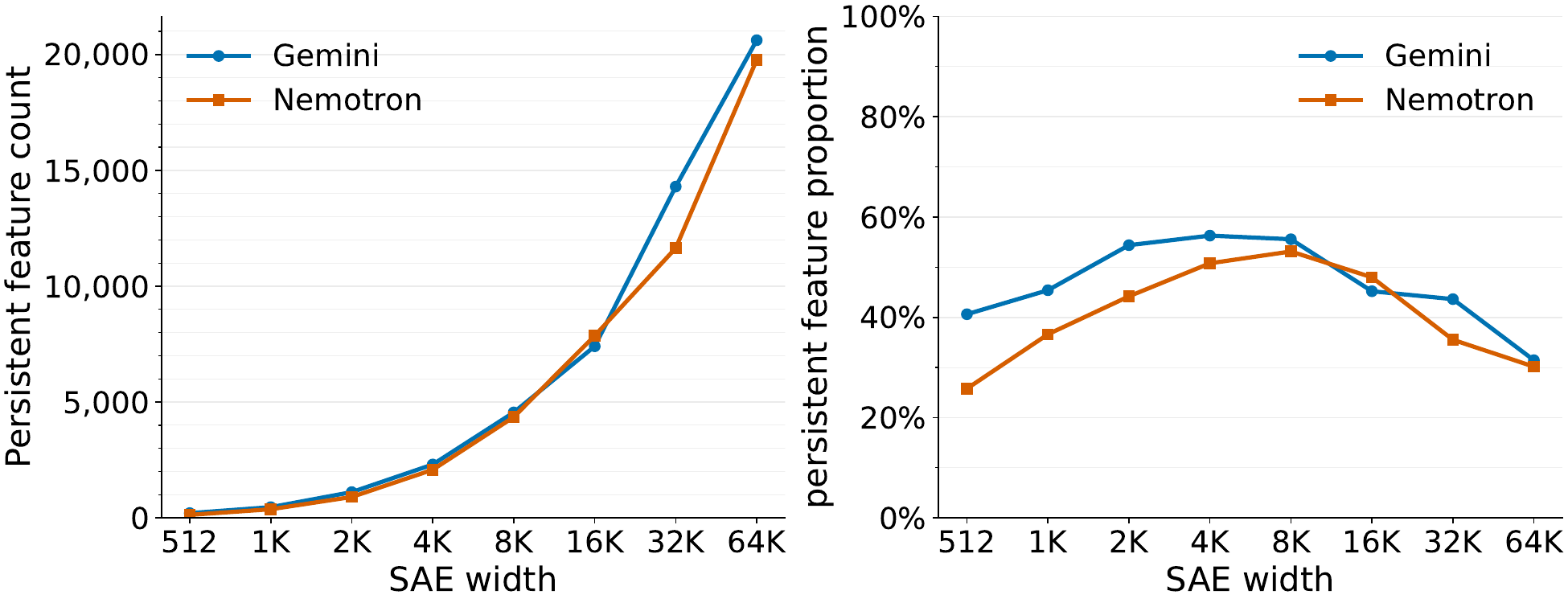}
    \caption{Persistent stability under strict exclusive matching at $0.7$. A source direction counts only if it has a signed-cosine match above $0.7$ in every larger trained dictionary, with no competing partner above $0.7$ at either endpoint in each comparison.}
    \label{fig:persistent-stability-strict}
\end{figure}

\subsection{Persistence by activation similarity}
\label{sec:activation-persistence}
\Cref{fig:persistent-stability-pearson} repeats the source--target comparisons in \Cref{sec:p1-main-protocol} using signed Pearson correlation at threshold $0.7$, with the metric and candidate-graph procedure defined in \Cref{sec:activation-similarity,sec:matching-candidates}.

\begin{figure}[htbp]
    \centering
    \includegraphics[width=0.9\linewidth]{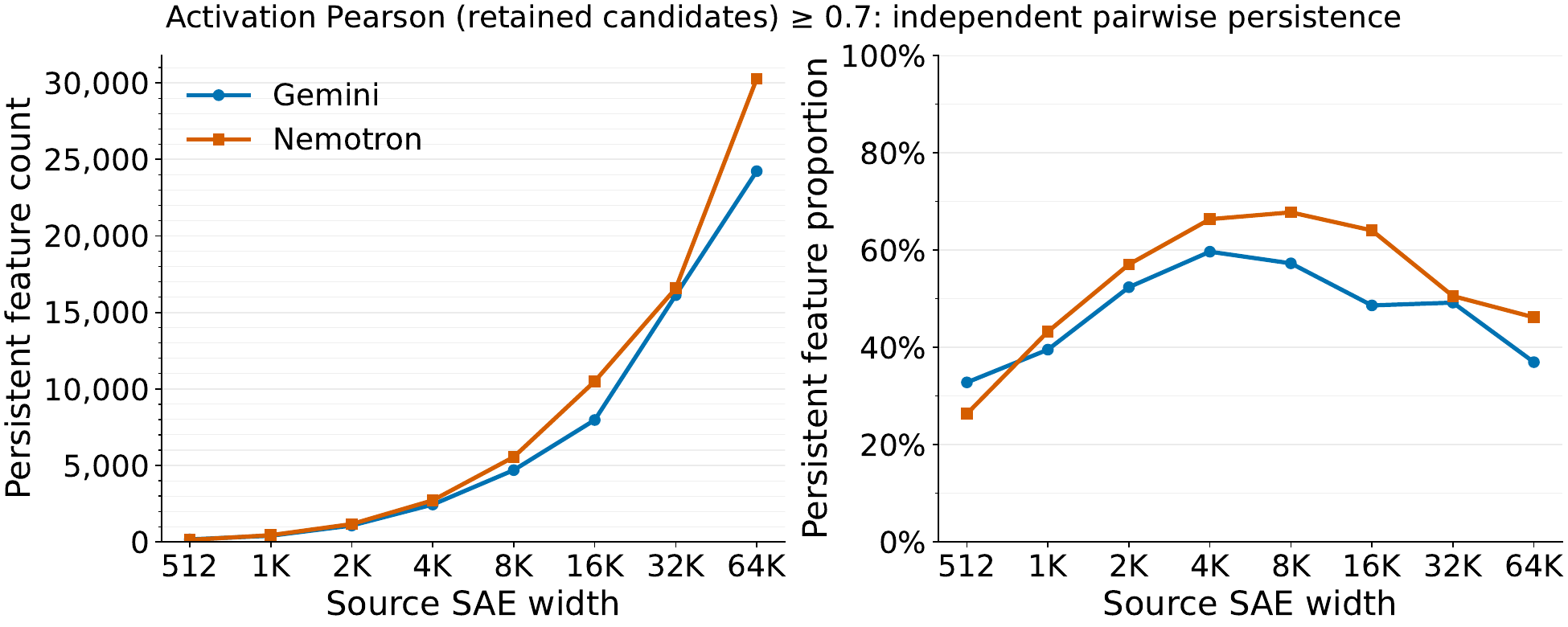}
    \caption{Persistent stability across SAE widths using activation similarity at signed Pearson correlation at least $0.7$. Left: the number of source features matched in every independently computed pairwise maximum matching to a larger SAE on the retained candidate graphs. Right: this count divided by source width. Both panels compare Gemini and Nemotron, using the fixed tie-breaking procedure in \Cref{sec:matching-computation}.}
    \label{fig:persistent-stability-pearson}
\end{figure}

\clearpage

\subsection{Sensitivity to the matching threshold}
\label{sec:persistence-threshold-sensitivity}
We vary the matching threshold over $t\in\{0.5,0.6,0.7,0.8\}$ using the independent pairwise matching and intersection criterion in \Cref{sec:matching-computation}. \Cref{fig:persistent-stability-thresholds} uses signed decoder cosine similarity, and \Cref{fig:persistent-stability-pearson-thresholds} uses signed activation Pearson correlation.

\begin{figure}[htbp]
    \centering
    \includegraphics[width=\linewidth]{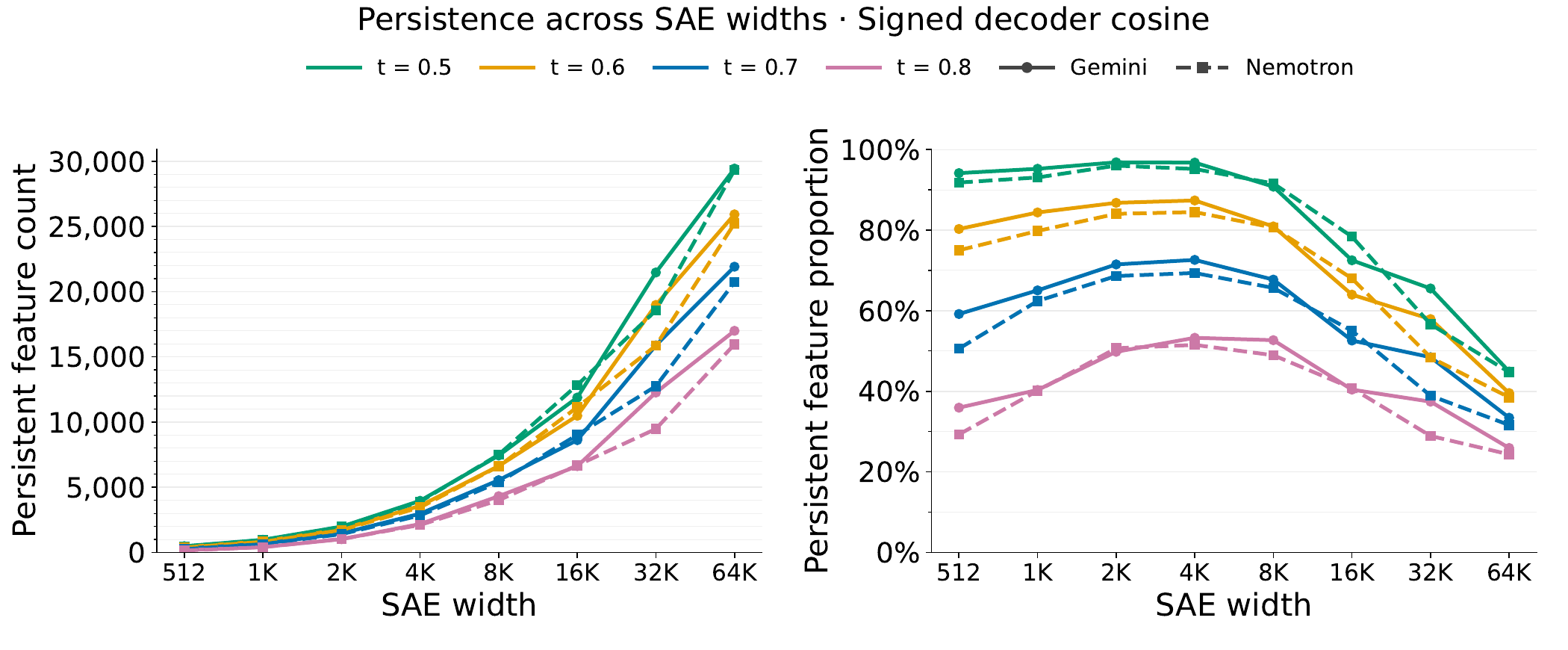}
    \caption{Threshold sensitivity of direction persistence.}
    \label{fig:persistent-stability-thresholds}
\end{figure}

\begin{figure}[htbp]
    \centering
    \includegraphics[width=\linewidth]{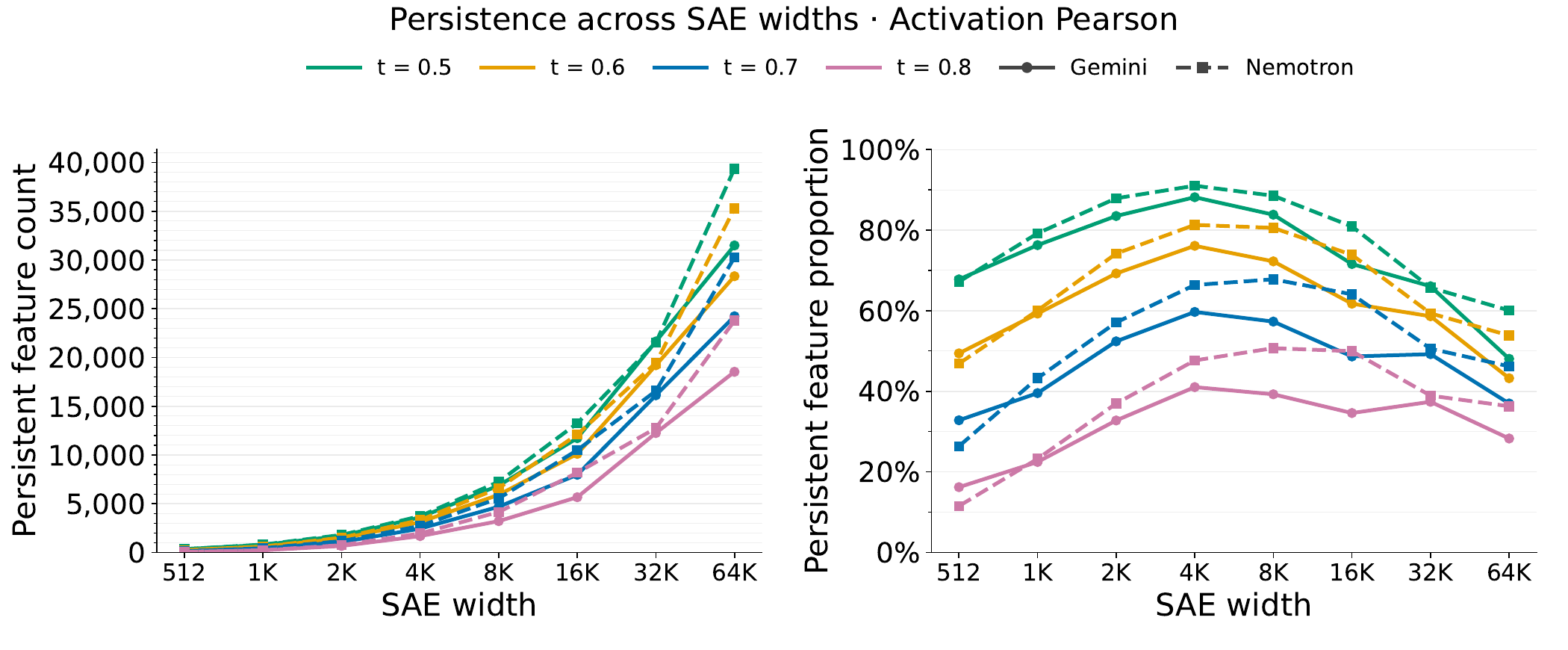}
    \caption{Threshold sensitivity of activation persistence.}
    \label{fig:persistent-stability-pearson-thresholds}
\end{figure}

\clearpage

\subsection{Feature activation rates}
\label{sec:activation-rates}
\Cref{fig:activation-rates} compares prevalence across SAE widths and KMeans cluster counts. For each SAE feature, we compute the fraction of the model's full evaluation corpus on which its continuous post-TopK activation exceeds the SAE's pooled median positive activation. This is one threshold per SAE, computed over all positive activation events across features and examples.

For KMeans, prevalence is the fraction of examples assigned to each nearest centroid. We use Euclidean KMeans in the full embedding space, trained on every row of the corresponding full corpus for 100 total iterations, at 512, 4,096, 16,384, and 131,072 clusters. Both methods are evaluated on the same full corpus within each embedding family: 89,827,558 Gemini examples and 89,227,558 Nemotron examples.

We plot kernel density estimates of $\log_{10}$ activation rates, using a bandwidth of $0.12$ and scaling each curve to integrate to the number of features with positive rates. Zero-rate features are omitted from the logarithmic axis. The vertical axis uses a square-root scale with tick labels in the original density units.

The total mass and average prevalence are partly constrained by dictionary width and sparsity. For $N$ examples and an SAE of width $m$ with TopK parameter $k$, there are at most $Nk$ positive activation events. Strictly exceeding their pooled median selects at most half of these events, with ties potentially reducing this number. Thus, if $p_j$ denotes feature $j$'s above-median prevalence, then $\frac{1}{m}\sum_{j=1}^{m}p_j\leq k/(2m)$. For KMeans, cluster prevalences sum to one and have mean $1/m$. Consequently, a growing low-prevalence tail and changes in the count-scaled density area partly reflect these constraints. The observed curve shapes provide descriptive evidence consistent with the recovery principle, rather than identifying atoms on their own.

We observe a ``nested'' behavior, in which larger SAEs tend to include a similar number of coarse features as smaller SAEs, while also learning a longer tail of more granular features, consistent with the recovery principle. The KMeans prevalence distributions instead shift toward smaller clusters as the number of clusters increases.

\begin{figure}[htbp]
    \centering
    \includegraphics[width=\linewidth]{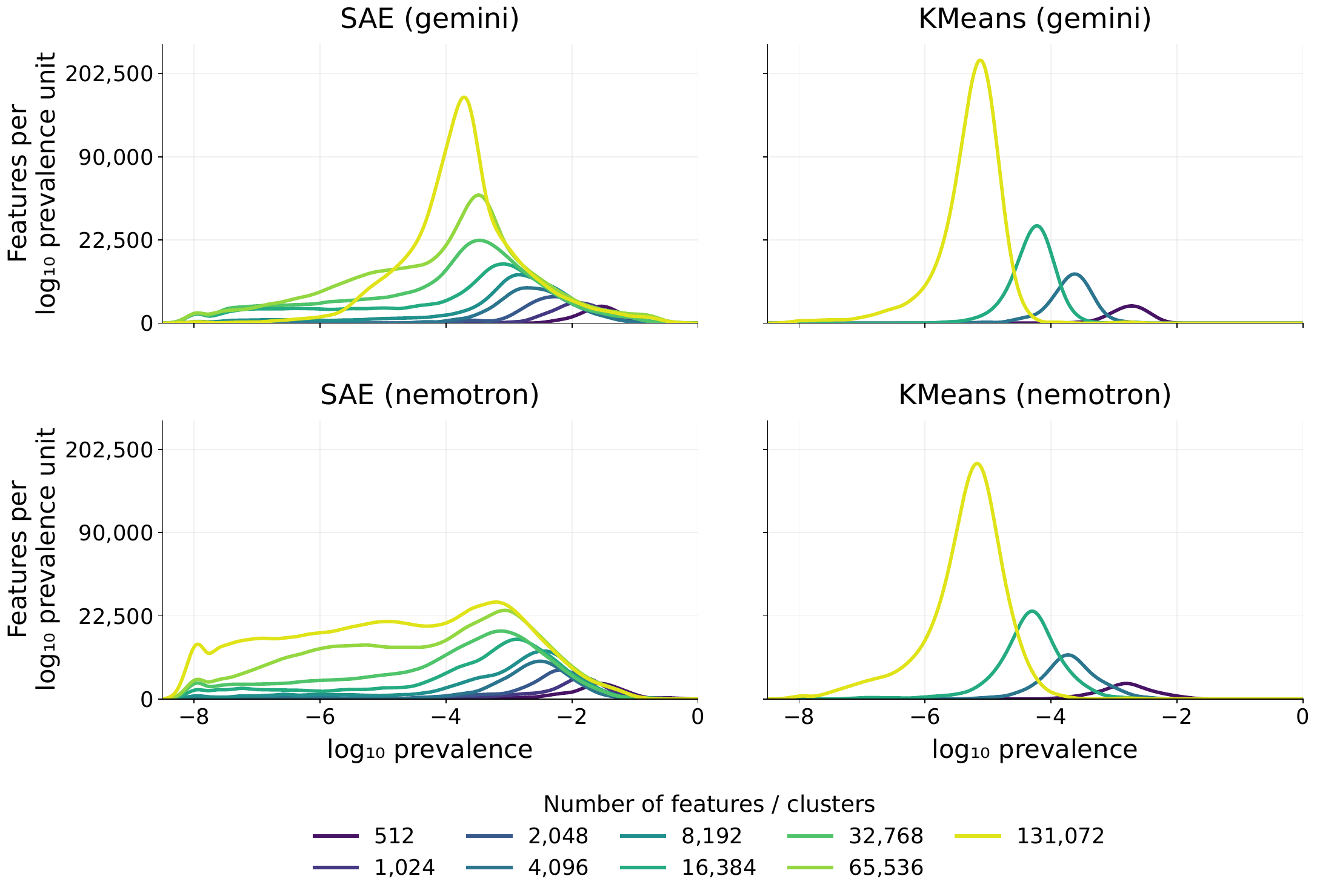}
    \caption{Prevalence across SAE widths (left) and KMeans cluster counts (right), for Gemini (top) and Nemotron (bottom). SAE prevalence uses activation strictly greater than each SAE's pooled median positive activation; KMeans prevalence uses nearest-centroid assignment. KMeans models are trained on the full corpus for 100 iterations. Colors indicate dictionary size. Curves show feature-count-scaled densities of $\log_{10}$ prevalence, with shared axes and a square-root vertical scale. Zero-rate features are omitted from the logarithmic axis.}
    \label{fig:activation-rates}
\end{figure}

\clearpage

To control for varying sparsities, \Cref{fig:activation-rates-fixed-k128} repeats the comparison with TopK sparsity fixed at $k=128$ for all SAE widths. We use widths 512, 4,096, 32,768, 65,536, and 131,072, with the same prevalence definitions, full evaluation corpora, and KMeans baselines as above.

\begin{figure}[htbp]
    \centering
    \includegraphics[width=\linewidth]{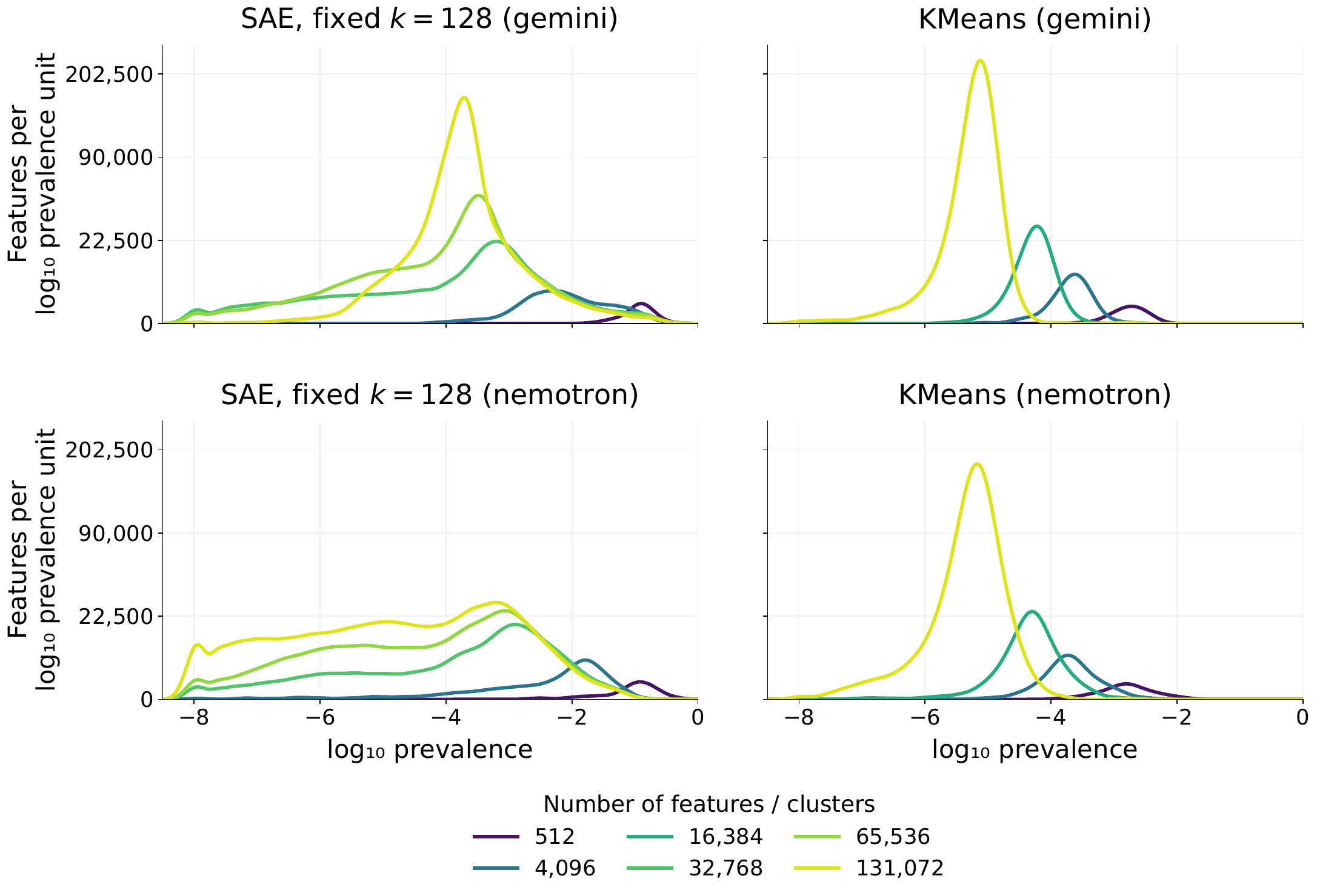}
    \caption{Prevalence for SAEs with fixed TopK sparsity $k=128$ (left) and KMeans (right), for Gemini (top) and Nemotron (bottom). SAE widths are 512, 4,096, 32,768, 65,536, and 131,072; KMeans cluster counts are 512, 4,096, 16,384, and 131,072. SAE prevalence is the fraction of examples with activation strictly above each SAE's pooled median positive activation; KMeans prevalence is nearest-centroid assignment frequency after 100 full-corpus training iterations. Colors indicate the number of features or clusters. Curves show feature-count-scaled densities of $\log_{10}$ prevalence with bandwidth $0.12$, shared axes, and a square-root vertical scale. Zero-rate features are omitted from the logarithmic axis.}
    \label{fig:activation-rates-fixed-k128}
\end{figure}

\clearpage

\subsection{Relationship between persistence and prevalence}
\label{sec:prevalence-persistence}
We directly compare feature prevalence with direction persistence using the median activation threshold from \Cref{fig:activation-rates}. For an SAE with continuous post-TopK activations $\hat z_i(x)$, let $\tau_{\mathrm{SAE}}$ be the pooled median of all strictly positive activation events, across features and examples in its embedding model's full evaluation corpus. We define
\begin{equation}
    p_i(D)=\frac{1}{|D|}\sum_{x\in D}\mathbf{1}\{\hat z_i(x)>\tau_{\mathrm{SAE}}\}.
    \label{eq:median-feature-prevalence}
\end{equation}
There is one threshold per SAE, rather than a separate median for each feature; the denominator includes every example. We use the full corpora for these measurements.

For each source width in the original sweep, we compare $p_i$ for persistent features and non-persistent features.

\begin{figure}[htbp]
    \centering
    \includegraphics[width=\linewidth]{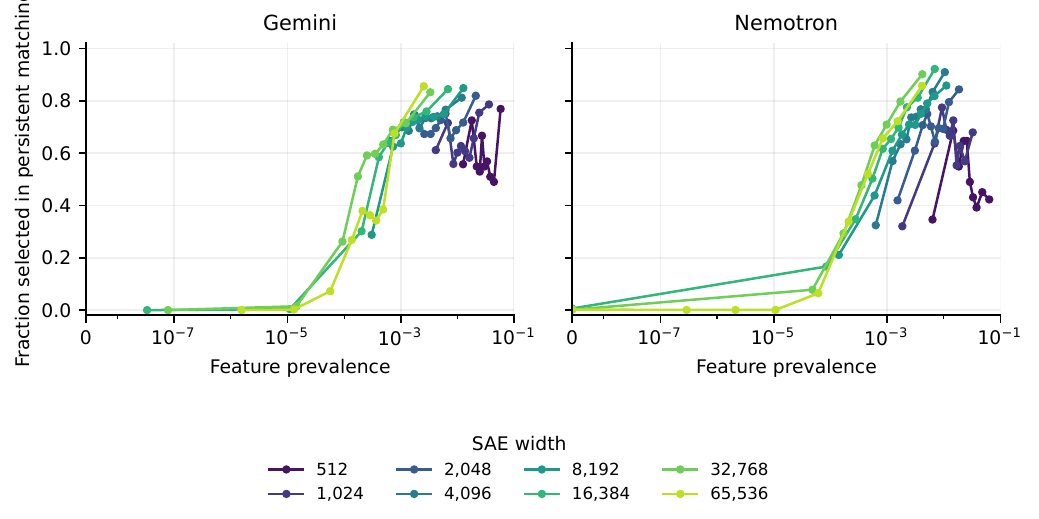}
    \caption{Median-threshold feature prevalence and direction persistence under independent pairwise matching and intersection in Gemini (left) and Nemotron (right). Each curve represents a source width in the original sparsity schedule. Points group source features into up to ten prevalence-quantile bins, keeping ties together, and show the median prevalence and fraction selected in the persistent matching within each bin. All source features are included. Bin sizes can differ; lines connect descriptive bin summaries. The horizontal axis is logarithmic above $10^{-8}$ with a linear segment retaining zero-prevalence bins.}
    \label{fig:width-prevalence-median}
\end{figure}

We measure the AUC of feature prevalence as a predictor of persistence. At widths 16,384, 32,768, and 65,536, AUCs are $0.814$, $0.801$, and $0.822$ for Gemini, and $0.818$, $0.905$, and $0.917$ for Nemotron. Restricting to features with nonzero prevalence gives respective ranges of $0.787$--$0.821$ and $0.792$--$0.905$. The association is weaker at the smallest widths.

Overall, these results are consistent with the recovery principle's prediction that the most prevalent atomic features are recovered.

\clearpage

\section{Extended Experiments for P2: Stability across Distributions}
\label{sec:p2-extended}
We extend the direction-based comparison in \Cref{fig:dataset-heatmaps} using the exclusive matching criterion, activation-based matching, and the relationship between cross-distribution stability and cross-distribution prevalence.

\subsection{Cross-distribution stability under exclusive matching}
\label{sec:split-exclusive-matching}
\Cref{fig:dataset-heatmaps-exclusive} applies the exclusive criterion from \Cref{sec:exclusive-matching} to SAE and PCA dictionaries trained on different distributions, at $t=0.7$. This tests whether cross-distribution agreement persists when each match must have an unambiguous partner, allowing either sign for PCA.

\begin{figure}[htbp]
    \centering
    \includegraphics[width=0.61\linewidth]{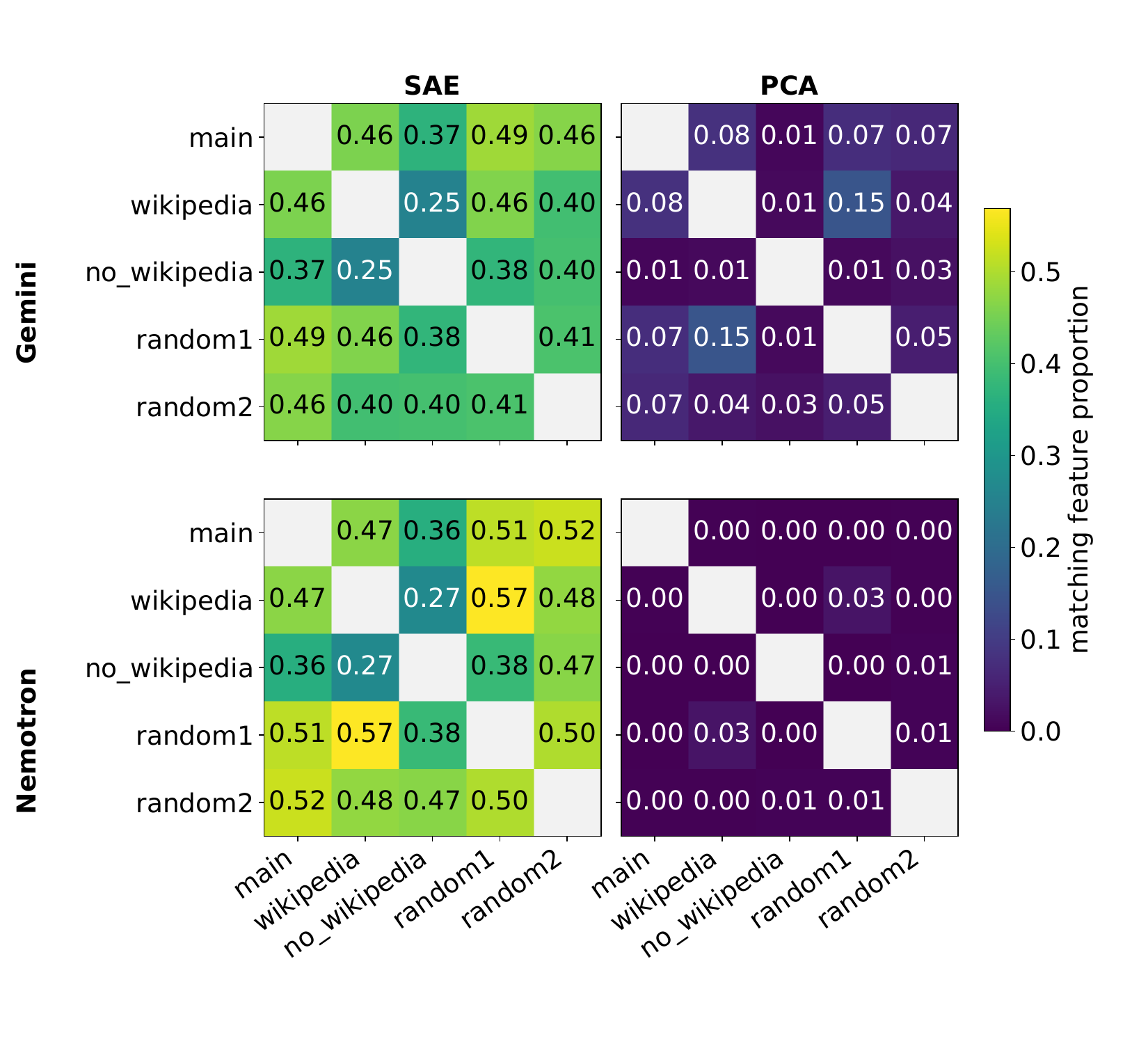}
    \caption{Cross-distribution SAE and PCA matches under the alternative exclusive criterion at $t=0.7$, using signed cosine for SAEs and absolute cosine for PCA. Each counted pair exceeds the threshold and has no other above-threshold partner at either endpoint. Entries divide the pair count by the dictionary width. Compared with the maximum-cardinality matching in \Cref{fig:dataset-heatmaps}, this criterion discards all pairs involving competing above-threshold partners. It therefore gives a stricter measure of unambiguous agreement between dictionaries.}
    \label{fig:dataset-heatmaps-exclusive}
\end{figure}

\clearpage

\subsection{Cross-distribution stability by activation similarity}
\label{sec:split-activation-matching}
\Cref{fig:dataset-heatmaps-activation} applies the activation metric, pairwise matching rule, and normalization in \Cref{sec:matching-metrics} to the width-16,384 SAEs trained on \full, \wiki, \nowiki, \randomone, and \randomtwo. Within each embedding model, the cell at row $D$ and column $D'$ evaluates both SAEs on all examples in $D$. Reversing the row and column changes the evaluation distribution, so the heatmap need not be symmetric.

We construct sketches and candidate graphs for each row distribution using \Cref{sec:matching-candidates}, with moments computed over that distribution.

Matched proportions range from $22.0\%$ to $48.9\%$ for Gemini and from $29.3\%$ to $60.7\%$ for Nemotron.

\begin{figure}[htbp]
    \centering
    \includegraphics[width=\linewidth]{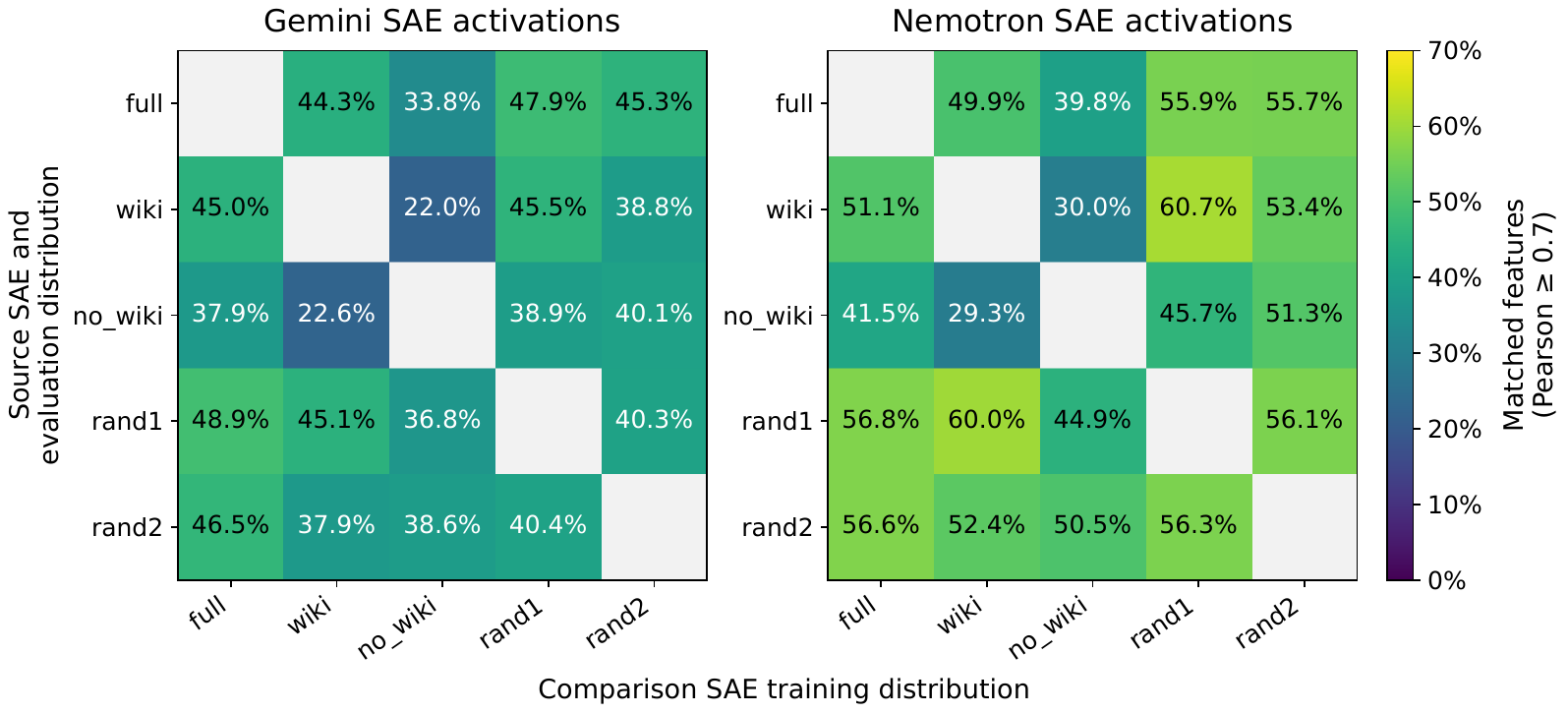}
    \caption{Cross-distribution SAE activation matching for Gemini (left) and Nemotron (right). Each cell reports the maximum number of one-to-one matches at signed Pearson correlation $\geq0.7$, divided by the SAE width of 16,384. Both SAEs are evaluated on the row's distribution; columns identify the comparison SAE's training distribution.}
    \label{fig:dataset-heatmaps-activation}
\end{figure}

\clearpage

\subsection{Feature prevalence and cross-distribution matching}
\label{sec:prevalence-distributions}
We test whether features matched between distribution-specific SAEs are more prevalent in both distributions, using the median-threshold definition in \Cref{eq:median-feature-prevalence}. \Cref{cor:stability-data} suggests that the atomic features that are prevalent in both data distributions are recovered.

For each width-16,384 source SAE trained on $D$, we compute its pooled median positive activation on the embedding model's full evaluation corpus, including MIRACL. We freeze that single threshold when evaluating the source SAE on both $D$ and the target training distribution $D'$, and summarize prevalence by
\begin{equation}
    b_i(D,D')=\min\{p_i(D),p_i(D')\}.
\end{equation}
Both frequencies refer to the same feature in the same source SAE, including when it has no match in the target SAE. This avoids assigning a partner to unmatched features. Reversing source and target evaluates a different SAE, so the prevalence analysis is directional.

We use the disjoint \wiki/\nowiki and \randomone/\randomtwo dataset partitions from \Cref{sec:setup}. These evaluation pairs exclude MIRACL, although threshold calibration uses the full corpus. They contain respectively 42,166,282/20,089,125 and 37,336,279/24,919,128 examples; they partition datasets rather than equal numbers of rows. The outcome is membership in one optimal SAE direction matching under the protocol in \Cref{sec:p2-main-protocol}.

\begin{figure}[htbp]
    \centering
    \includegraphics[width=\linewidth]{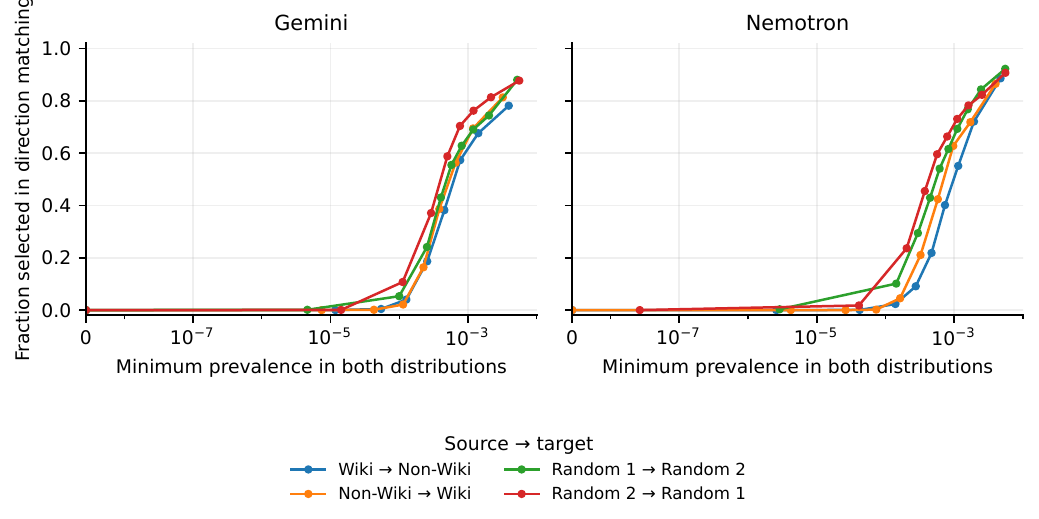}
    \caption{Median-threshold prevalence in both distributions and cross-distribution direction matching for Gemini (left) and Nemotron (right). Curves identify the source and target training distributions. The horizontal coordinate is the bin median of the minimum prevalence of a source feature across the two distributions; the vertical coordinate is the fraction selected in a one-to-one direction matching. Quantile binning preserves ties and includes all source features, with the same axis convention as \Cref{fig:width-prevalence-median}.}
    \label{fig:split-prevalence-median}
\end{figure}

Higher prevalence in both distributions is associated with matching in all eight directional comparisons. AUCs, defined as in \Cref{sec:prevalence-persistence}, range from $0.864$ to $0.923$ for Gemini and from $0.839$ to $0.927$ for Nemotron. They remain large after restricting to features with nonzero prevalence in both distributions (ranges $0.844$--$0.905$ and $0.828$--$0.921$, respectively).

\clearpage

\section{Extended Experiments for P3: Hierarchical Feature Recovery}
\label{sec:p3-extended}
We complement the aggregate hierarchy results in \Cref{fig:hierarchy-mean-f1} by controlling for sparsity, examining recovery of individual features and parent-child pairs, examining the prefix recovery examples, and discussing hierarchy extraction more broadly.

\subsection{Hierarchical recovery at fixed sparsity}
\label{sec:hierarchy-fixed-k}
\Cref{fig:hierarchy-mean-f1-fixed-k} repeats the mean-F1 analysis of \Cref{fig:hierarchy-mean-f1} with $k=128$ at widths 512, 4,096, 32,768, 65,536, and 131,072. We use the additional fixed-$k$ checkpoints at the first three widths and the existing $k=128$ checkpoints at the final two. The datasets, eligible categories, train/test splits, detector selection, and averaging over category--model pairs are the same as in the main analysis.

\begin{figure}[t]
    \centering
    \includegraphics[width=\linewidth]{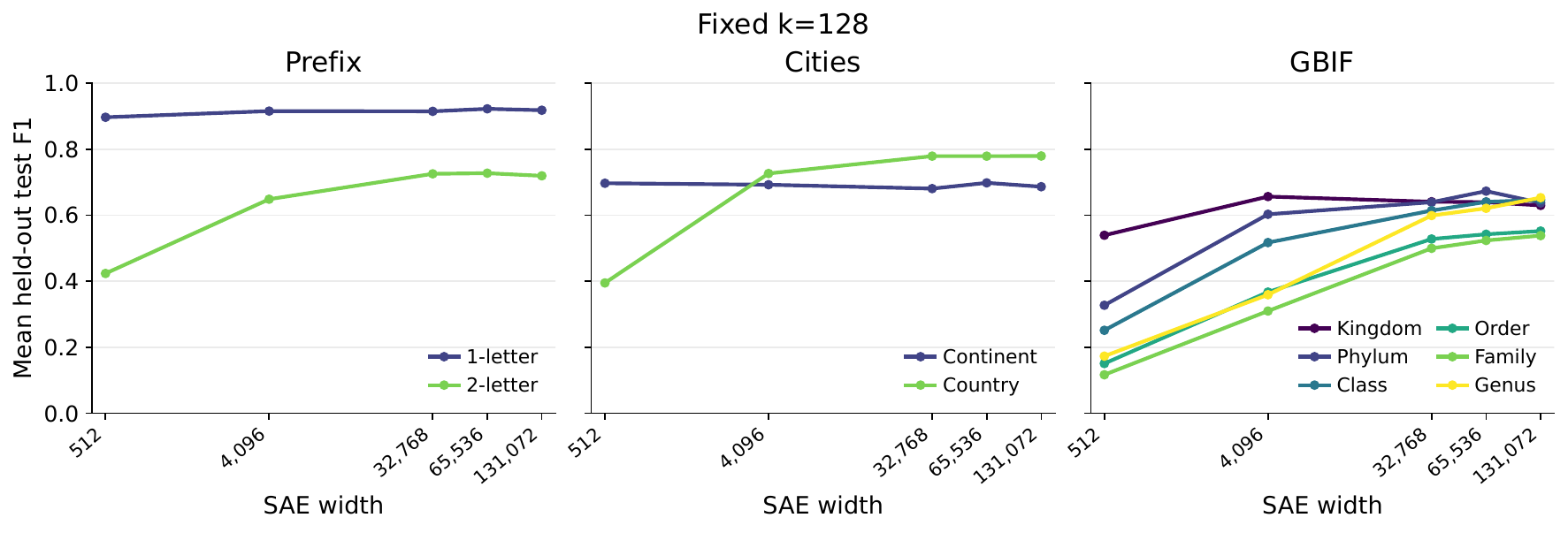}
    \caption{\textbf{Hierarchical feature recovery at fixed $k=128$.} Mean held-out test F1 for Prefix, Cities, and GBIF at five SAE widths, pooling Gemini and Nemotron. Each curve averages the same category--model pairs as the corresponding curve in \Cref{fig:hierarchy-mean-f1}; each category's feature and threshold are selected using training rows only.}
    \label{fig:hierarchy-mean-f1-fixed-k}
\end{figure}

\subsection{Counts and proportions of recovered categories and adjacent pairs}
\label{sec:hierarchy-recovery-counts}
To complement mean F1, \Cref{fig:individual-category-recovery,fig:adjacent-pair-recovery} count categories and parent--child pairs (in gbif, this means adjacent levels of the hierarchy) whose held-out test F1 is strictly greater than $0.6$.

\begin{figure}[p]
    \centering
    \includegraphics[width=\linewidth]{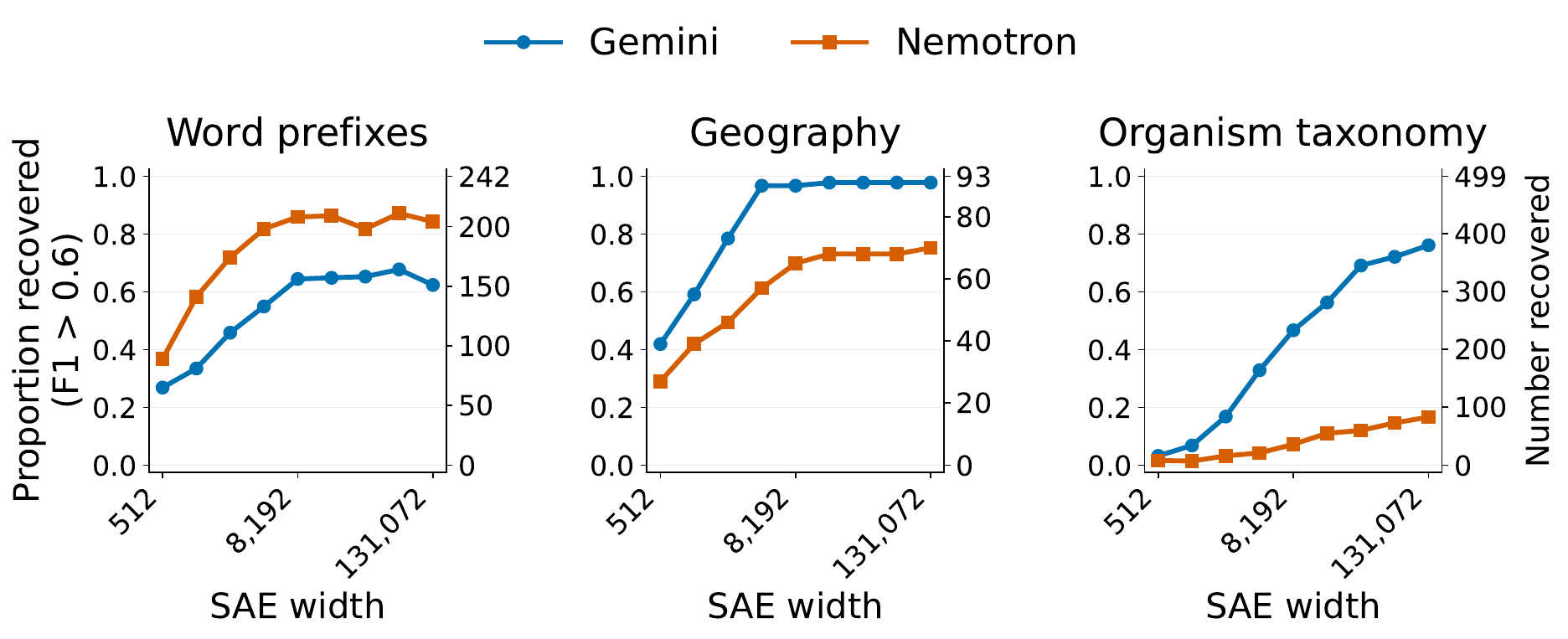}
    \caption{\textbf{Recovery of individual categories across SAE widths.} A category is recovered when its selection-chosen detector achieves held-out test F1 $>0.6$. Left axes show proportions; right axes show raw counts among 242 word prefixes, 93 geographic categories, and 499 organism categories. Geography uses different continent and country prompts.}
    \label{fig:individual-category-recovery}

    \vspace{1.5em}
    \includegraphics[width=\linewidth]{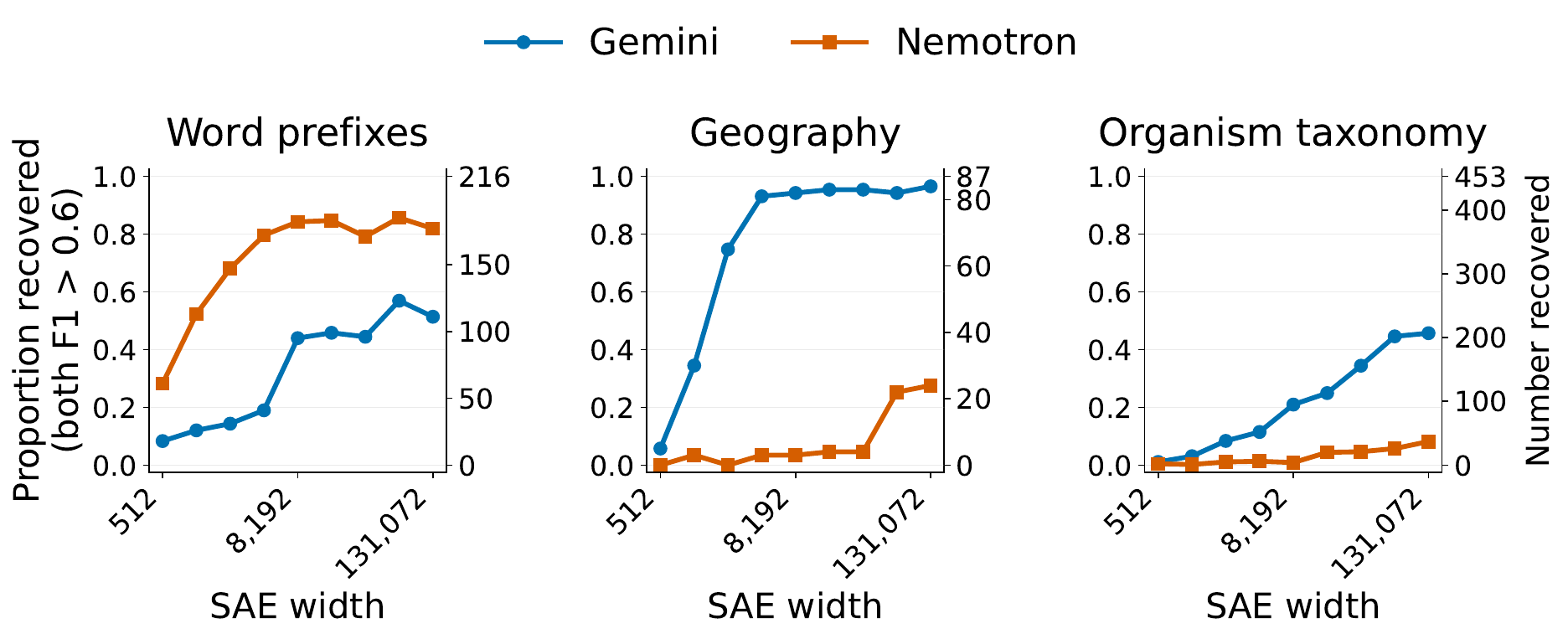}
    \caption{\textbf{Simultaneous recovery of adjacent parent--child pairs.} Both selection-chosen detectors must achieve held-out test F1 $>0.6$ and use distinct SAE coordinates within the pair. Left axes show proportions; right axes show raw counts among 216 prefix pairs, 87 continent--country pairs, and 453 adjacent-rank organism pairs.}
    \label{fig:adjacent-pair-recovery}
\end{figure}

\subsection{Recovery of letter-prefix hierarchies}
\label{sec:a-prefix-example}
The upper panel of \Cref{fig:a-prefix} shows Gemini test F1 for the \texttt{a} parent and its eligible two-letter children, selecting a separate detector for each. The lower recall heatmap instead fixes the selection-chosen \texttt{a} detector at each width and measures its recall within each prefix group. That descriptive heatmap uses the full 200,000-word collection, including selection and test rows, rather than only the held-out split. \Cref{fig:gemini-prefix-f1-grid} extends the upper F1 panel to the 25 one-letter prefixes with eligible two-letter children, excluding \texttt{x} (since it does not have any qualifying two-letter children), using the same saved probes, nine widths, and category eligibility rules. These curves report continuous held-out test F1 without filtering by F1. We see more instances of mesoscale splitting (a U-shaped parent F1 curve). \Cref{fig:gemini-prefix-f1-grid-fixed-k} repeats this grid at fixed $k=128$ for widths 512, 4K, 32K, 65K, and 131K, using the checkpoints described in \Cref{sec:hierarchy-fixed-k}.

\begin{figure}[p]
    \centering
    \includegraphics[width=\linewidth,height=0.82\textheight,keepaspectratio]{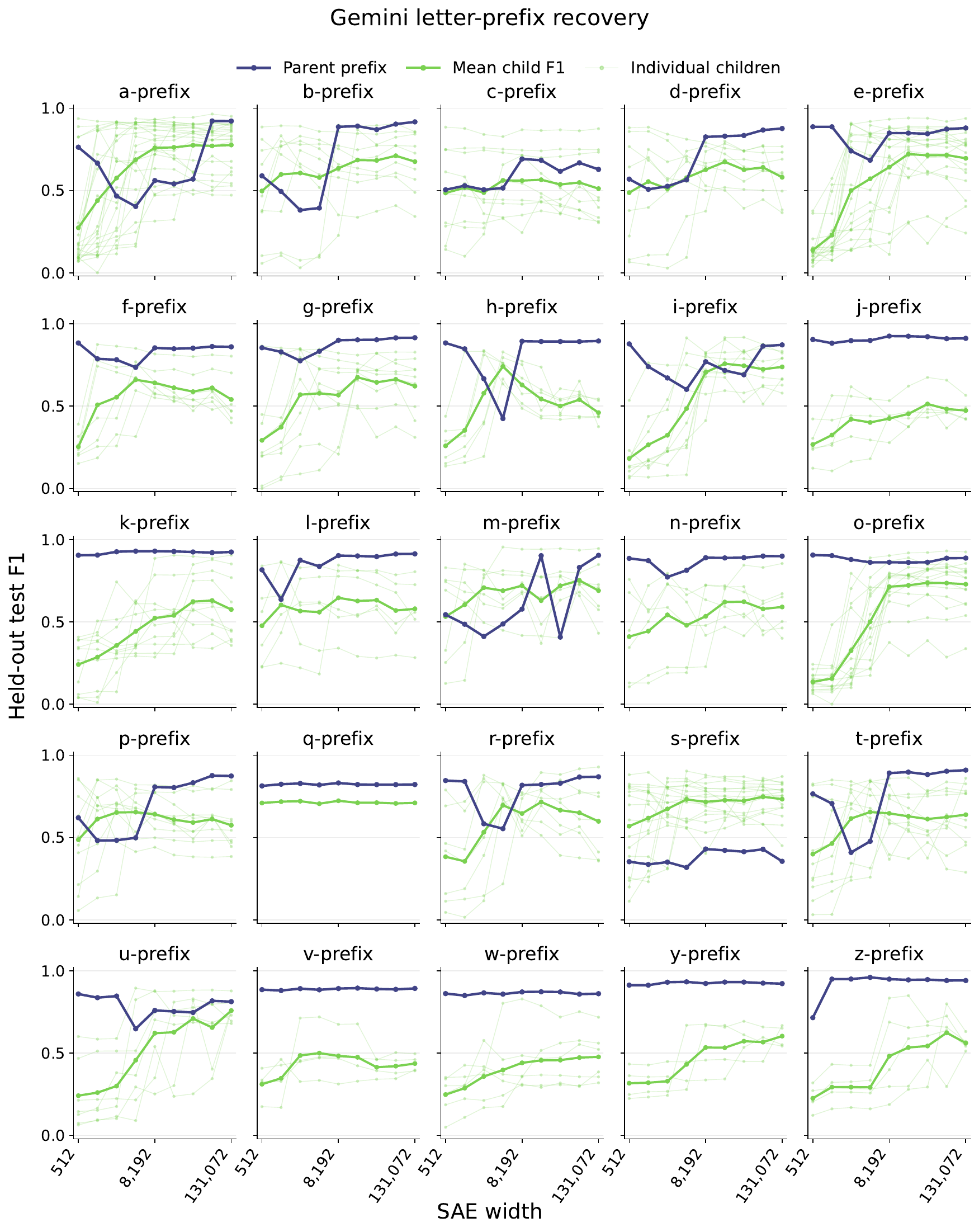}
    \caption{\textbf{Gemini F1 recovery for 25 letter prefixes.} Each panel is the analog of the upper panel of \Cref{fig:a-prefix} for one prefix from \texttt{a} to \texttt{z}, excluding \texttt{x}, across the original nine SAE widths and sparsity schedule. Blue curves show parent F1; bold green curves show the unweighted mean F1 of eligible two-letter children; faint green curves show individual children.}
    \label{fig:gemini-prefix-f1-grid}
\end{figure}

\begin{figure}[p]
    \centering
    \includegraphics[width=\linewidth,height=0.82\textheight,keepaspectratio]{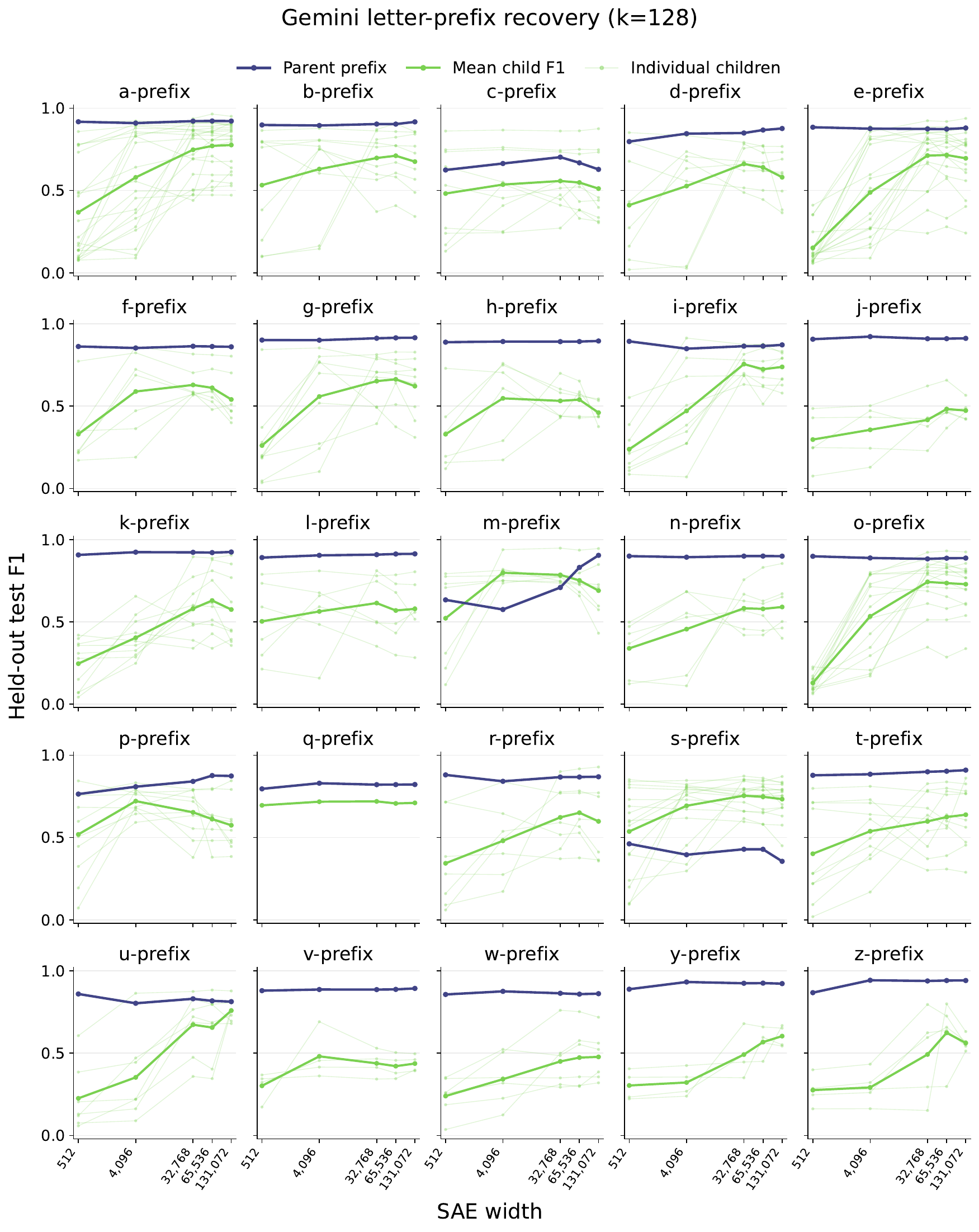}
    \caption{\textbf{Gemini F1 recovery for 25 letter prefixes at fixed $k=128$.} The analog of \Cref{fig:gemini-prefix-f1-grid} at five SAE widths: 512, 4,096, 32,768, 65,536, and 131,072. Panels cover \texttt{a}--\texttt{z}, excluding \texttt{x}. Blue curves show parent F1; bold green curves show the unweighted mean F1 of eligible two-letter children; faint green curves show individual children. Each category's detector is selected using selection rows and evaluated on held-out test rows, with no F1 filtering.}
    \label{fig:gemini-prefix-f1-grid-fixed-k}
\end{figure}
\clearpage

\subsection{Empirical extraction of feature hierarchies}
\label{sec:empirical-hierarchy-extraction}

The food and music hierarchies in \Cref{fig:learned-hierarchies} are curated subsets of a graph extracted from activation co-occurrences in the Gemini SAE with $131{,}072$ features and $k=128$. We use cached post-TopK activations on $N=89{,}827{,}558$ corpus rows. A feature is active when its activation is strictly greater than $0.05$; this threshold is the same for candidate parents and children. For a candidate parent $P$ and child $C$, let $n_P$ and $n_C$ denote their activation counts and $n_{PC}$ their co-activation count. We estimate forward and reverse containment as
\begin{equation}
    \widehat p_{P\mid C}=\frac{n_{PC}}{n_C},
    \qquad
    \widehat p_{C\mid P}=\frac{n_{PC}}{n_P}.
\end{equation}
High forward containment indicates that the parent usually activates when the child does; low reverse containment indicates that the parent also covers examples outside the child.

\paragraph{Candidate edges and directionality.}
We retain a candidate edge $P\to C$ only if $n_C\geq100$, $n_{PC}\geq100$, $n_P>n_C$, and
\begin{equation}
    L_{P\mid C}\geq0.70,
    \qquad \widehat p_{C\mid P}\leq0.50,
    \qquad \frac{\widehat p_{P\mid C}}{n_P/N}\geq5.
\end{equation}
The last condition is a lift threshold: containment must be at least five times the parent's marginal activation rate. The reverse-containment and support conditions orient edges from broader to narrower activation sets; the strict support ordering also rules out directed cycles.

Here $L_{P\mid C}$ is the lower endpoint of the two-sided $95\%$ Wilson interval for a binomial proportion. Writing $\widehat p=\widehat p_{P\mid C}$, $n=n_C$, and $z=1.96$, it is
\begin{equation}
    L_{P\mid C}=
    \frac{\widehat p+z^2/(2n)-z\sqrt{\widehat p(1-\widehat p)/n+z^2/(4n^2)}}{1+z^2/n}.
\end{equation}
This penalizes estimates based on fewer observations: for example, $90/100$ and $900/1000$ both give $90\%$ containment, but their lower bounds are approximately $82.6\%$ and $88.0\%$. We use the bound as a screening heuristic, since corpus rows can be duplicated or dependent and the interval does not adjust for testing many feature pairs.

\paragraph{Removing ancestor shortcuts.}
After collecting all qualifying edges, we remove $P\to C$ whenever a feature $M$ has both qualifying edges $P\to M$ and $M\to C$. Thus, Food $\to$ Cheddar is removed if Food $\to$ Cheese and Cheese $\to$ Cheddar both qualify. Multiple parents can remain, and a broader ancestor edge can remain if an intermediate edge fails the thresholds. The procedure yields $154{,}316$ qualifying ancestor pairs and $114{,}762$ retained parent--child pairs.

\paragraph{Selection and interpretation of the figures.}
Each hierarchy in \Cref{fig:learned-hierarchies} displays 33 selected features and 32 retained edges. The compact main-text table in \Cref{fig:learned-hierarchies-compact} represents all 66 features and 64 parent--child relationships in three hierarchy-level columns, omitting only the feature IDs. We inspected cached activating texts to choose interpretable branches and assign descriptive labels; \Cref{tab:food-hierarchy,tab:music-hierarchy} provide example texts. All edge statistics use the full corpus counts, rather than the sampled texts. The diagrams illustrate selected activation-containment relationships, not an exhaustive taxonomy or disjoint partition, and the labels are interpretations rather than independent ground truth.

\begin{figure}[htbp]
    \centering
    \includegraphics[width=0.95\linewidth]{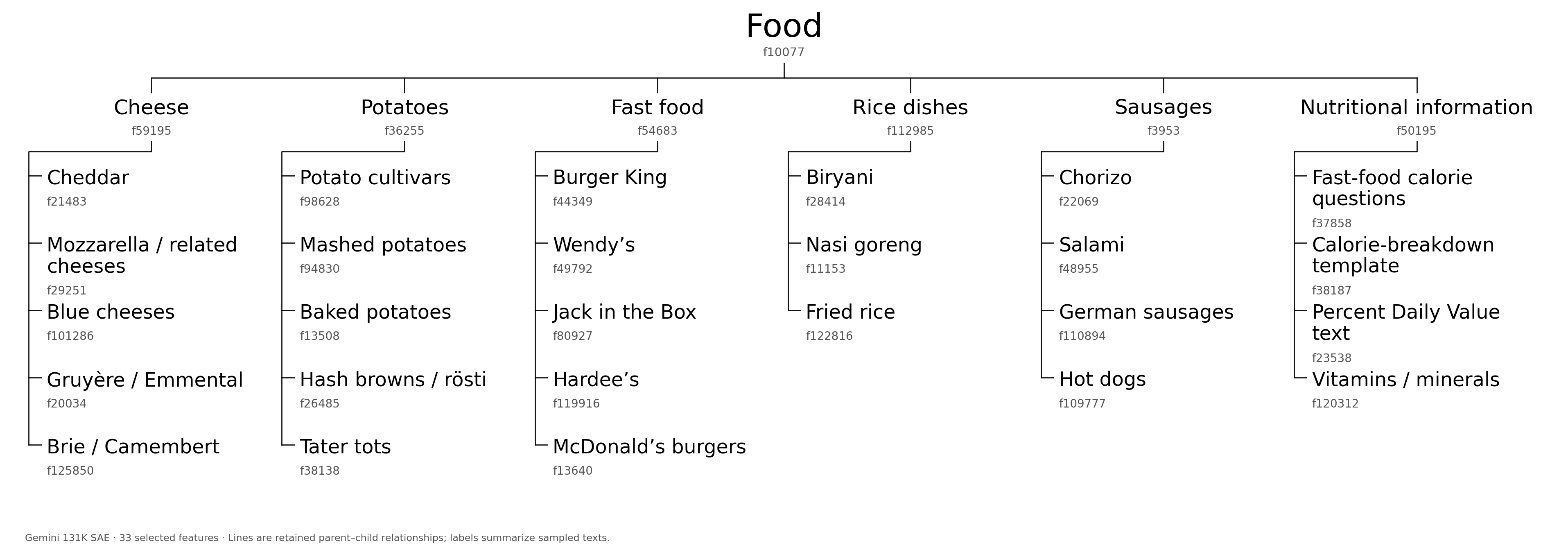}\\
    \includegraphics[width=0.95\linewidth]{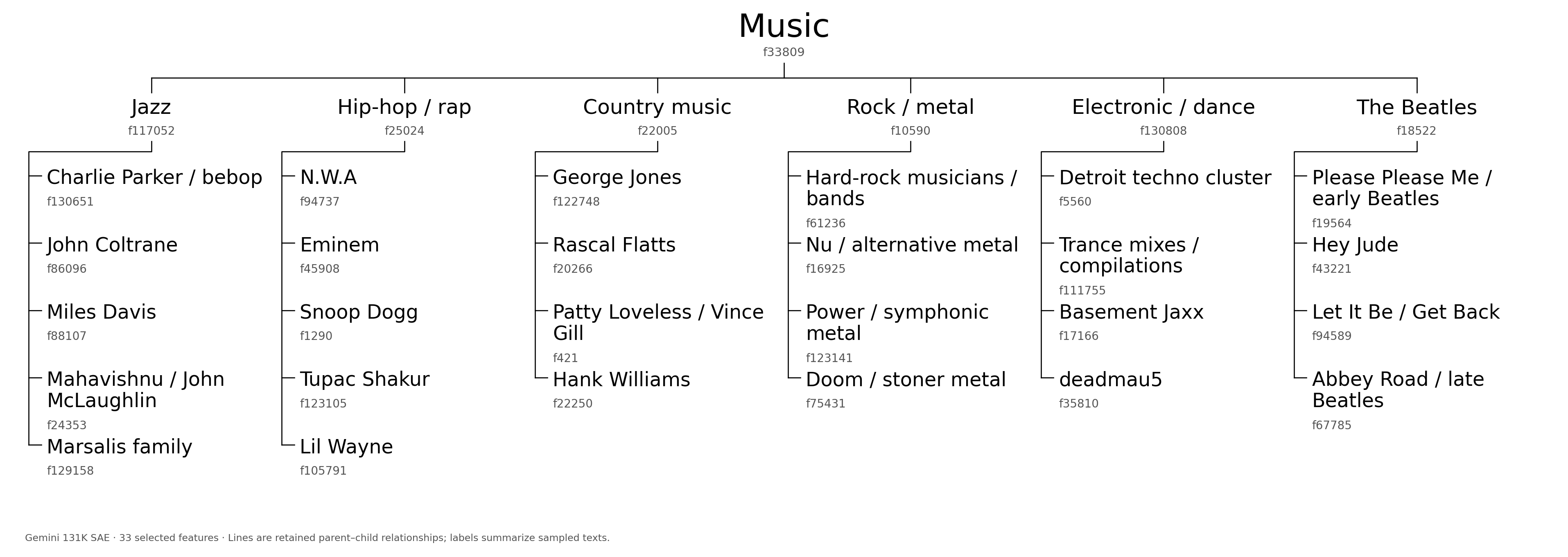}

    \caption{Learned hierarchies for Music and Food. The labels here are automatically generated using Codex; activating examples for each feature in the hierarchy are given in \Cref{tab:food-hierarchy,tab:music-hierarchy}.}
    \label{fig:learned-hierarchies}
\end{figure}

\clearpage
\begin{table}[p]
\centering
\caption{Food hierarchy with five unique activating examples per feature.}
\label{tab:food-hierarchy}
\includegraphics[width=\linewidth]{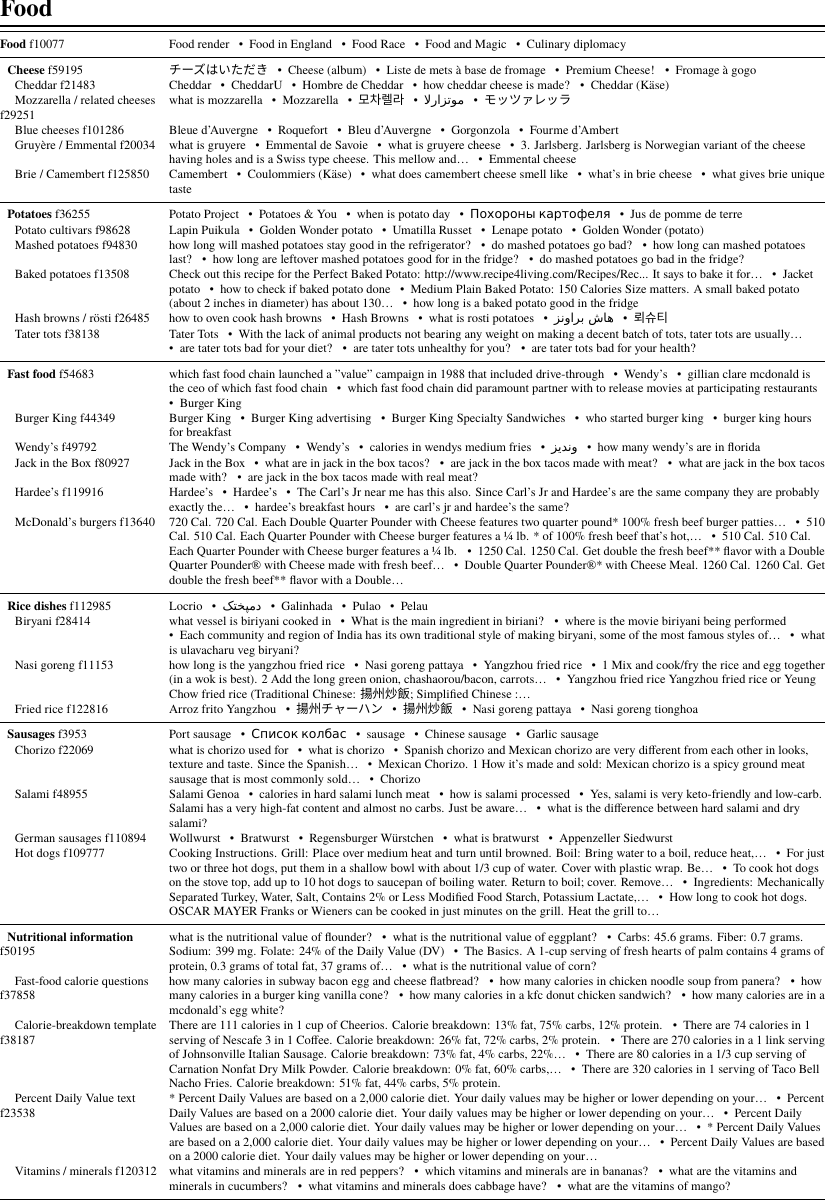}
\end{table}

\clearpage
\begin{table}[p]
\centering
\caption{Music hierarchy with five unique activating examples per feature.}
\label{tab:music-hierarchy}
\includegraphics[width=\linewidth]{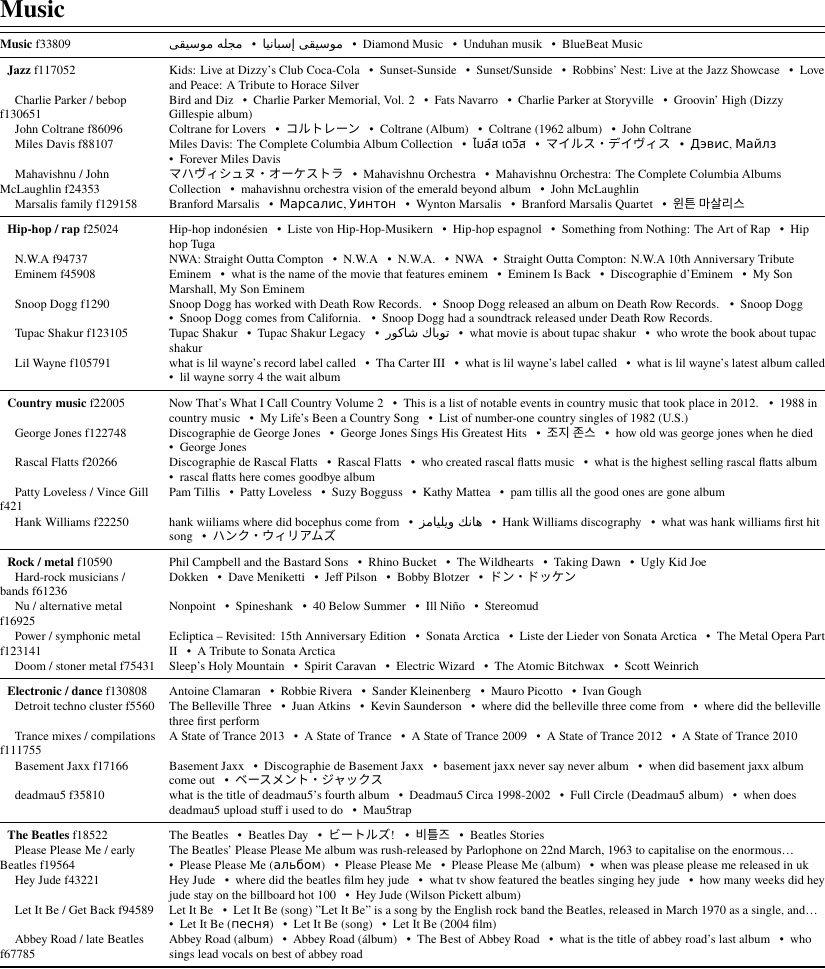}
\end{table}

\clearpage

\section{Simulation experiments}\label{sec:recovery-simulations}
\subsection{Recovery under width and sparsity variation}\label{sec:simulation-width-sparsity}
We study recovery in finite synthetic distributions with known ground-truth atoms. These experiments assess whether trained TopK SAEs exhibit prefix recovery and how this behavior depends on the frequency distribution and the SAE sparsity. Figures~\ref{fig:simulation-alpha-prefix} and~\ref{fig:simulation-k-prefix} report approximate recovered prefix lengths, expressed in atoms and divided by SAE width.

\paragraph{Data-generating distributions.}
Inputs lie in $\RR^{512}$ and are nonnegative linear combinations of exactly $K=8$ atoms. We draw each ground-truth direction independently from a standard Gaussian and normalize it to unit length. In the flat setting, there are $M=4096$ atoms. An eight-element support $S$ is sampled according to
\[
    \Pr(S) = \frac{\prod_{i\in S}i^{-\alpha}}{\sum_{T:\,|T|=8}\prod_{j\in T}j^{-\alpha}},
    \qquad |S|=8,
\]
and its nonzero coefficients are independent $\mathrm{Uniform}(0,1)$ draws. This is a product-weight distribution over fixed-size subsets; the power law specifies the support weights, rather than the marginal activation probabilities directly.
In the hierarchical setting, $M=3072$ atoms form 1024 families, each containing a parent and two children. We sample four families using the same product-weight law over family ranks, then activate each selected parent and one of its two children, chosen uniformly. The eight active coefficients are again independent $\mathrm{Uniform}(0,1)$ draws. Thus a child can only occur with its parent, while either child can accompany the parent.

\paragraph{Training and sweeps.}
We train nonnegative TopK SAEs at widths $m\in\{128,256,512,1024,2048,M\}$. Generating sparsity $K=8$ is fixed throughout; $k$ denotes the SAE's activation budget. The frequency sweep uses $\alpha\in\{0.8,1.2,1.6,2.0,2.4\}$ at $k=8$. The sparsity sweep uses $k\in\{1,2,4,6,8,10,12,16,32\}$ at $\alpha=1.6$. \Cref{thm:recovery-principle} does not give guarantees for $\alpha=0.8$. We use three paired random seeds per setting, giving 468 distinct training trajectories after deduplicating the shared $\alpha=1.6,k=8$ setting.

Each trajectory trains on 1,024M examples using Adam with learning rate $10^{-3}$, $\beta_1=0.9$, $\beta_2=0.99$, and batch size 1024, with no learning-rate schedule. Decoder directions have unit norm and the decoder bias is fixed to zero. An auxiliary reconstruction loss with coefficient $0.25$ targets up to 64 neurons inactive for 256 steps.

\paragraph{Recovery criteria.}
Geometric recovery requires absolute cosine similarity at least $0.9$ with a learned decoder direction. For hierarchical data, child comparisons project both directions onto the orthogonal complement of the true parent; recovery of a complete family requires three distinct decoder directions meeting the parent and two projected-child criteria. Activation recovery uses 200,000 independent validation examples to select coordinate thresholds and a maximum-cardinality injective assignment among atom--coordinate pairs with validation F1 at least $0.8$. A matched atom is recovered if its F1 is also at least $0.8$ on a separate set of 400,000 test examples. An activation-recovered family must recover all three members.

For either criterion, let $q_i$ indicate recovery of atom $i$ in the flat setting or complete family $i$ in the hierarchical setting, ordered by generating weight. We compute
\[
    r=\max\left(\{0\}\cup\left\{s:\sum_{i=1}^{s}q_i\geq 0.9s\right\}\right).
\]

\begin{figure}[p]
    \centering
    \includegraphics[width=\linewidth]{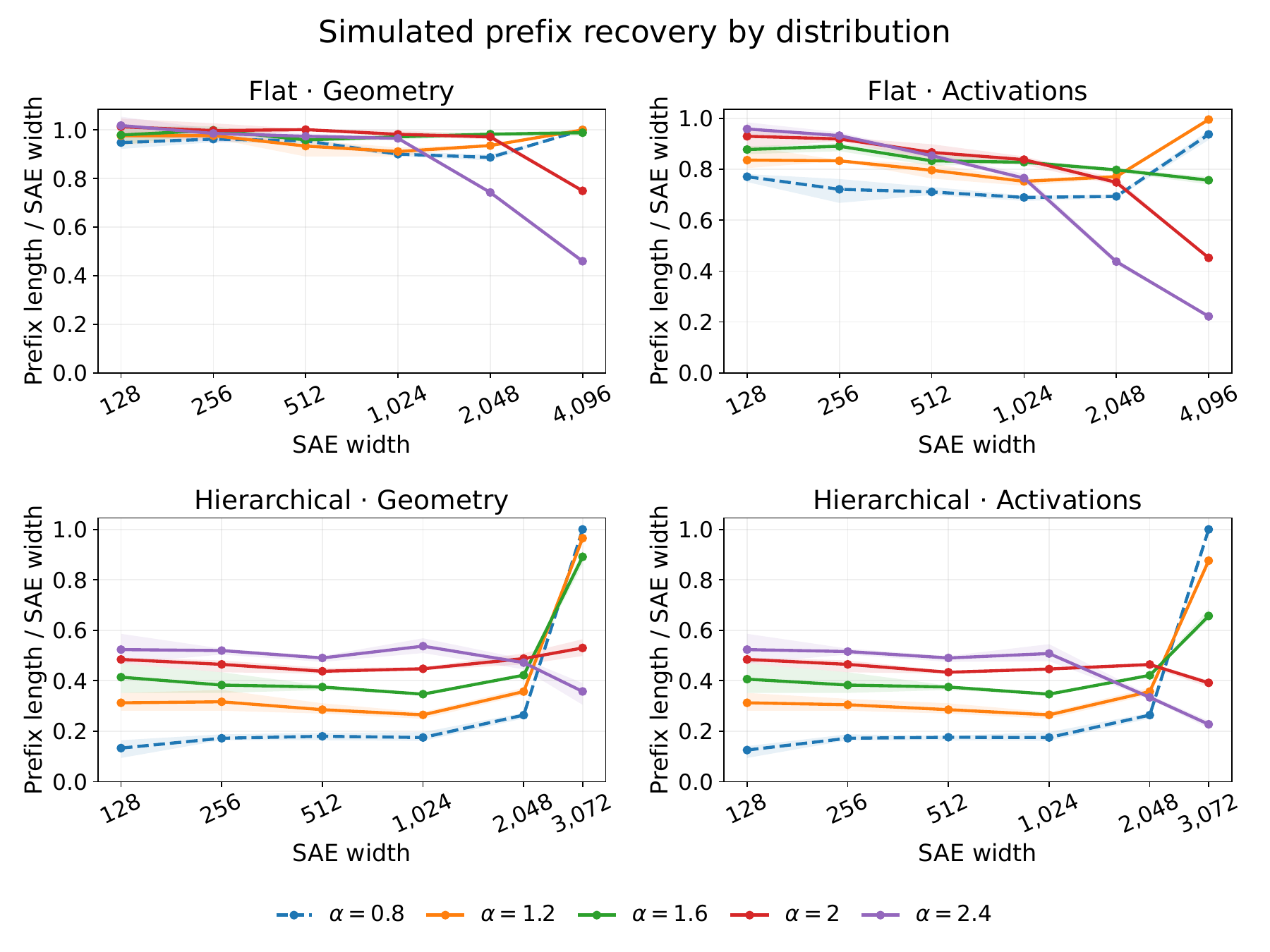}
    \caption{\textbf{Width-normalized prefix recovery across frequency distributions.} Generating sparsity $K=8$ and SAE sparsity $k=8$ are fixed; each curve corresponds to a different $\alpha$. All models train on 1,024M examples. Rows show flat and hierarchical data; columns show geometric and activation recovery. The vertical axis is approximate prefix length in atoms divided by SAE width (three atoms per complete family for hierarchical data), with a 10\% miss allowance. Lines show three-seed means and bands show seed ranges.}
    \label{fig:simulation-alpha-prefix}
\end{figure}

\begin{figure}[p]
    \centering
    \includegraphics[width=\linewidth]{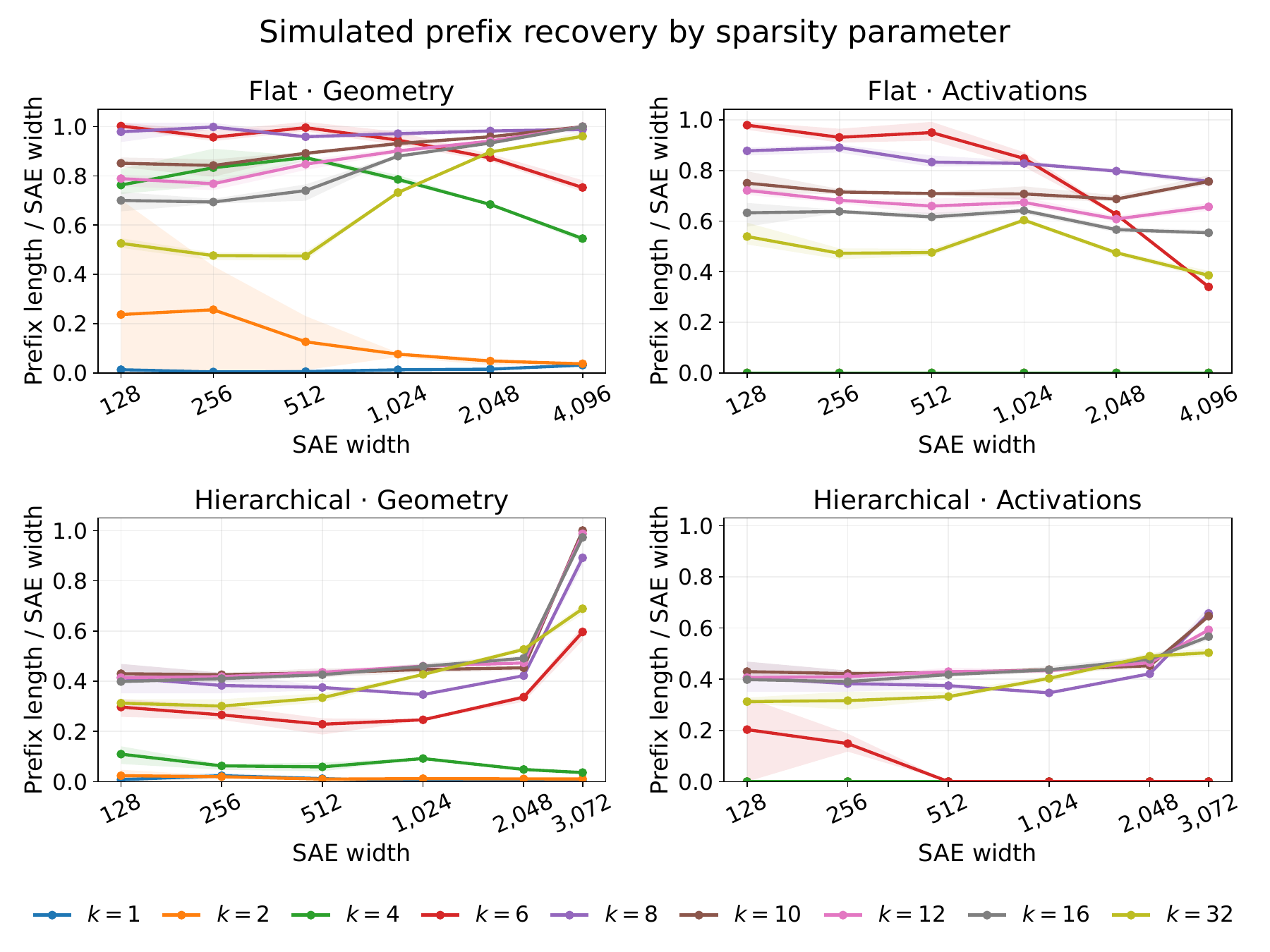}
    \caption{\textbf{Width-normalized prefix recovery across SAE sparsities.} Generating sparsity $K=8$ and frequency exponent $\alpha=1.6$ are fixed; curves vary SAE sparsity $k$. All models train on 1,024M examples. Panels, normalization, and the 10\% prefix miss allowance follow Figure~\ref{fig:simulation-alpha-prefix}. Lines show three-seed means and bands show seed ranges.}
    \label{fig:simulation-k-prefix}
\end{figure}
\clearpage

\subsection{Persistent features across simulated SAE widths}
\label{sec:simulation-persistence}
We apply the persistence metric in \Cref{sec:matching-computation}. For each source width $m$, we independently compute a maximum matching into each larger trained SAE and intersect the matched source-feature sets, holding the data distribution, SAE sparsity, and paired seed fixed. Matches require signed decoder cosine similarity or signed Pearson correlation of continuous post-TopK activations at least $t=0.8$. The largest width ($4096$ for flat data, $3072$ for hierarchical data) is a matching target only. We report the persistent-feature proportion, dividing the size of the intersection by all $m$ source features, including inactive features.

We evaluate each group of six widths on the same 400,000 held-out test examples and compute all pairwise similarities without sketching or candidate truncation. Constant activation columns cannot form Pearson edges. The procedure in \Cref{sec:matching-computation} certifies each pairwise maximum on these empirical threshold graphs and verifies their matched source-set intersection. This measures persistence between learned SAEs.

\Cref{fig:simulation-persistence-alpha-proportion-t08} varies the frequency exponent; \Cref{fig:simulation-persistence-k-proportion-t08} varies SAE sparsity. At $\alpha=1.6,k=8$, mean activation persistence across source widths ranges from $93.6$--$96.1\%$ for flat data and $91.9$--$99.1\%$ for hierarchical data. Hierarchical decoder persistence rises from $61.5\%$ at width 128 to $91.8\%$ at width 2048. Severe sparsity underspecification particularly reduces activation persistence, showing that persistent decoder directions need not imply persistent activation patterns.

\begin{figure}[htbp]
    \centering
    \includegraphics[width=\linewidth]{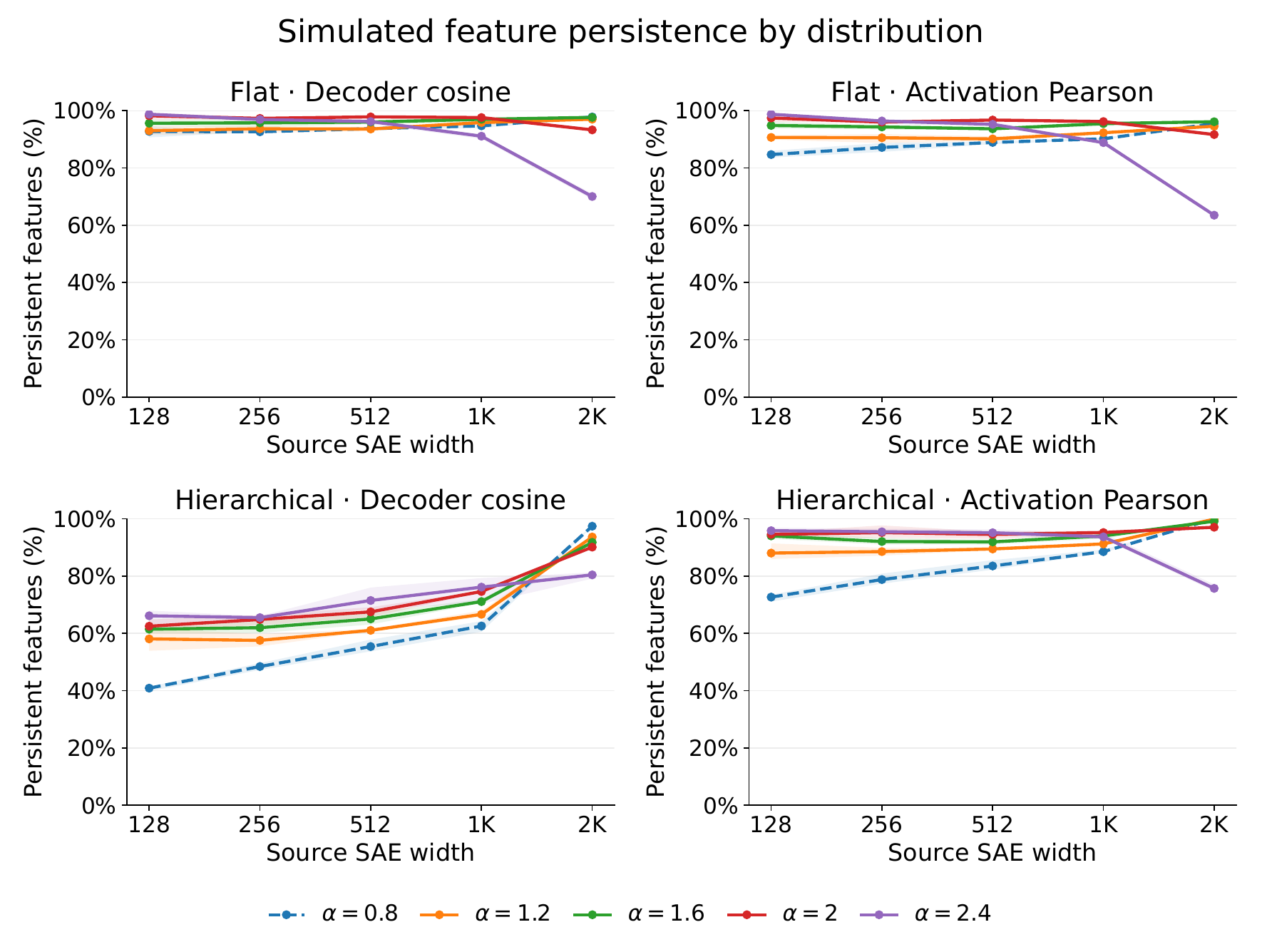}
    \caption{\textbf{Persistent-feature proportions across frequency distributions at $t=0.8$.} The intersection of source sets matched by independent maximum matchings to every larger trained width, divided by source SAE width. Generating sparsity $K=8$ and SAE sparsity $k=8$ are fixed; colors vary $\alpha$. Rows show flat and hierarchical data; columns show signed decoder cosine and activation Pearson similarity. All models train on 1,024M examples. Lines show three-seed means and bands show seed ranges. The dashed $\alpha=0.8$ curve is the finite-distribution control outside the theorem's guarantees.}
    \label{fig:simulation-persistence-alpha-proportion-t08}
\end{figure}

\begin{figure}[p]
    \centering
    \includegraphics[width=\linewidth]{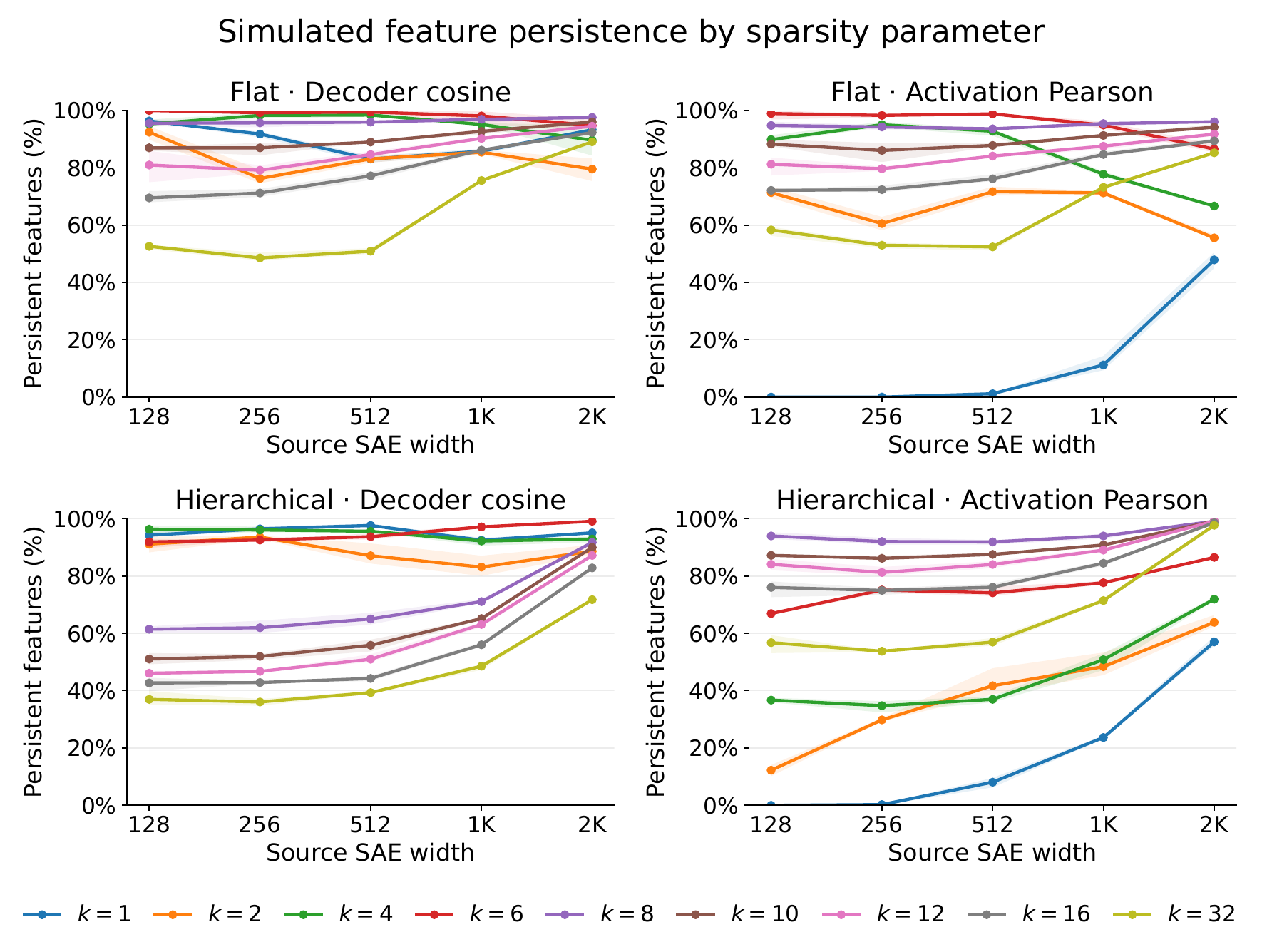}
    \caption{\textbf{Persistent-feature proportions across SAE sparsities at $t=0.8$.} The intersection of source sets matched by independent maximum matchings to every larger trained width, divided by source SAE width. Generating sparsity $K=8$ and frequency exponent $\alpha=1.6$ are fixed; colors vary SAE sparsity $k$. Rows show flat and hierarchical data; columns show signed decoder cosine and activation Pearson similarity. All models train on 1,024M examples. Lines show three-seed means and bands show seed ranges.}
    \label{fig:simulation-persistence-k-proportion-t08}
\end{figure}
\clearpage

\clearpage

\section{A Molecular Model of SAE Scaling}\label{sec:molecules}

\newcommand{\zcell}[2]{%

  \multicolumn{1}{>{\columncolor{orange!#1}}c}{$#2$}%

}

\begin{table}
\centering
\setlength{\tabcolsep}{5pt}

\setlength{\minrowclearance}{0pt}

\setlength{\aboverulesep}{0pt}

\setlength{\belowrulesep}{0pt}

\footnotesize

\begin{tabular}{
    >{\centering\arraybackslash}m{0.7cm}
    >{\centering\arraybackslash}m{0.4cm}
    >{\centering\arraybackslash}m{5.7cm}
    >{\centering\arraybackslash}m{1.0cm}
    >{\centering\arraybackslash}m{1.0cm}
    >{\centering\arraybackslash}m{1.2cm}
    >{\centering\arraybackslash}m{1.2cm}
}
\toprule

dict. size
&
$i$
&
$\hat a_i$ with molecule illustration
&
$\hat z_i(\mathrm{cat})$
&
$\hat z_i(\mathrm{dog})$
&
$\hat{A}\hat{z}(\mathrm{cat})$
&
$\hat{A}\hat{z}(\mathrm{dog})$
\\

\midrule

1
&
1
&
\(
\frac{2}{\sqrt6}a_{\mathrm{pet}}
+
\frac{1}{\sqrt6}a_{\mathrm{cat}}
+
\frac{1}{\sqrt6}a_{\mathrm{dog}}
\)
\quad
\petcatdogmol
&
\zcell{70}{\displaystyle\frac{3}{\sqrt6}}
&
\zcell{70}{\displaystyle\frac{3}{\sqrt6}}
&
\multicolumn{2}{c}{\(
\displaystyle
\frac{2a_{\mathrm{pet}}
+
a_{\mathrm{cat}}
+
a_{\mathrm{dog}}}{2}
\)}
\\

\midrule

\multirow{2}{*}[-1.5ex]{2}
&
1
&
\(
\frac{1}{\sqrt2}a_{\mathrm{pet}}
+
\frac{1}{\sqrt2}a_{\mathrm{cat}}
\)
\quad
\petcatmol
&
\zcell{100}{\sqrt2}
&
\zcell{0}{0}
&
\multirow{2}{*}[-1.5ex]{$a_{\mathrm{pet}} + a_{\mathrm{cat}}$}
&
\multirow{2}{*}[-1.5ex]{$a_{\mathrm{pet}} + a_{\mathrm{dog}}$}
\\

\cmidrule(lr){2-5}

&
2
&
\(
\frac{1}{\sqrt2}a_{\mathrm{pet}}
+
\frac{1}{\sqrt2}a_{\mathrm{dog}}
\)
\quad
\petdogmol
&
\zcell{0}{0}
&
\zcell{100}{\sqrt2}
&
&
\\

\midrule

\multirow{3}{*}[-3ex]{3}
&
1
&
\(a_{\mathrm{pet}}\)
\quad
\petmol
&
\zcell{50}{1}
&
\zcell{50}{1}
&
\multirow{3}{*}[-3ex]{$a_{\mathrm{pet}} + a_{\mathrm{cat}}$}
&
\multirow{3}{*}[-3ex]{$a_{\mathrm{pet}} + a_{\mathrm{dog}}$}
\\

\cmidrule(lr){2-5}

&
2
&
\(a_{\mathrm{cat}}\)
\quad
\catmol
&
\zcell{50}{1}
&
\zcell{0}{0}
&
&
\\

\cmidrule(lr){2-5}

&
3
&
\(a_{\mathrm{dog}}\)
\quad
\dogmol
&
\zcell{0}{0}
&
\zcell{50}{1}
&
&
\\

\bottomrule
\end{tabular}
\caption{Learned directions illustrate molecular splitting. Corresponding learned features illustrate mesoscale feature splitting and absorption: at size 2 only, there is no learned feature that is active on both cat and dog.}
\label{tab:pet-feature-splitting}
\end{table}

\subsection{Computing the empirical molecules}\label{sec:empirical-molecule-methods}

The empirical molecules in \Cref{tab:a-prefix-empirical-feature-splitting,tab:target-3290-transform-mix-selected-with-molecule} approximate smaller-SAE decoder directions using decoder directions from the Gemini SAE of width 131,072 and sparsity $k=128$ as proxy atoms. These are decompositions of learned directions in embedding space; the activation values and recalls are measured separately by evaluating the corresponding SAEs. The 131K features serve as empirical proxies, rather than certified atomic directions. Below, $d$ denotes a target decoder row and $b_j$ a proxy-atom decoder row, both at their stored checkpoint scale.

\paragraph{Prefix molecules: basis and least-squares fit.}
For \Cref{tab:a-prefix-empirical-feature-splitting}, we use the word-prefix data, saved selection/test split, and single-feature probe procedure in \Cref{sec:p3-hierarchy-protocol}. For each of the 23 available one- and two-letter prefixes beginning with \texttt{a}, we take the 131K feature and positive activation threshold maximizing selection F1, without filtering by test F1. Their decoder rows form a matrix $B=[b_1,\ldots,b_{23}]$. There are 22 distinct directions, since \texttt{ac} and \texttt{ak} select the same feature; the regression retains both columns. For each selected smaller-SAE decoder direction, we compute
\[
    c^\star=\operatorname*{arg\,min}_{c\in\mathbb{R}^{23}}\|d-Bc\|_2^2,
\]
using the minimum-norm solution, computed in double precision with \texttt{numpy.linalg.lstsq} and \texttt{rcond=None}. There is no intercept, regularization, or nonnegativity constraint, and no additional normalization of the decoder rows. The fit uses all 23 basis columns, even though the table displays only the five prefixes \texttt{a}, \texttt{ah}, \texttt{ab}, \texttt{ap}, and \texttt{as}. Their proxy-atom feature IDs are 38618, 121657, 100965, 120576, and 32170, respectively.

At each displayed width (512/$k=32$, 2,048/$k=32$, and 131,072/$k=128$), we select the best probe for each of those five categories and deduplicate feature IDs. If a feature is selected for multiple categories, we use the threshold from the first category in the order \texttt{a}, \texttt{ah}, \texttt{ab}, \texttt{ap}, \texttt{as}. For a category $q$, its recall cell is
\[
    \frac{\#\{x\in T_q:\hat z_j(x)\geq\tau_j\}}{|T_q|},
\]
where $T_q$ is the set of held-out words beginning with $q$. The four displayed child categories contain 70, 339, 292, and 456 test words, respectively. We display only rows whose recall is at least $0.5$ on at least one category for which that feature was selected, using the same threshold across all cells in a row. This is a display filter using held-out recall; the probes and thresholds themselves are selected on the selection split.

The displayed sums and circles retain only coefficients at least $0.015$ among the five illustrated prefixes. Negative coefficients, smaller positive coefficients, and all other prefix components are omitted without refitting or renormalizing the retained coefficients. Circle radii are proportional to the cube root of the coefficient. Thus, these pictures summarize selected positive contributions rather than the complete signed least-squares reconstruction. At width 131,072, the illustrated directions are themselves basis columns, so their single-atom decompositions hold by construction; their held-out activation patterns are measured independently.

\paragraph{Polysemantic molecule: sparse decoder reconstruction.}
For feature 3290 in the Gemini width-4,096, $k=32$ SAE (\Cref{tab:target-3290-transform-mix-selected-with-molecule}), the candidate dictionary contains all 131,072 proxy-atom decoder rows. We use orthogonal matching pursuit (OMP). Starting with residual $r=d$, each step selects the unselected atom maximizing $|\langle r,b_j/\|b_j\|_2\rangle|$, refits unrestricted signed least-squares coefficients on the selected support, and updates the residual. Atom normalization is used for selection; reported coefficients are expressed against the stored decoder rows. We retain the first support along this greedy path satisfying
\[
    \left\|d-\sum_{j\in S}c_jb_j\right\|_2^2\leq 0.05.
\]
This is an absolute squared-error threshold on the stored target row, not a relative-error criterion or a guarantee of globally minimal support. The retained fit has 17 terms, squared error $0.049297$, and cosine similarity $0.975040$ between the target and reconstruction.

For readability, the diagram shows the five largest positive coefficients of this 17-term fit (also the first five atoms selected by OMP): $0.6763$ on feature 8574, $0.1955$ on 114331, $0.1665$ on 68236, $0.1615$ on 72509, and $0.0593$ on 97729. These coefficients come from the full fit, not a refit restricted to five atoms, and the reported reconstruction error applies to all 17 terms. Circle areas approximately encode these positive coefficients; the smallest circle (97729) is enlarged for legibility. The other twelve signed terms are omitted from the illustration.

\paragraph{Activation examples.}
The numerical cells in \Cref{tab:target-3290-transform-mix-selected-with-molecule,tab:target-3290-five-atom-binned-examples} are cached post-TopK SAE activations on the same corpus inputs, with absent features assigned zero. They are not inferred from the decoder-decomposition coefficients. The upper block of the main table is manually selected from activating texts to illustrate different senses of transformation, including examples containing and not containing the substring \texttt{transform}. The appendix gives twelve examples per proxy atom, curated across activation bins and semantic subtypes and sorted by decreasing activation. The lower block of the main table takes one example per atom from this appendix selection, with activation at least $0.09$ on that atom and below $0.05$ on each of the other four. These curated examples illustrate semantic selectivity rather than estimating its corpus frequency. Main-table shading uses a shared linear activation scale capped at $0.3$.

\paragraph{Decoder-similarity heatmaps.}
For \Cref{fig:a-prefix-probe-decoder-cosines}, we select the best probes for all 23 available a-prefix categories separately at widths 32,768/$k=64$ and 131,072/$k=128$, again without a test-F1 filter. After deduplicating feature IDs, each panel contains 22 directions. Entries are signed decoder cosines $\langle d_i,d_j\rangle/(\|d_i\|_2\|d_j\|_2)$, computed directly from the checkpoint rows, rather than from the molecule fits. The reported medians use the 231 unordered off-diagonal pairs in each panel.

\Cref{tab:target-3290-five-atom-binned-examples} gives twelve activating examples for each of the five 131K features in \Cref{tab:target-3290-transform-mix-selected-with-molecule}.

\begin{table}
\centering
\caption{Activating examples for the five 131K atoms. Each column contains twelve examples, sorted by decreasing activation. These appear more monosemantic than the feature in \Cref{tab:target-3290-transform-mix-selected-with-molecule}.}
\label{tab:target-3290-five-atom-binned-examples}
\includegraphics[width=\linewidth]{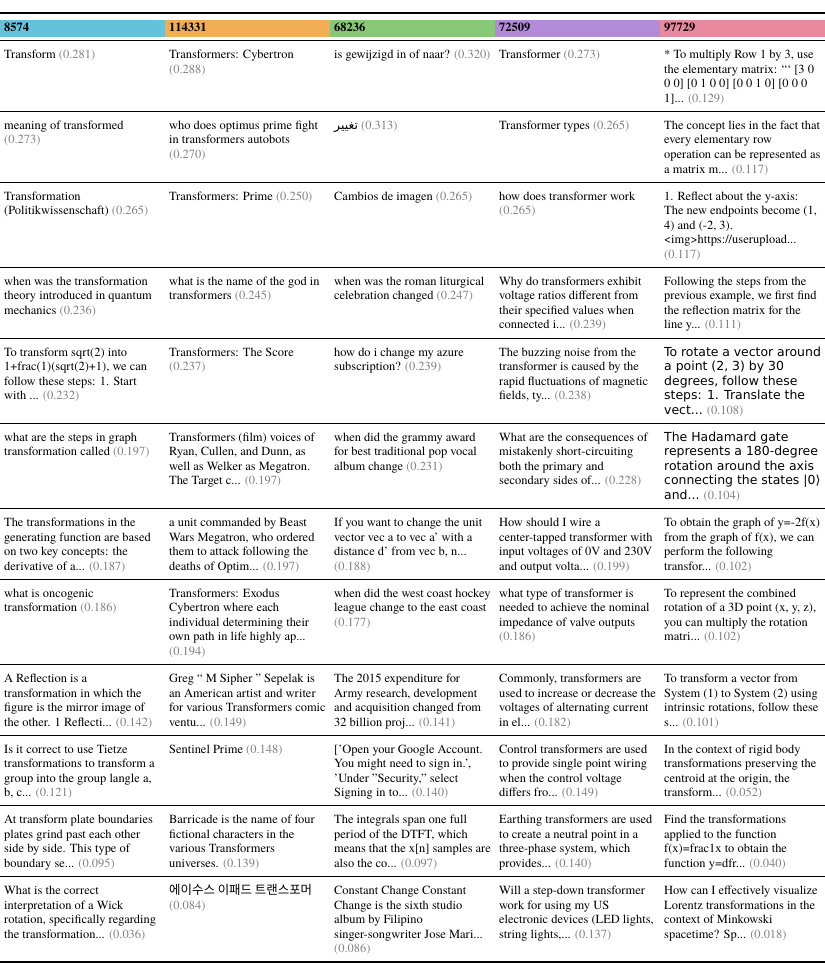}
\end{table}

\begin{figure}[t]
    \centering
    \includegraphics[width=\linewidth]{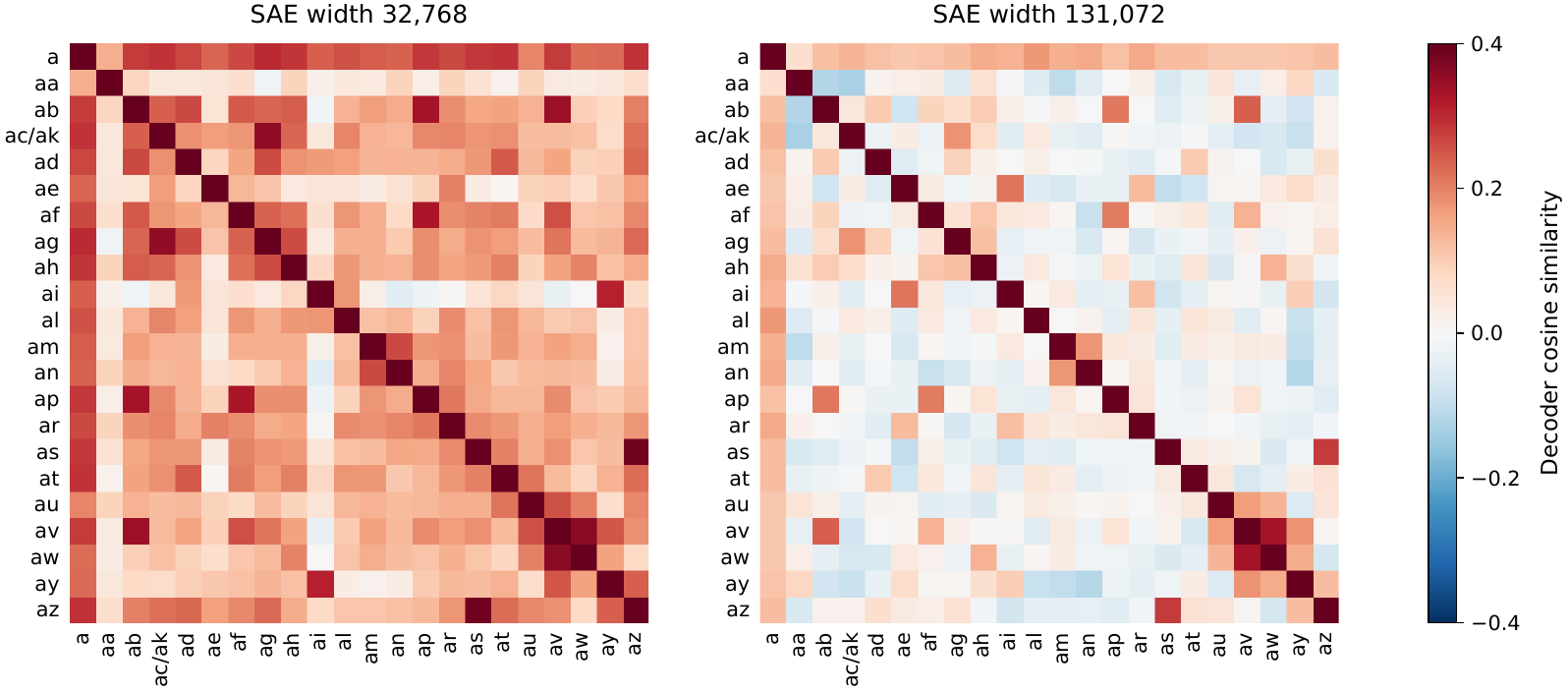}
    \caption{\textbf{Decoder direction similarities among a-prefix probes.} Signed pairwise cosine similarities for Gemini SAEs of widths 32,768 ($k=64$, left) and 131,072 ($k=128$, right). At each width, we take the best selected feature for each of the 23 available one- and two-letter a-prefix categories, without an F1 filter, and deduplicate feature IDs. This yields 22 unique decoder directions: \texttt{ac} and \texttt{ak} share a probe at both widths. Rows and columns use the same prefix order in both panels, with a common color scale from $-0.4$ to $0.4$. Diagonal self-similarities equal 1 and are shown at the maximum red. The median cosine over the 231 unique off-diagonal pairs is 0.1438 at width 32,768 and 0.0084 at width 131,072.}
    \label{fig:a-prefix-probe-decoder-cosines}
\end{figure}

\clearpage

\section{An Atomic View of Feature Manifolds}\label{sec:manifold}

A growing body of work has observed that features do  not necessarily exclusively have one-dimensional representations. For example, considering two-dimensional PCA projections of days of the year, one often recovers circular structure. The observation of such ``manifolds'' is sometimes used to argue against the linear representation hypothesis (which the atomic hypothesis strengthens). 

Here, we point out that the atomic hypothesis is not inconsistent with the observation of higher-dimensional non-linear structure. However, the atomic hypothesis can be viewed as implying that it should be possible to reconstruct this non-linear structure using atomic features. In particular, if we believe that the information represented by $f$ is captured entirely by $z$, we should see the same non-linear structure arise in $z$. Looking exclusively at $z$ disentangles the information stored in the features themselves from information potentially stored by the directions corresponding to each feature.

We consider subsets of SAE features $S\subseteq[131072]$ and consider PCA projections of $z_S(x)\in \RR^M$ for $x\in X^*$. We demonstrate that this procedure reveals geometric structure from atomic features. Notice this procedure only uses the learned features $z$ and not the learned directions $A$. Importantly, this disentangles geometric structure arising from the embedding space (i.e., if different atoms are represented along correlated directions in $\RR^d$) from geometric structure arising purely from correlation in feature activations in data.

We study two datasets: dates and colors. For each dataset, \Cref{fig:manifold} compares PCA of the raw embeddings with PCA of the continuous post-TopK activations restricted to a selected subset $S$ of the 131K Gemini SAE's coordinates. The feature subsets and color filtering were chosen through exploratory inspection.

\paragraph{Dates.}
The inputs are the 366 strings ``[Month] [Day]'' in calendar order, including ``February 29'', with no year in the strings. We identify twelve month-related features in the gemini 131K SAE.
We apply the full 131K encoder and TopK operation before extracting these twelve coordinates, assigning zero when a coordinate is not retained by TopK. PCA mean-centers the $366\times3072$ raw embedding matrix and the $366\times12$ activation matrix separately. The heatmap displays the same uncentered, continuous activation matrix in calendar order.

\paragraph{Colors.}
The original collection contains 865 color-name rows with hex-string embedding inputs. We deduplicate hex codes case-insensitively, retaining the first row for each of the 765 distinct colors. After scaling RGB channels to $[0,1]$, we retain colors with HSV saturation $S\geq0.60$, HSV value $V\geq0.50$, and HLS lightness $L\leq0.70$. This leaves 374 colors. We sort these by increasing hue, decreasing saturation, and then increasing value.

The eight selected 131K features were chosen from ranks 6, 7, 8, 10, 11, 12, 15, and 16, respectively, in the color collection's feature-prevalence ranking at activation threshold $0.05$. These corresponded to features that seemed color-related (rather than activating on hex patterns). That threshold is used only for selection: the plotted values are raw post-TopK activations. We refit three-component PCA separately to the $374\times3072$ raw embeddings and the $374\times8$ activation matrix using the full SVD solver, mean-centering without standardization, whitening, or additional normalization. These are fits on the retained colors, rather than subsets of PCA coordinates fitted to all 865 rows. The color heatmap uses the same raw activations, with rows ordered by activation-weighted circular mean position along the displayed hue sequence.

\paragraph{Text and artwork examples.}
\Cref{tab:gemini-color-examples} shows five texts and five WikiArt images for each of the eight color features: 40 texts and 40 images in total. Texts are ranked by raw post-TopK activation over all 89,827,558 Gemini training-corpus rows, removing exact duplicate full texts separately for each feature. Images are ranked independently by the same activation definition over all 116,261 rows of the cached WikiArt corpus. We take the five highest positive image activations per feature, with ties broken by decreasing image-vector index; no hex-color, style, or enrichment filter is applied to these images.

WikiArt images are outside the SAE training corpus. These selected examples illustrate transfer beyond hex codes to multilingual text and images; they are not a held-out population performance estimate.

\begin{figure}[htbp]
    \centering
    \includegraphics[width=\linewidth]{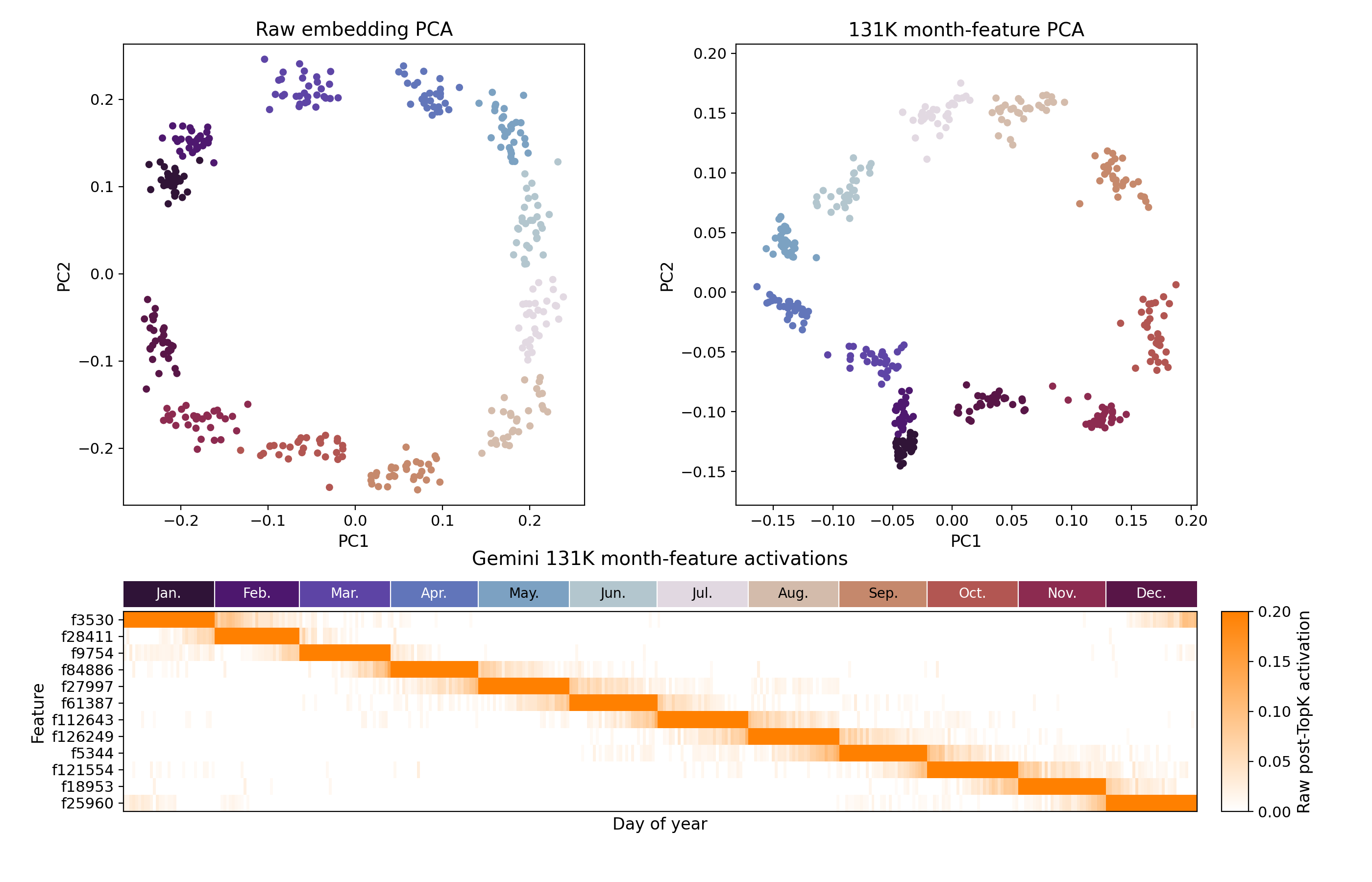}\\
    \includegraphics[width=\linewidth]{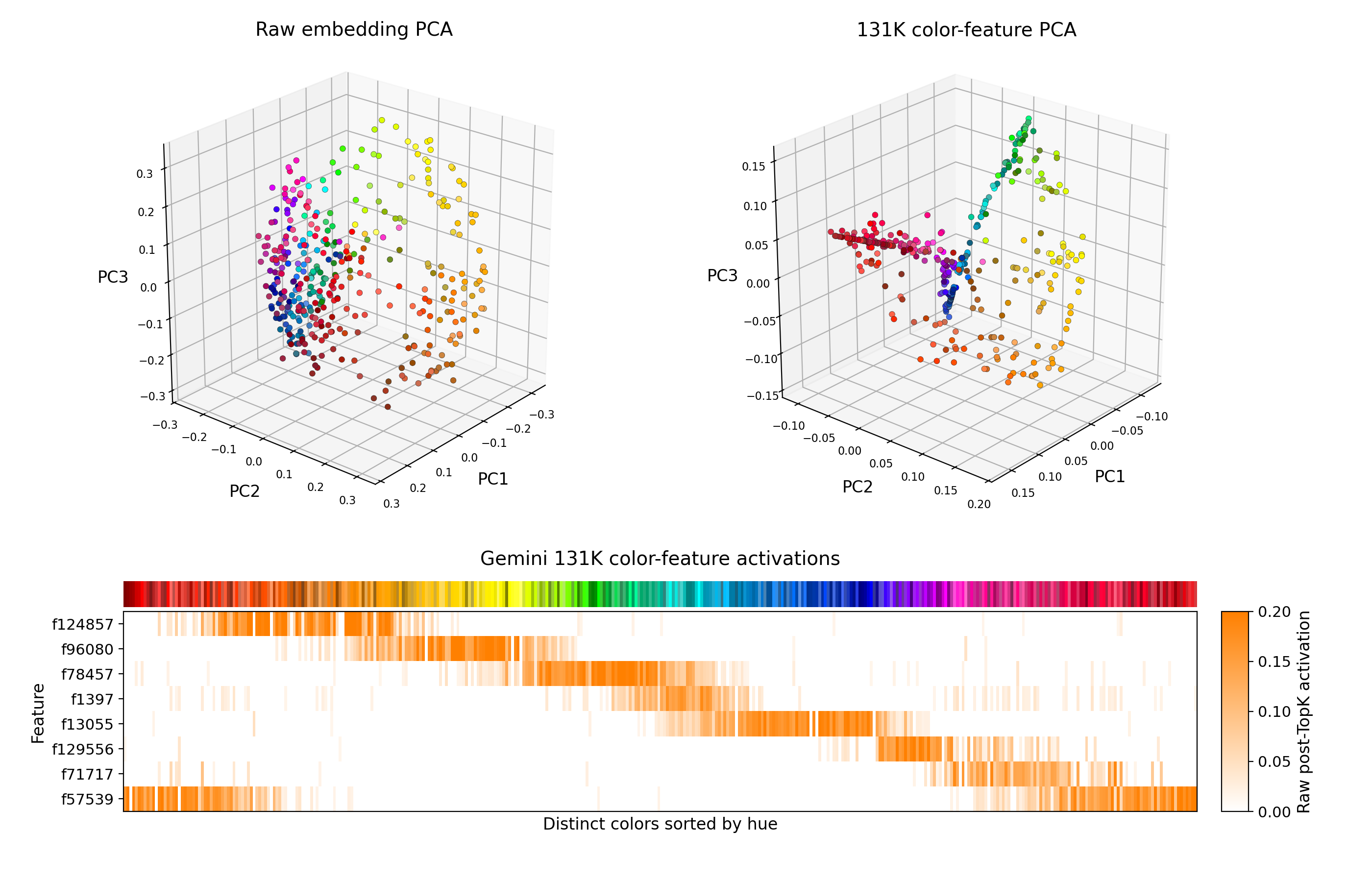}
    \caption{Manifolds arising from small sets of 131K gemini SAE activations. Top left shows PCA of raw embeddings, top right shows PCA of the SAE feature representation restricted to a subset of coordinates, and the bottom shows the activation pattern on these selected features. Importantly, the SAE-induced manifolds arise purely from correlation in activations, not correlation in feature directions.}
    \label{fig:manifold}
\end{figure}

\begin{table*}[p]
\centering
\providecommand{\ColorExamplesPath}{gemini_color_feature_examples_table}
\providecommand{\ColorExamplesMaxHeight}{0.80\textheight}
\ifdefined\ColorExamplesBox\else\newsavebox{\ColorExamplesBox}\fi
\begingroup
\begin{lrbox}{\ColorExamplesBox}%
\begin{minipage}{\linewidth}
\setlength{\parindent}{0pt}
\scriptsize
\setlength{\tabcolsep}{0pt}
\renewcommand{\arraystretch}{1}
\begin{tabular}{@{}ll*{5}{c}@{}}
& \textbf{top activating texts} & \multicolumn{5}{c}{\textbf{top activating artwork}} \\[2mm]
\parbox[c]{0.07\linewidth}{\hspace*{-2mm}\textbf{f124857}} &
\parbox[c]{0.27\linewidth}{\raggedright \raisebox{-0.2ex}{\includegraphics[height=1.5ex]{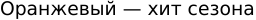}} (0.4119)
\par\vspace{1pt}
Orange Record Label (0.3469)
\par\vspace{1pt}
Orange Hat (0.3398)
\par\vspace{1pt}
Orange Link (0.3359)
\par\vspace{1pt}
Orange ribbon (0.3344)} &
\parbox[c]{0.132\linewidth}{\centering
\includegraphics[width=0.96\linewidth,height=0.80\linewidth,keepaspectratio]{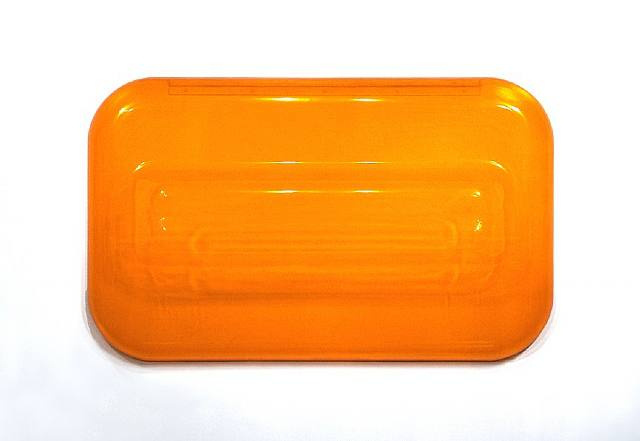}
\par\vspace{1pt} (0.2292)} &
\parbox[c]{0.132\linewidth}{\centering
\includegraphics[width=0.96\linewidth,height=0.80\linewidth,keepaspectratio]{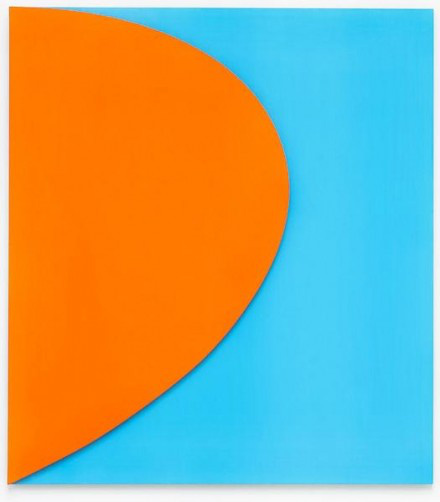}
\par\vspace{1pt} (0.2244)} &
\parbox[c]{0.132\linewidth}{\centering
\includegraphics[width=0.96\linewidth,height=0.80\linewidth,keepaspectratio]{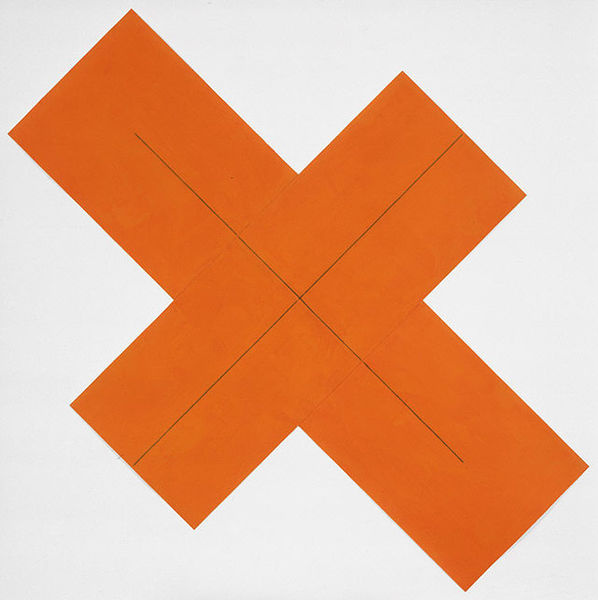}
\par\vspace{1pt} (0.2110)} &
\parbox[c]{0.132\linewidth}{\centering
\includegraphics[width=0.96\linewidth,height=0.80\linewidth,keepaspectratio]{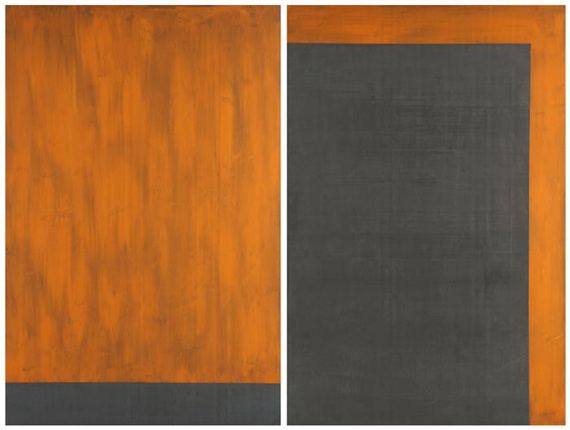}
\par\vspace{1pt} (0.2063)} &
\parbox[c]{0.132\linewidth}{\centering
\includegraphics[width=0.96\linewidth,height=0.80\linewidth,keepaspectratio]{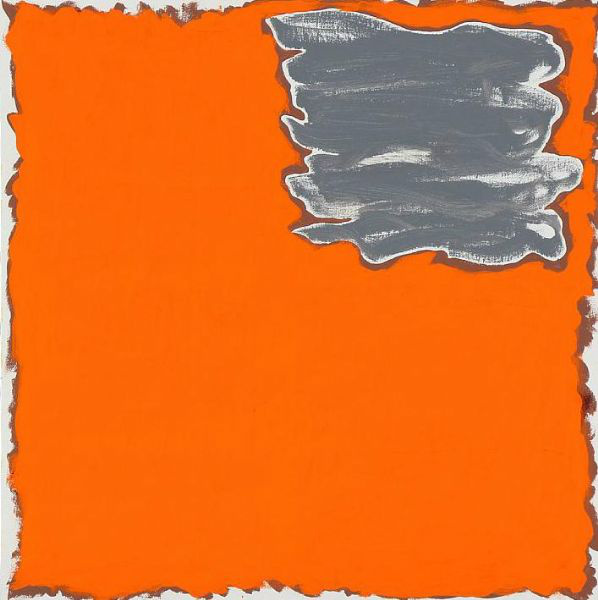}
\par\vspace{1pt} (0.2052)} \\[1mm]

\parbox[c]{0.07\linewidth}{\hspace*{-2mm}\textbf{f96080}} &
\parbox[c]{0.27\linewidth}{\raggedright Yellow hat (0.3367)
\par\vspace{1pt}
Yellow Arrow (0.3251)
\par\vspace{1pt}
Yellow monitor (0.3236)
\par\vspace{1pt}
Yellow dog (0.3185)
\par\vspace{1pt}
Yellow Hair (0.3169)} &
\parbox[c]{0.132\linewidth}{\centering
\includegraphics[width=0.96\linewidth,height=0.80\linewidth,keepaspectratio]{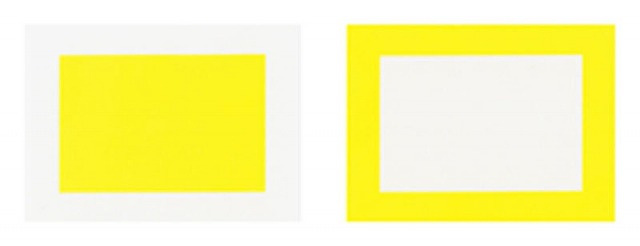}
\par\vspace{1pt} (0.1967)} &
\parbox[c]{0.132\linewidth}{\centering
\includegraphics[width=0.96\linewidth,height=0.80\linewidth,keepaspectratio]{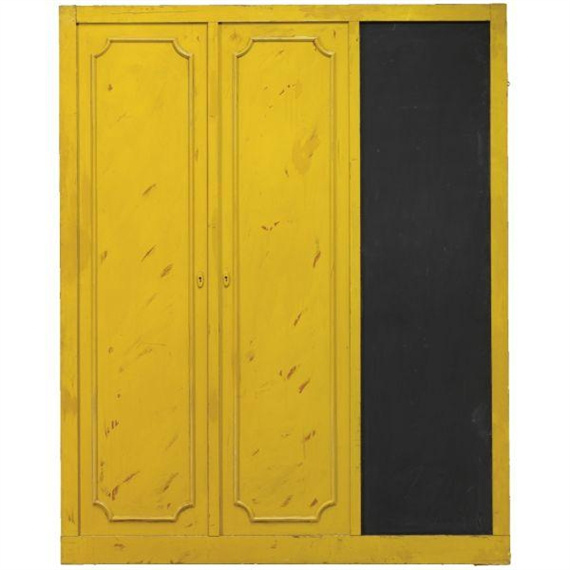}
\par\vspace{1pt} (0.1800)} &
\parbox[c]{0.132\linewidth}{\centering
\includegraphics[width=0.96\linewidth,height=0.80\linewidth,keepaspectratio]{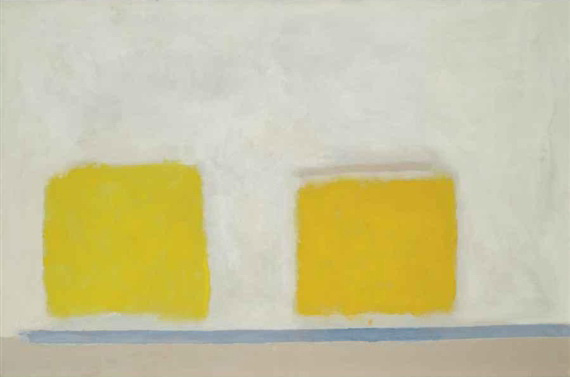}
\par\vspace{1pt} (0.1798)} &
\parbox[c]{0.132\linewidth}{\centering
\includegraphics[width=0.96\linewidth,height=0.80\linewidth,keepaspectratio]{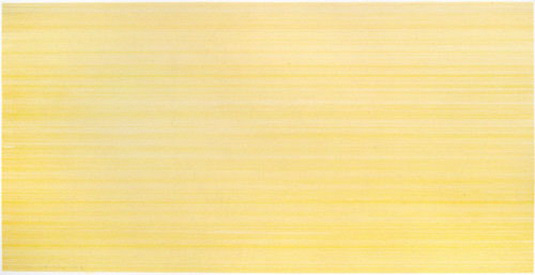}
\par\vspace{1pt} (0.1796)} &
\parbox[c]{0.132\linewidth}{\centering
\includegraphics[width=0.96\linewidth,height=0.80\linewidth,keepaspectratio]{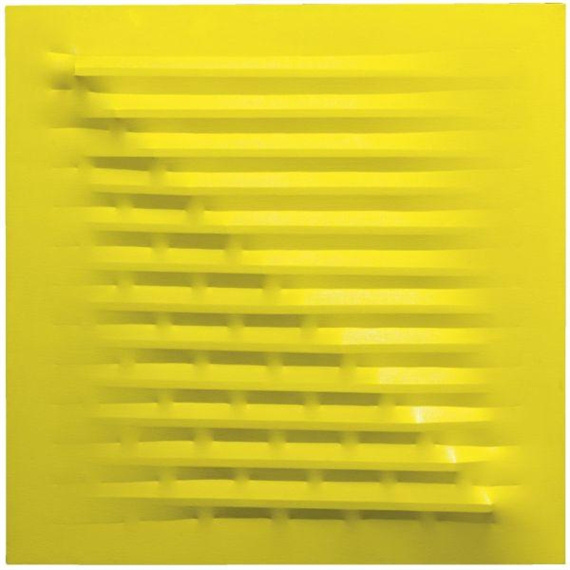}
\par\vspace{1pt} (0.1787)} \\[1mm]

\parbox[c]{0.07\linewidth}{\hspace*{-2mm}\textbf{f78457}} &
\parbox[c]{0.27\linewidth}{\raggedright Green hat (0.3725)
\par\vspace{1pt}
Green and Silver (0.3619)
\par\vspace{1pt}
Green Rose (0.3583)
\par\vspace{1pt}
Green S (0.3556)
\par\vspace{1pt}
Green Light (0.3518)} &
\parbox[c]{0.132\linewidth}{\centering
\includegraphics[width=0.96\linewidth,height=0.80\linewidth,keepaspectratio]{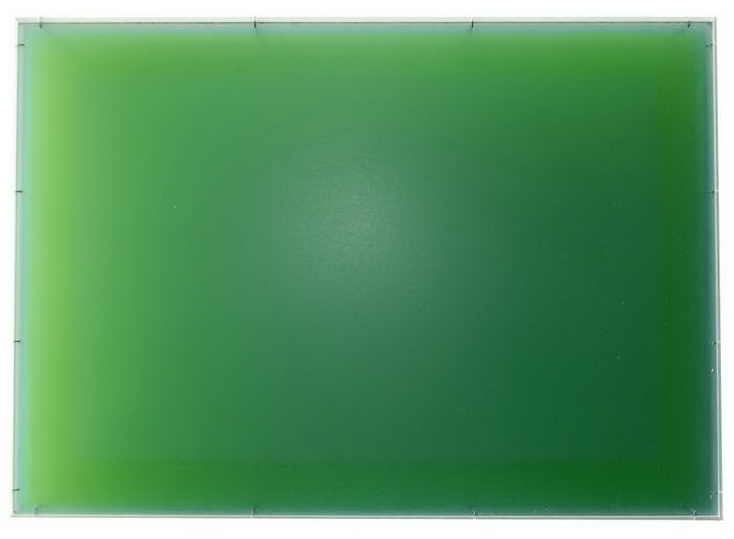}
\par\vspace{1pt} (0.1837)} &
\parbox[c]{0.132\linewidth}{\centering
\includegraphics[width=0.96\linewidth,height=0.80\linewidth,keepaspectratio]{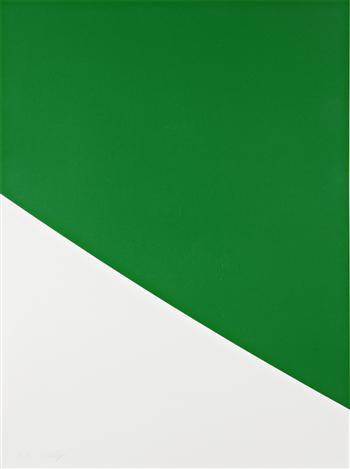}
\par\vspace{1pt} (0.1826)} &
\parbox[c]{0.132\linewidth}{\centering
\includegraphics[width=0.96\linewidth,height=0.80\linewidth,keepaspectratio]{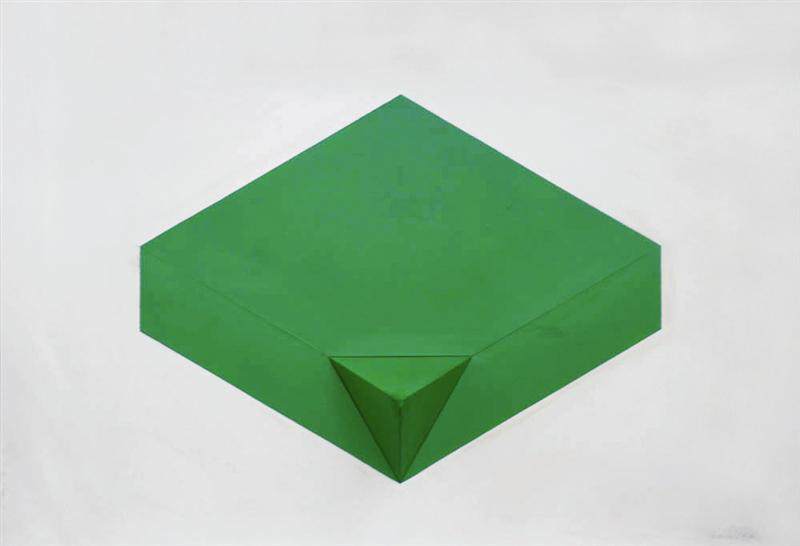}
\par\vspace{1pt} (0.1816)} &
\parbox[c]{0.132\linewidth}{\centering
\includegraphics[width=0.96\linewidth,height=0.80\linewidth,keepaspectratio]{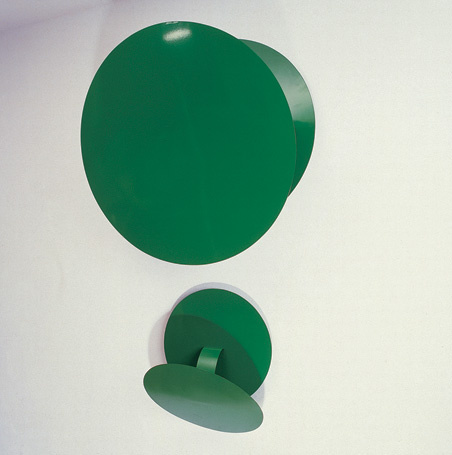}
\par\vspace{1pt} (0.1791)} &
\parbox[c]{0.132\linewidth}{\centering
\includegraphics[width=0.96\linewidth,height=0.80\linewidth,keepaspectratio]{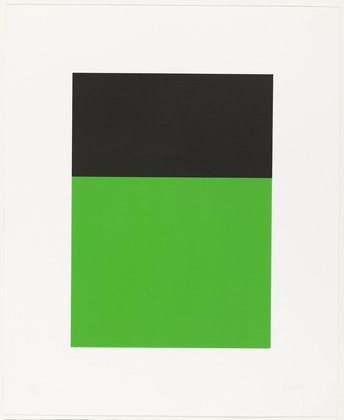}
\par\vspace{1pt} (0.1762)} \\[1mm]

\parbox[c]{0.07\linewidth}{\hspace*{-2mm}\textbf{f1397}} &
\parbox[c]{0.27\linewidth}{\raggedright Turquoise is not a color choice you make for your car. It chooses you! I've been... (0.2814)
\par\vspace{1pt}
Turquoise ribbon (0.2587)
\par\vspace{1pt}
\raisebox{-0.2ex}{\includegraphics[height=1.5ex]{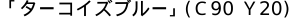}} (0.2533)
\par\vspace{1pt}
turquoise number (0.2505)
\par\vspace{1pt}
Colors That Go With Turquoise. ... Different shades of turquoise don't have to match... (0.2317)} &
\parbox[c]{0.132\linewidth}{\centering
\includegraphics[width=0.96\linewidth,height=0.80\linewidth,keepaspectratio]{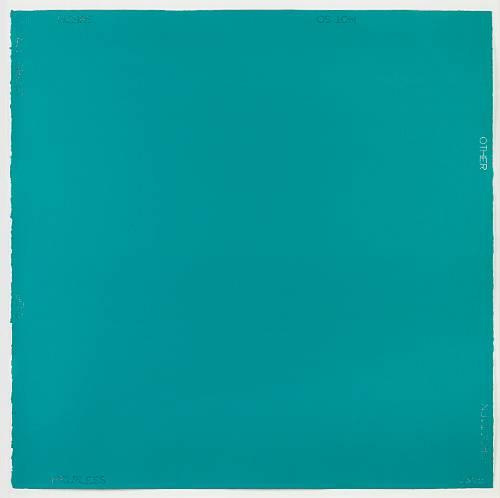}
\par\vspace{1pt} (0.1757)} &
\parbox[c]{0.132\linewidth}{\centering
\includegraphics[width=0.96\linewidth,height=0.80\linewidth,keepaspectratio]{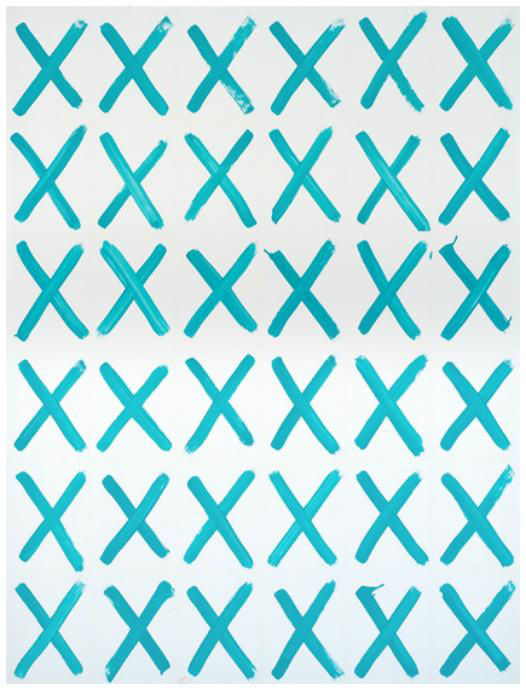}
\par\vspace{1pt} (0.1655)} &
\parbox[c]{0.132\linewidth}{\centering
\includegraphics[width=0.96\linewidth,height=0.80\linewidth,keepaspectratio]{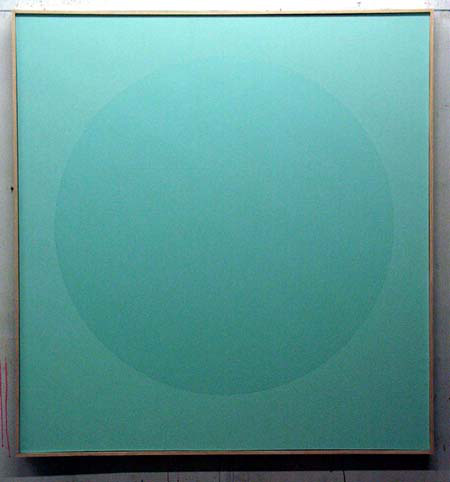}
\par\vspace{1pt} (0.1618)} &
\parbox[c]{0.132\linewidth}{\centering
\includegraphics[width=0.96\linewidth,height=0.80\linewidth,keepaspectratio]{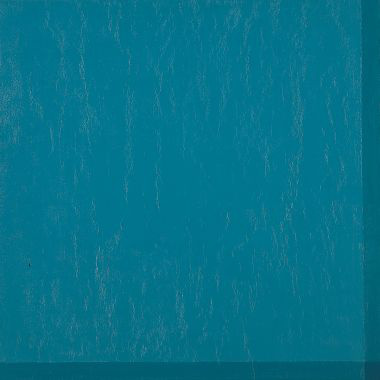}
\par\vspace{1pt} (0.1565)} &
\parbox[c]{0.132\linewidth}{\centering
\includegraphics[width=0.96\linewidth,height=0.80\linewidth,keepaspectratio]{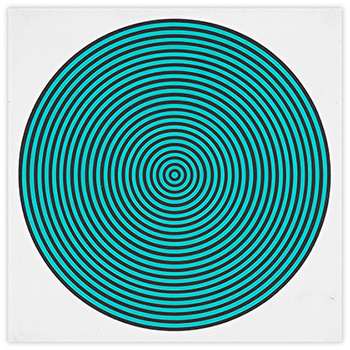}
\par\vspace{1pt} (0.1511)} \\[1mm]

\parbox[c]{0.07\linewidth}{\hspace*{-2mm}\textbf{f13055}} &
\parbox[c]{0.27\linewidth}{\raggedright Blue Rainbow (0.3702)
\par\vspace{1pt}
Blue hat (0.3697)
\par\vspace{1pt}
Blue Girl (0.3693)
\par\vspace{1pt}
Blue Snowman (0.3691)
\par\vspace{1pt}
Blue Monkey (0.3658)} &
\parbox[c]{0.132\linewidth}{\centering
\includegraphics[width=0.96\linewidth,height=0.80\linewidth,keepaspectratio]{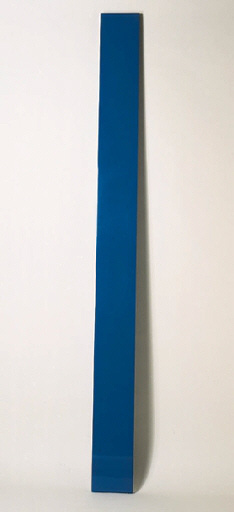}
\par\vspace{1pt} (0.1800)} &
\parbox[c]{0.132\linewidth}{\centering
\includegraphics[width=0.96\linewidth,height=0.80\linewidth,keepaspectratio]{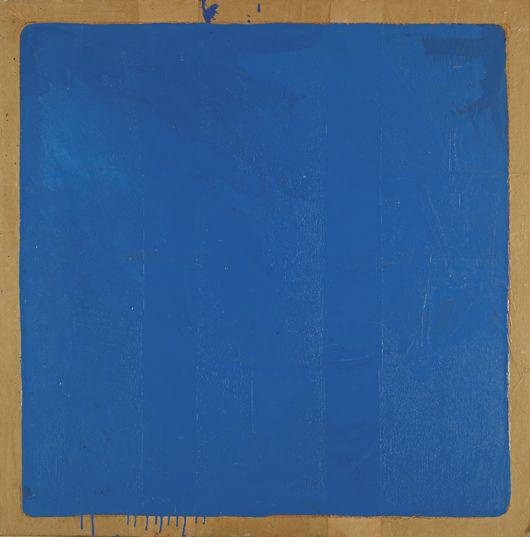}
\par\vspace{1pt} (0.1731)} &
\parbox[c]{0.132\linewidth}{\centering
\includegraphics[width=0.96\linewidth,height=0.80\linewidth,keepaspectratio]{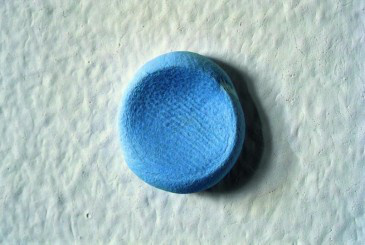}
\par\vspace{1pt} (0.1716)} &
\parbox[c]{0.132\linewidth}{\centering
\includegraphics[width=0.96\linewidth,height=0.80\linewidth,keepaspectratio]{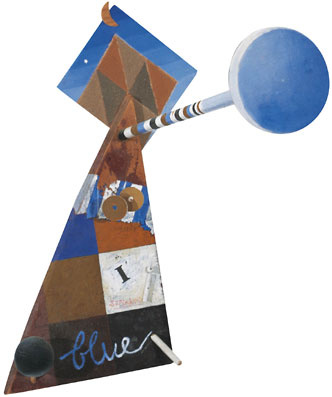}
\par\vspace{1pt} (0.1708)} &
\parbox[c]{0.132\linewidth}{\centering
\includegraphics[width=0.96\linewidth,height=0.80\linewidth,keepaspectratio]{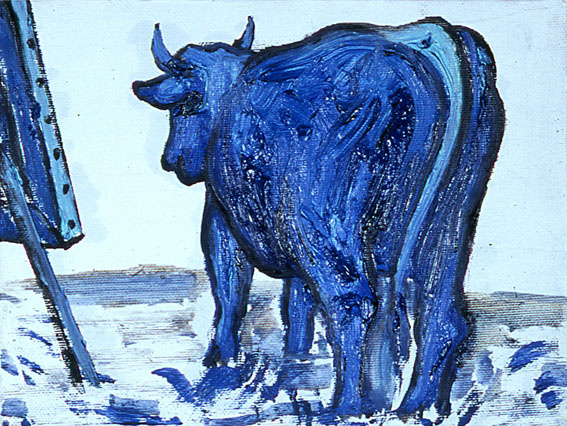}
\par\vspace{1pt} (0.1637)} \\[1mm]

\parbox[c]{0.07\linewidth}{\hspace*{-2mm}\textbf{f129556}} &
\parbox[c]{0.27\linewidth}{\raggedright Purple Numbers (0.3519)
\par\vspace{1pt}
Purple crocodile (0.3409)
\par\vspace{1pt}
i feel like rich purple and gold are a match made in heaven and this reinforces that... (0.3329)
\par\vspace{1pt}
Purple Love (0.3328)
\par\vspace{1pt}
Purple Radio (0.3317)} &
\parbox[c]{0.132\linewidth}{\centering
\includegraphics[width=0.96\linewidth,height=0.80\linewidth,keepaspectratio]{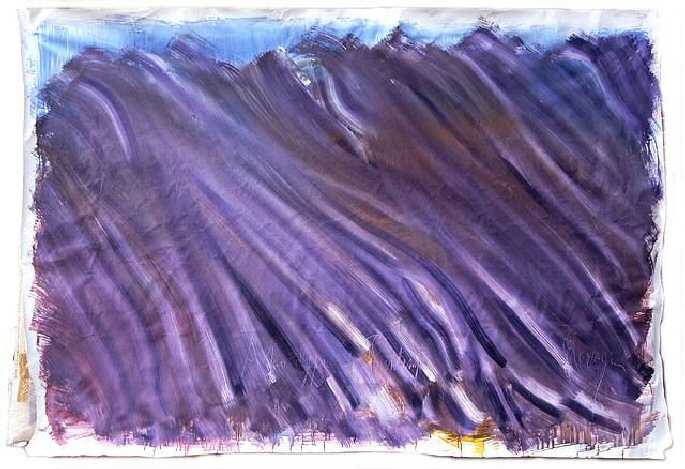}
\par\vspace{1pt} (0.2225)} &
\parbox[c]{0.132\linewidth}{\centering
\includegraphics[width=0.96\linewidth,height=0.80\linewidth,keepaspectratio]{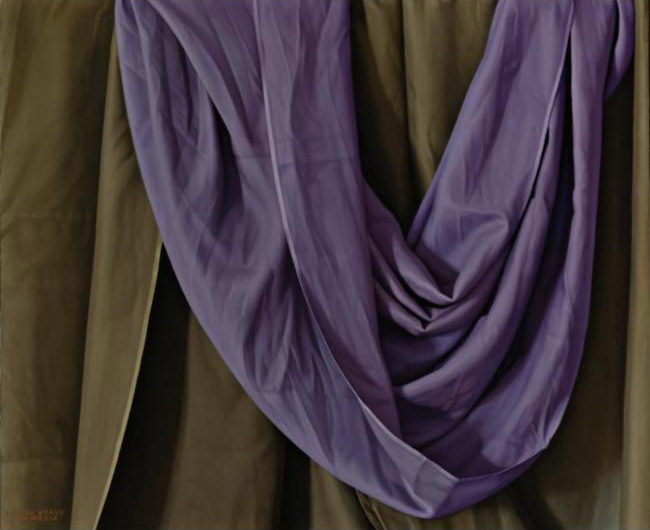}
\par\vspace{1pt} (0.2112)} &
\parbox[c]{0.132\linewidth}{\centering
\includegraphics[width=0.96\linewidth,height=0.80\linewidth,keepaspectratio]{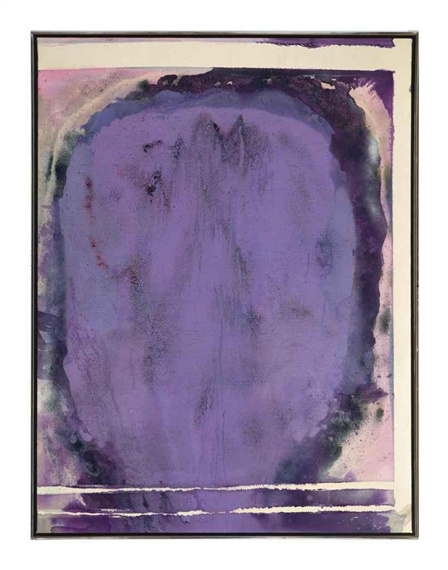}
\par\vspace{1pt} (0.1996)} &
\parbox[c]{0.132\linewidth}{\centering
\includegraphics[width=0.96\linewidth,height=0.80\linewidth,keepaspectratio]{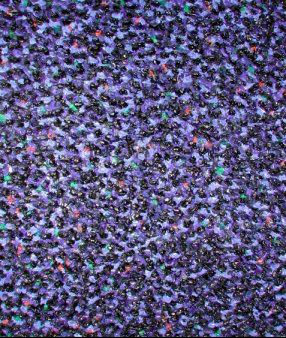}
\par\vspace{1pt} (0.1966)} &
\parbox[c]{0.132\linewidth}{\centering
\includegraphics[width=0.96\linewidth,height=0.80\linewidth,keepaspectratio]{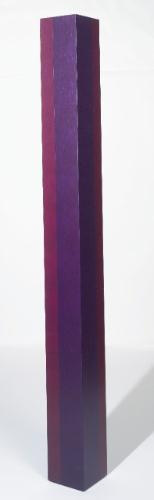}
\par\vspace{1pt} (0.1950)} \\[1mm]

\parbox[c]{0.07\linewidth}{\hspace*{-2mm}\textbf{f71717}} &
\parbox[c]{0.27\linewidth}{\raggedright Pink Subaru (0.3593)
\par\vspace{1pt}
Pink contract (0.3525)
\par\vspace{1pt}
Pink Pineapple (0.3458)
\par\vspace{1pt}
Pink and Gray (0.3367)
\par\vspace{1pt}
Pink Man (0.3366)} &
\parbox[c]{0.132\linewidth}{\centering
\includegraphics[width=0.96\linewidth,height=0.80\linewidth,keepaspectratio]{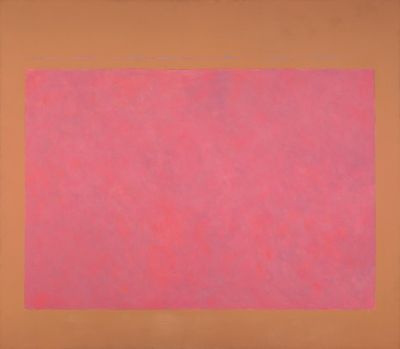}
\par\vspace{1pt} (0.2032)} &
\parbox[c]{0.132\linewidth}{\centering
\includegraphics[width=0.96\linewidth,height=0.80\linewidth,keepaspectratio]{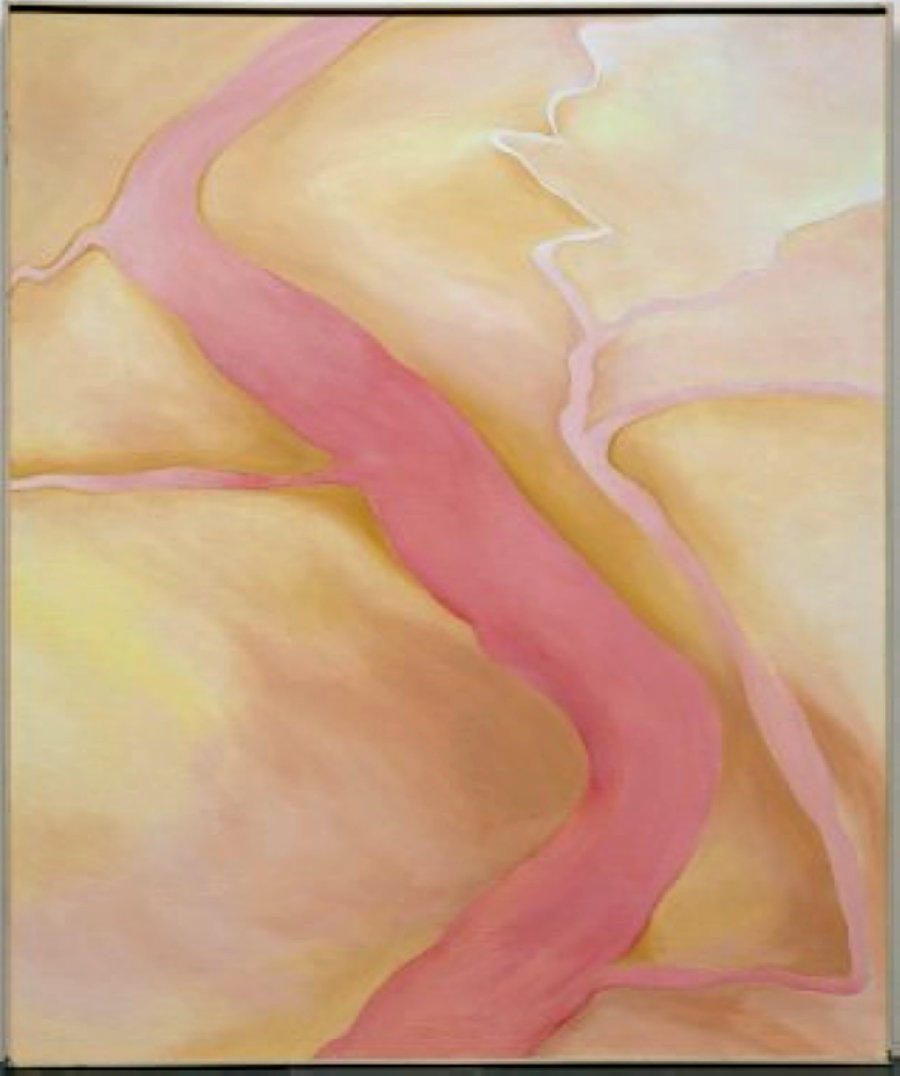}
\par\vspace{1pt} (0.1902)} &
\parbox[c]{0.132\linewidth}{\centering
\includegraphics[width=0.96\linewidth,height=0.80\linewidth,keepaspectratio]{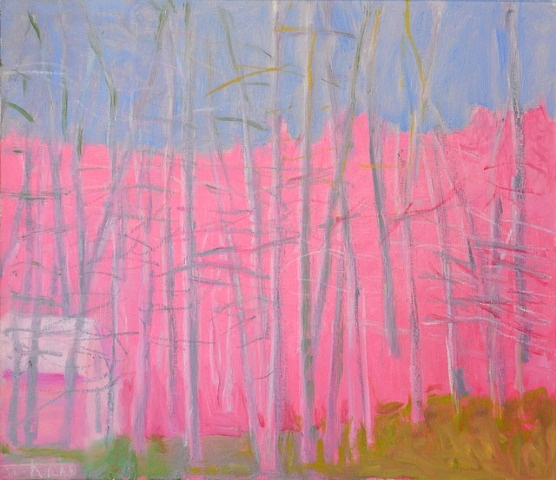}
\par\vspace{1pt} (0.1841)} &
\parbox[c]{0.132\linewidth}{\centering
\includegraphics[width=0.96\linewidth,height=0.80\linewidth,keepaspectratio]{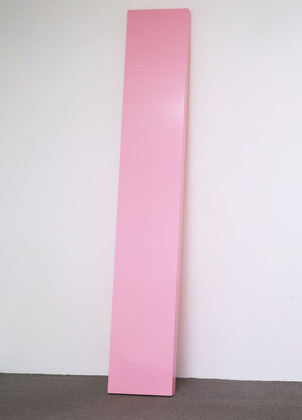}
\par\vspace{1pt} (0.1840)} &
\parbox[c]{0.132\linewidth}{\centering
\includegraphics[width=0.96\linewidth,height=0.80\linewidth,keepaspectratio]{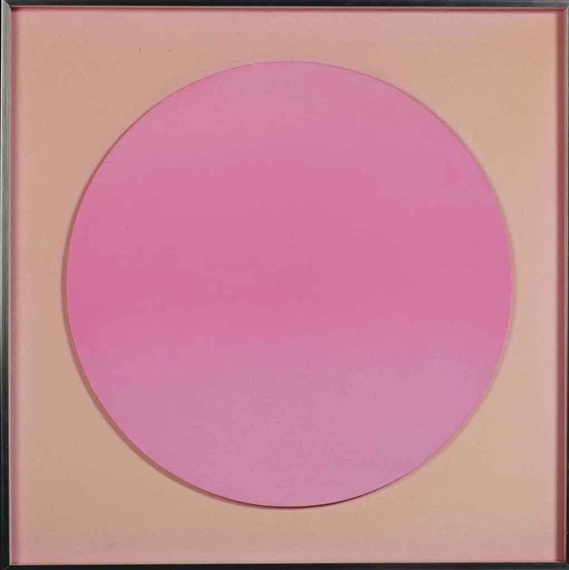}
\par\vspace{1pt} (0.1839)} \\[1mm]

\parbox[c]{0.07\linewidth}{\hspace*{-2mm}\textbf{f57539}} &
\parbox[c]{0.27\linewidth}{\raggedright \raisebox{-0.2ex}{\includegraphics[height=1.5ex]{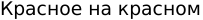}} (0.4016)
\par\vspace{1pt}
Red Elephant (0.3942)
\par\vspace{1pt}
Red Bridge (0.3836)
\par\vspace{1pt}
Red Post (0.3779)
\par\vspace{1pt}
Red jersey (0.3765)} &
\parbox[c]{0.132\linewidth}{\centering
\includegraphics[width=0.96\linewidth,height=0.80\linewidth,keepaspectratio]{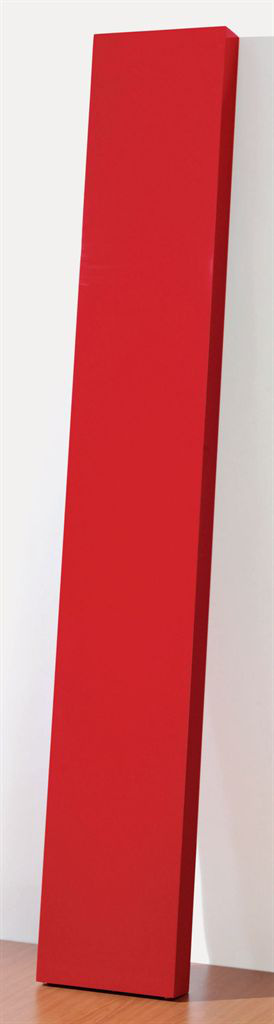}
\par\vspace{1pt} (0.1869)} &
\parbox[c]{0.132\linewidth}{\centering
\includegraphics[width=0.96\linewidth,height=0.80\linewidth,keepaspectratio]{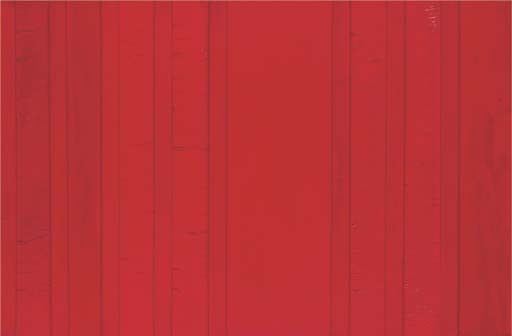}
\par\vspace{1pt} (0.1832)} &
\parbox[c]{0.132\linewidth}{\centering
\includegraphics[width=0.96\linewidth,height=0.80\linewidth,keepaspectratio]{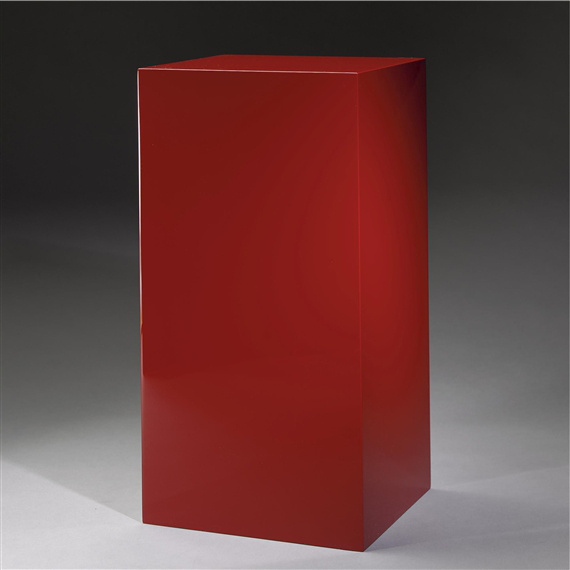}
\par\vspace{1pt} (0.1821)} &
\parbox[c]{0.132\linewidth}{\centering
\includegraphics[width=0.96\linewidth,height=0.80\linewidth,keepaspectratio]{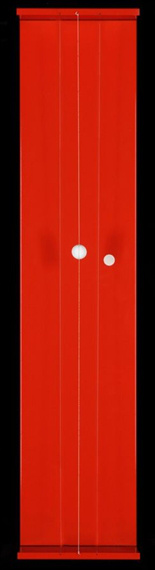}
\par\vspace{1pt} (0.1800)} &
\parbox[c]{0.132\linewidth}{\centering
\includegraphics[width=0.96\linewidth,height=0.80\linewidth,keepaspectratio]{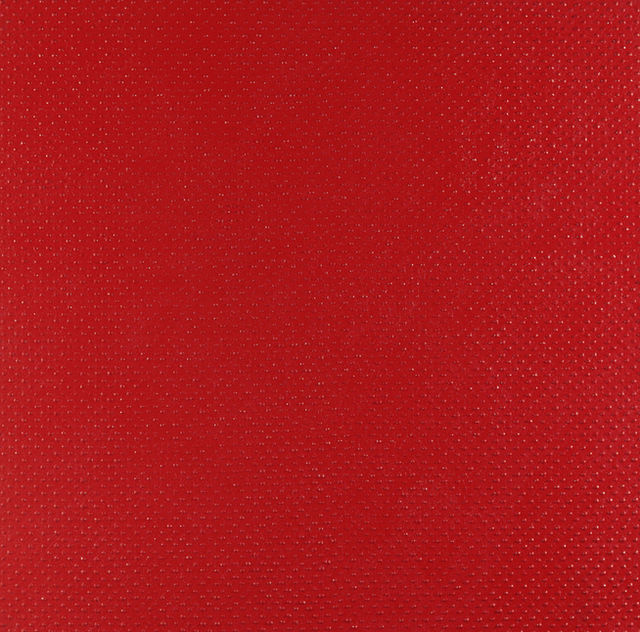}
\par\vspace{1pt} (0.1792)} \\[1mm]
\end{tabular}
\end{minipage}%
\end{lrbox}%
\noindent\makebox[\linewidth][c]{%
  \ifdim\dimexpr\ht\ColorExamplesBox+\dp\ColorExamplesBox\relax>\ColorExamplesMaxHeight\relax
    \resizebox*{!}{\ColorExamplesMaxHeight}{\usebox{\ColorExamplesBox}}%
  \else
    \usebox{\ColorExamplesBox}%
  \fi
}%
\endgroup

\par
\caption{Top-activating text and image examples for eight color-selective
features in the Gemini 131K sparse autoencoder. Each row shows the five
highest-activating distinct general-corpus texts and five highest-activating
WikiArt images for one feature, ordered by decreasing activation within
each corpus. Parentheses give raw post-TopK activations. Texts are
truncated for display.}
\label{tab:gemini-color-examples}
\end{table*}

\clearpage

\section{An Atomic View of the Platonic Representation Hypothesis}\label{sec:platonic}

If we interpret the rigidity property at an even higher level, then we can arrive at an even stronger theory: that atomic features are not atomic only with respect to an individual model, but are actually atomic as descriptions of the world (or at least, of descriptions of human-generated text). We might describe these atomic features as also being ``platonic'' \citep{huh2024platonic}.

So while there is no reason for atomic directions to be model independent (one can have an equivalent representation that is a rotation), it seems possible that $z$ could be model independent, in which case different representations $f,f'$ are just different projections of the feature representation: that is, $f=Az$ and $f'=A'z.$

\paragraph{Same-width cross-model matching.}
We compare activations between models for all nine SAE widths from 512 to 131,072. At each width and threshold $t\in\{0.5,0.6,0.7,0.8\}$, we apply the pairwise matching and candidate-graph procedures in \Cref{sec:matching-computation,sec:matching-candidates} to Gemini and Nemotron features.

\Cref{fig:cross-model-cardinality-thresholds} shows that the number of matched pairs increases with SAE width, while the fraction of each dictionary that is matched generally decreases.

\begin{figure}[htbp]
    \centering
    \includegraphics[width=\linewidth]{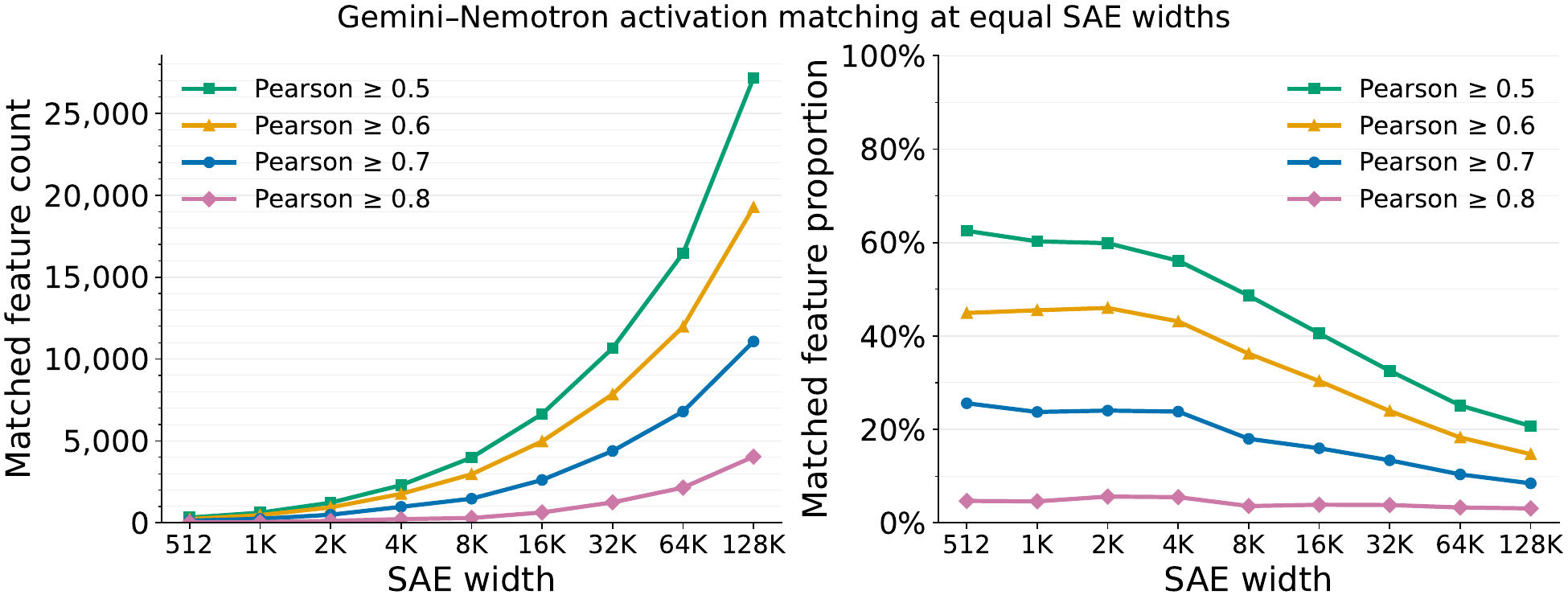}
    \caption{Feature-activation matching between Gemini and Nemotron SAEs of equal width. Curves give maximum-cardinality one-to-one matching at signed Pearson thresholds $t\in\{0.5,0.6,0.7,0.8\}$ on $89{,}227{,}558$ shared examples. Left: matched-pair count. Right: count divided by the common SAE width, including dead and unmatched features in the denominator. Counts are exact on the retained candidate graphs and lower bounds on exhaustive matching. The 131K comparison is included because matching across models does not require a larger SAE.}
    \label{fig:cross-model-cardinality-thresholds}
\end{figure}

\paragraph{Example matches.}
\Crefrange{tab:platonic-pearson-semantic-combined-1}{tab:platonic-pearson-diverse-10} show twelve illustrative correlated pairs of features: two main examples and ten additional diverse pairs, with Pearson similarity ranging from $0.3$ to $0.9$. They cover clarification of meaning, algebraic transformations, education, trigonometry, astronomy, trade, programming, mechanics, food, sign languages, music, and botany. The captions describe qualitative interpretations of the activating texts. Each table contains fifteen sampled texts from each model: eight with source-feature activation in $(0.05,0.1]$ and seven above $0.1$. Gemini feature 71850 is an exception: it has no cached activations above $0.1$, so all fifteen of its examples come from $(0.05,0.1]$.

\clearpage
\begin{table*}[p]
\centering
\caption{Clarifying the meaning or scope of a term. Pearson $r=0.352$.}
\label{tab:platonic-pearson-semantic-combined-1}
\includegraphics[width=\linewidth,height=0.44\textheight,keepaspectratio]{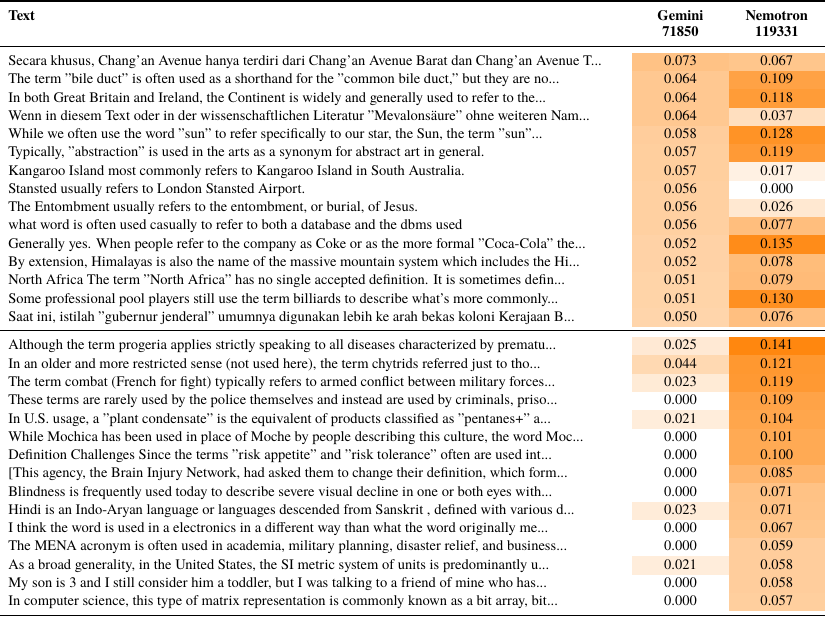}
\end{table*}

\begin{table*}[p]
\centering
\caption{Whether algebraic transformations preserve solutions. Pearson $r=0.362$.}
\label{tab:platonic-pearson-semantic-combined-2}
\includegraphics[width=\linewidth,height=0.44\textheight,keepaspectratio]{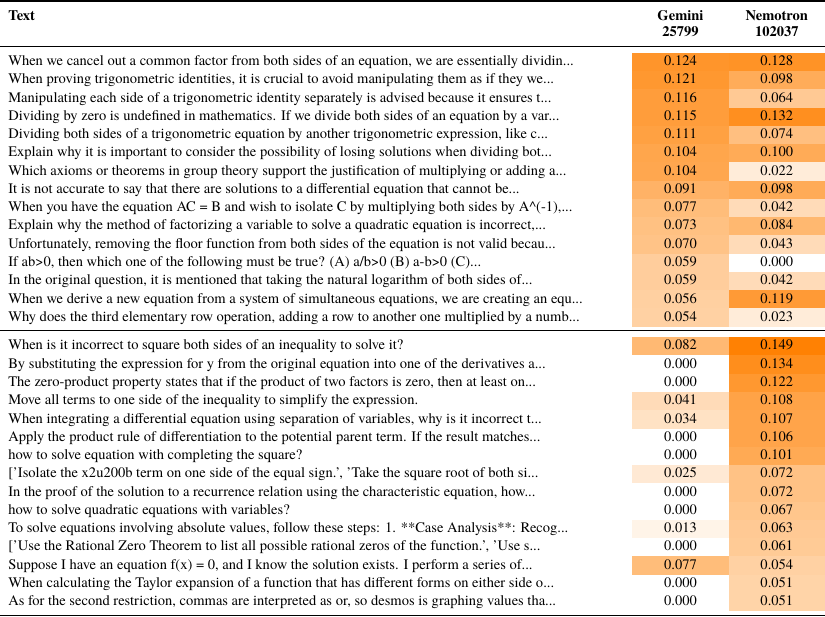}
\end{table*}

\clearpage
\begin{table*}[p]
\centering
\caption{School activities and educational offerings. Pearson $r=0.516$.}
\label{tab:platonic-pearson-diverse-1}
\includegraphics[width=\linewidth,height=0.44\textheight,keepaspectratio]{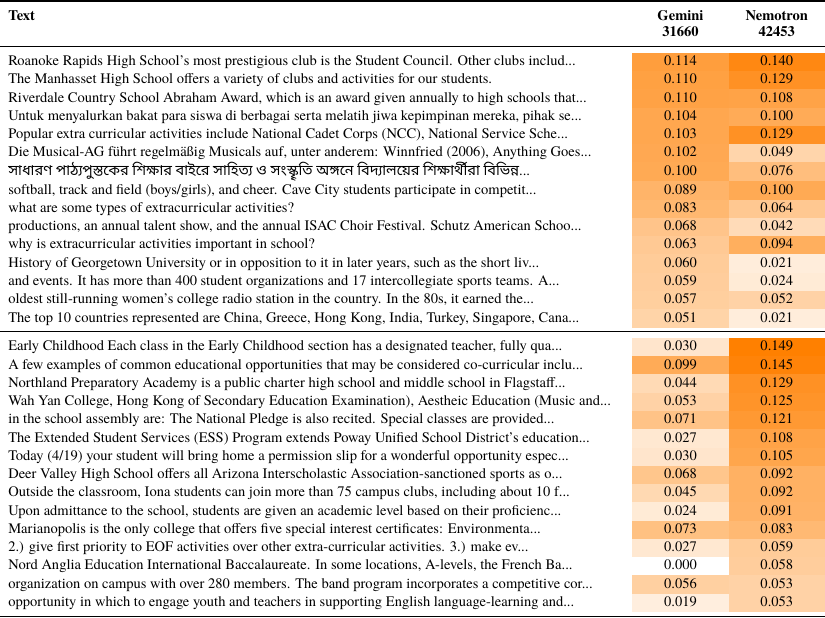}
\end{table*}

\begin{table*}[p]
\centering
\caption{Trigonometric half-angle identities. Pearson $r=0.578$.}
\label{tab:platonic-pearson-diverse-2}
\includegraphics[width=\linewidth,height=0.44\textheight,keepaspectratio]{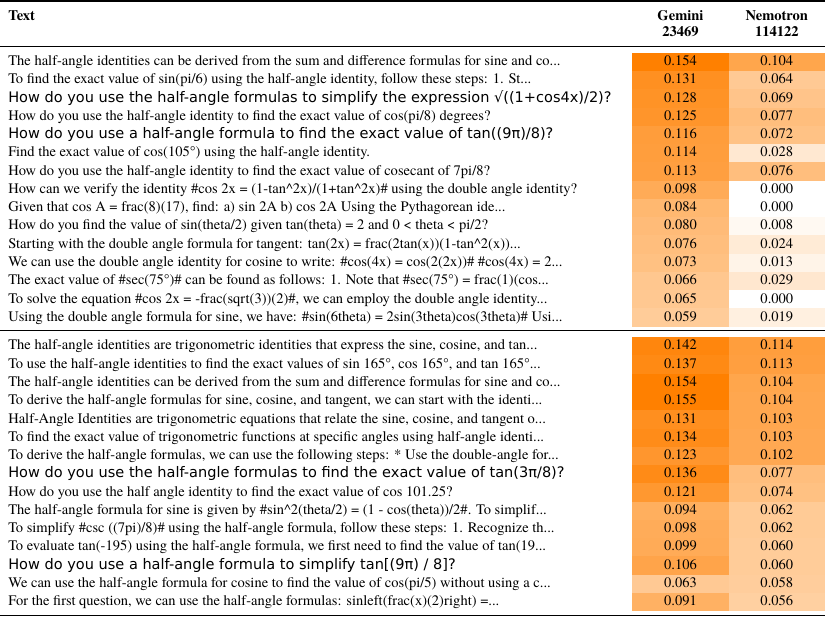}
\end{table*}

\clearpage
\begin{table*}[p]
\centering
\caption{Telescopes and astronomical optics. Pearson $r=0.604$.}
\label{tab:platonic-pearson-diverse-3}
\includegraphics[width=\linewidth,height=0.44\textheight,keepaspectratio]{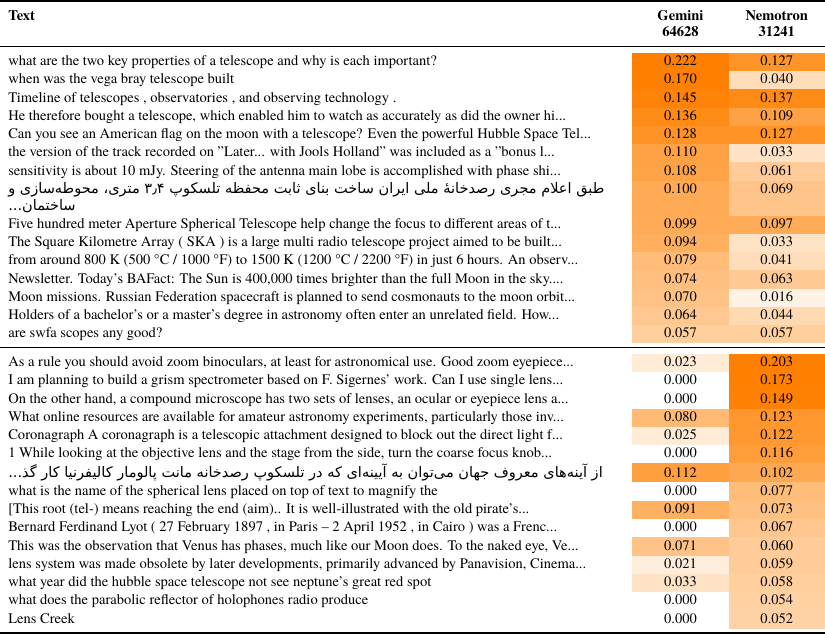}
\end{table*}

\begin{table*}[p]
\centering
\caption{Trade tariffs and protectionism. Pearson $r=0.651$.}
\label{tab:platonic-pearson-diverse-4}
\includegraphics[width=\linewidth,height=0.44\textheight,keepaspectratio]{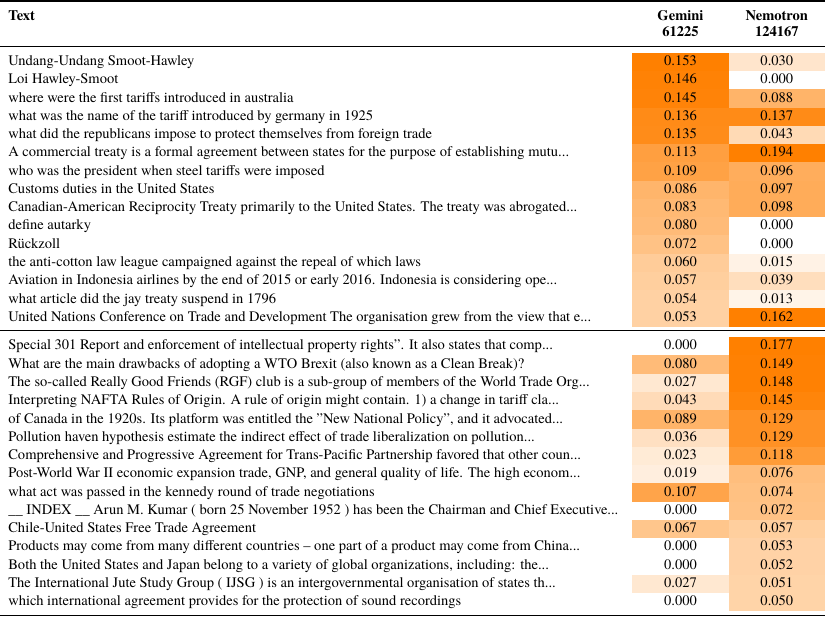}
\end{table*}

\clearpage
\begin{table*}[p]
\centering
\caption{Comments in source code. Pearson $r=0.667$.}
\label{tab:platonic-pearson-diverse-5}
\includegraphics[width=\linewidth,height=0.44\textheight,keepaspectratio]{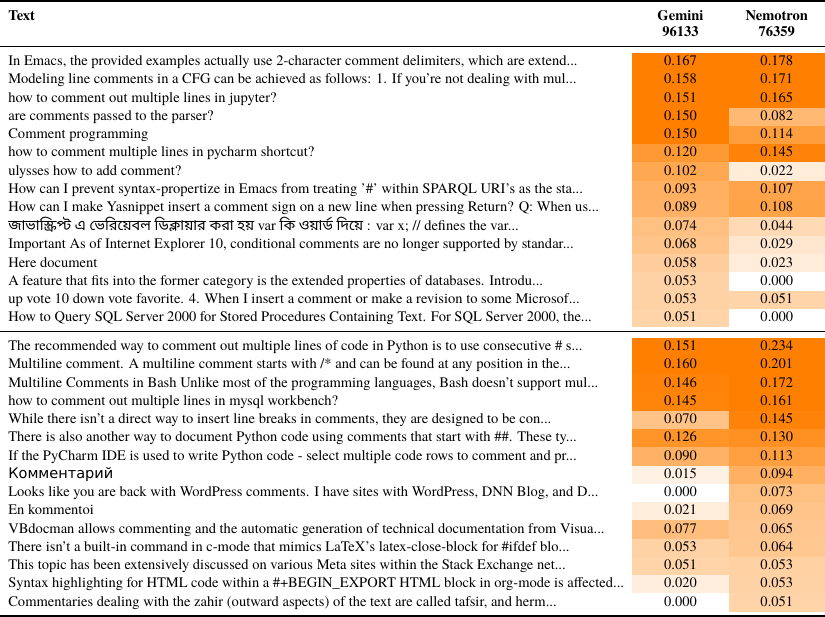}
\end{table*}

\begin{table*}[p]
\centering
\caption{Mechanical bearings. Pearson $r=0.709$.}
\label{tab:platonic-pearson-diverse-6}
\includegraphics[width=\linewidth,height=0.44\textheight,keepaspectratio]{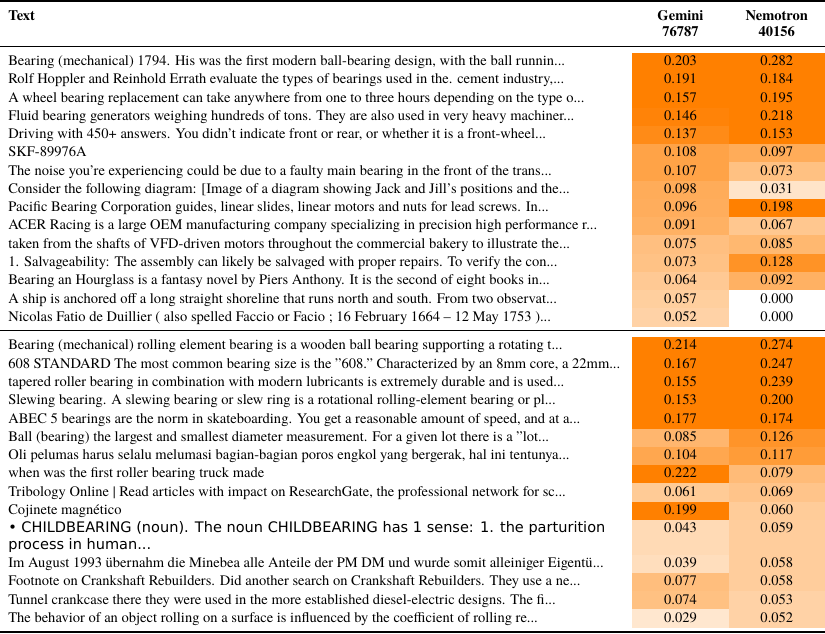}
\end{table*}

\clearpage
\begin{table*}[p]
\centering
\caption{Butter. Pearson $r=0.788$.}
\label{tab:platonic-pearson-diverse-7}
\includegraphics[width=\linewidth,height=0.44\textheight,keepaspectratio]{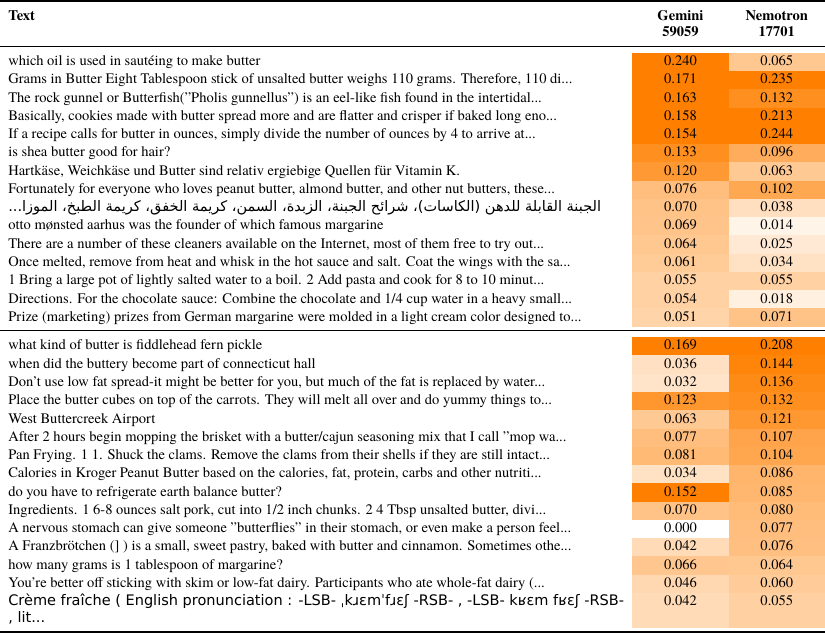}
\end{table*}

\begin{table*}[p]
\centering
\caption{Sign languages. Pearson $r=0.802$.}
\label{tab:platonic-pearson-diverse-8}
\includegraphics[width=\linewidth,height=0.44\textheight,keepaspectratio]{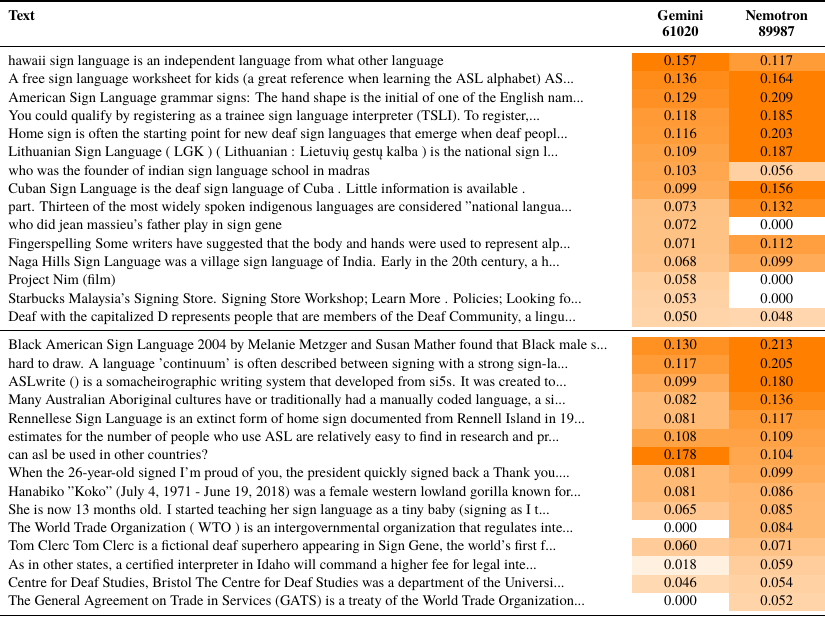}
\end{table*}

\clearpage
\begin{table*}[p]
\centering
\caption{Eurovision Song Contest entries. Pearson $r=0.865$.}
\label{tab:platonic-pearson-diverse-9}
\includegraphics[width=\linewidth,height=0.44\textheight,keepaspectratio]{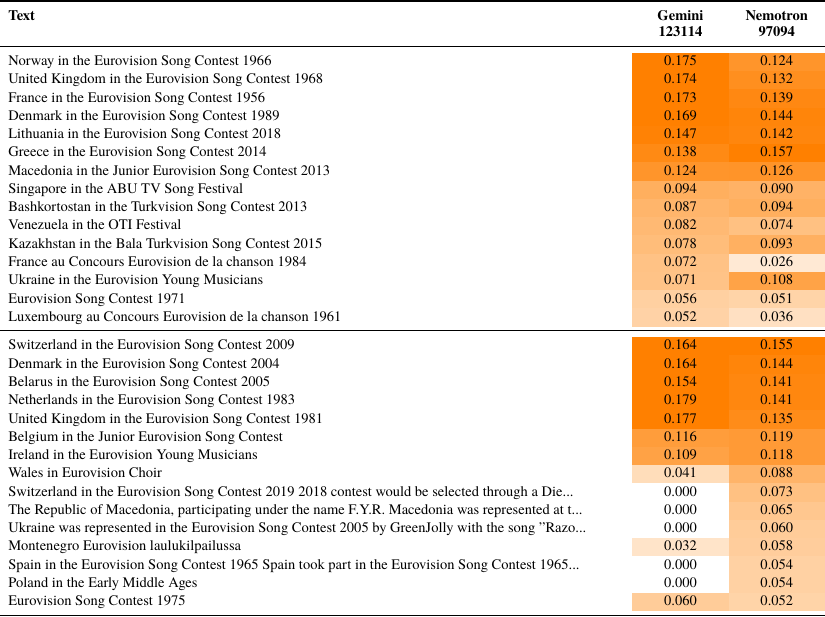}
\end{table*}

\begin{table*}[p]
\centering
\caption{Astragalus and milkvetch species. Pearson $r=0.871$.}
\label{tab:platonic-pearson-diverse-10}
\includegraphics[width=\linewidth,height=0.44\textheight,keepaspectratio]{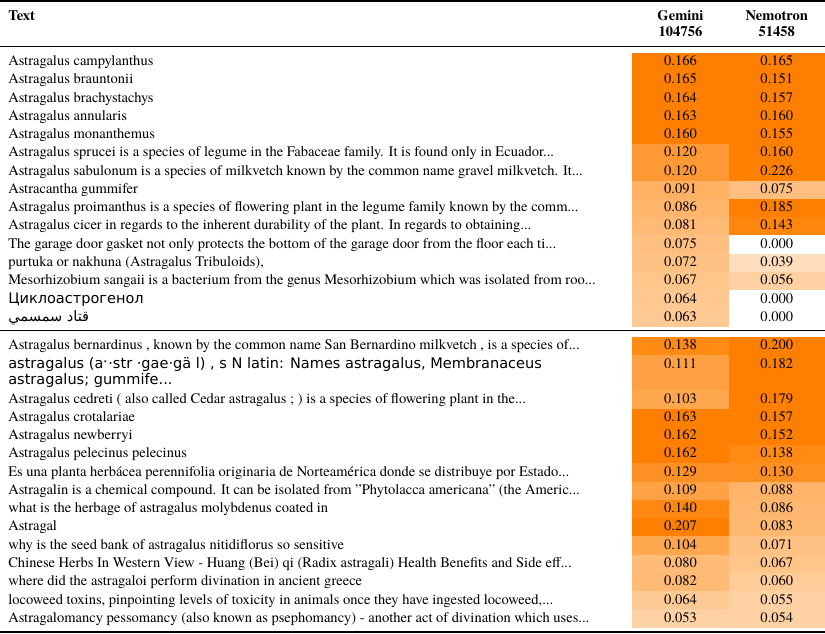}
\end{table*}

\clearpage

\end{document}